\documentclass{article} 
\usepackage{iclr2021_conference,times}

\usepackage{amsmath,amsfonts,bm,bbm}

\def\eqref#1{equation~(\ref{#1})}
\def\Eqref#1{Equation~(\ref{#1})}

\def\Algref#1{Algorithm~\ref{#1}}

\def\1{\bm{1}}

\usepackage{cancel}

\usepackage{listings}
\usepackage{enumerate}

\DeclareMathAlphabet{\mathsfit}{\encodingdefault}{\sfdefault}{m}{sl}
\SetMathAlphabet{\mathsfit}{bold}{\encodingdefault}{\sfdefault}{bx}{n}

\def\sR{{\mathbb{R}}}

\DeclareMathOperator*{\argmax}{arg\,max}
\DeclareMathOperator*{\argmin}{arg\,min}

\DeclareMathOperator{\Tr}{Tr}

\usepackage{amsthm}
\newtheorem{theorem}{Theorem}[section]
\newtheorem{proposition}{Proposition}[section]
\newtheorem{corollary}{Corollary}[theorem]
\newtheorem{lemma}[theorem]{Lemma}
\newtheorem{assumption}[theorem]{Assumption}
\newtheorem{example}{Example}

\newcommand\myeq{\mathrel{\overset{\makebox[0pt]{\mbox{\normalfont\tiny\sffamily def}}}{=}}}

\newcommand{\tr}{^{\top}}

\newcommand{\assign}[2]{%
  \mathrel{#1}\mathrel{#2}%
}

\usepackage[utf8]{inputenc} 
\usepackage[T1]{fontenc}    
\usepackage{hyperref}       
\usepackage{url}            
\usepackage{booktabs}       
\usepackage{amsfonts}       
\usepackage{nicefrac}       
\usepackage{microtype}      

\usepackage{graphicx}
\usepackage{tikz}
\usetikzlibrary{arrows.meta,calc,decorations.markings,math,arrows.meta}

\usepackage{epigraph}
\usepackage{wrapfig}

\usepackage{subcaption}

\usepackage{floatrow}
\newfloatcommand{capbtabbox}{table}[][\FBwidth]
\usepackage[useregional]{datetime2}
\definecolor{green}{rgb}{0.0, 0.42, 0.24} 
\definecolor{orange}{rgb}{0.8, 0.33, 0.} 
\definecolor{blue}{rgb}{0.16, 0.32, 0.75} 
\definecolor{cobalt}{rgb}{0.0, 0.28, 0.67} 

\newcommand{\pcagame}{EigenGame}
\newcommand{\imagenet}{\textsc{ImageNet}}
\newcommand{\resnet}{\textsc{ResNet}}
\newcommand{\mnist}{\textsc{Mnist}}

\newcommand{\qr}{\texttt{QR}}

\usepackage{algpseudocode,algorithm,algorithmicx}
\MakeRobust{\Call}

\usepackage[normalem]{ulem}

\title{\pcagame{}: PCA as a Nash Equilibrium}

\iclrfinalcopy

\author{Ian Gemp, Brian McWilliams, Claire Vernade \& Thore Graepel \\
DeepMind\\
\texttt{\{imgemp,bmcw,vernade,thore\}@google.com} \\
}

\begin{document}

\maketitle

\begin{abstract}
We present a novel view on principal component analysis (PCA) as a competitive game in which each approximate eigenvector is controlled by a player whose goal is to maximize their own utility function. We analyze the properties of this PCA game and the behavior of its gradient based updates. The resulting algorithm---which combines elements from Oja's rule with a  generalized Gram-Schmidt orthogonalization---is naturally decentralized and hence parallelizable through message passing. We demonstrate the scalability of the algorithm with experiments on large image datasets and neural network activations. We discuss how this new view of PCA as a differentiable game can lead to further algorithmic developments and insights.
\end{abstract}

\section{Introduction}

The \emph{principal components} of data are the vectors that align with the directions of maximum variance.
These have two main purposes: a) as interpretable features and b) for data compression. Recent methods for principal component analysis (PCA) focus on the latter, explicitly stating objectives to find the $k$-dimensional subspace that captures maximum variance (e.g.,~\citep{tang2019exponentially}), and leaving the problem of rotating within this subspace to, for example, a more efficient downstream singular value (SVD) decomposition  step\footnote{After learning the top-$k$ subspace $V \in \sR^{d \times k}$, the rotation can be recovered via an SVD of $XV$.}. This point is subtle, yet critical. For example, any pair of two-dimensional, orthogonal vectors spans all of $\mathbb{R}^2$ and, therefore, captures maximum variance of any two-dimensional dataset. However, for these vectors to be principal components, they must, in addition, align with the directions of maximum variance which depends on the covariance of the data. By learning the optimal subspace, rather than the principal components themselves, objectives focused on subspace error ignore the first purpose of PCA. In contrast, modern nonlinear representation learning techniques focus on learning features that are both disentangled (uncorrelated) and low dimensional~\citep{chen2016infogan,mathieu2018disentangling,locatello2019challenging,sarhan2019learning}.

It is well known that the PCA solution of the $d$-dimensional dataset $X \in \sR^{n \times d}$ is given by the eigenvectors of $X\tr X$ or equivalently, the right singular vectors of $X$. Impractically, the cost of computing the full SVD scales with $\mathcal{O}(\min \{nd^2, n^2d \})$-time and $\mathcal{O}(nd)$-space~\citep{shamir2015stochastic,tang2019exponentially}. For moderately sized data, randomized methods can be used \citep{halko2011finding}. Beyond this, stochastic---or online---methods based on Oja's rule \citep{oja1982simplified} or power iterations \citep{rutishauser1971simultaneous} are common. Another option is to use  \emph{streaming k-PCA} algorithms such as Frequent Directions (FD)~\citep{ghashami2016frequent} or Oja's algorithm\footnote{FD approximates the top-$k$ subspace; Oja's algorithm approximates the top-$k$ eigenvectors.}~\citep{allen2017first} with storage complexity $\mathcal{O}(kd)$. Sampling or sketching methods also scale well, but again, focus on the top-$k$ subspace~\citep{sarlos2006improved,cohen2017input,feldman2020turning}.

In contrast to these approaches, we view each principal component (equivalently eigenvector) as a player in a game whose objective is to maximize their own local utility function in controlled competition with other vectors. The proposed utility gradients are interpretable as a combination of Oja's rule and a generalized Gram-Schmidt process. We make the following contributions:
\begin{itemize}
    \item A novel formulation of PCA as finding the Nash equilibrium of a suitable game,
    \item A sequential, globally convergent algorithm for approximating the Nash on full-batch data,
    \item A decentralized algorithm with experiments demonstrating the approach as competitive with modern streaming $k$-PCA algorithms on synthetic and real data,
    \item In demonstration of the scaling of the approach, we compute the top-$32$ principal components of the matrix of \resnet-200 activations on the \imagenet{} dataset ($n\approx 10^6$, $d \approx 20 \cdot 10^6$).
\end{itemize}
Each of these contributions is important. Novel formulations often lead to deeper understanding of problems, thereby, opening doors to improved techniques. In particular, $k$-player games are in general complex and hard to analyze. In contrast, PCA has been well-studied. By combining the two fields we hope to develop useful analytical tools. Our specific formulation is important because it obviates the need for any centralized orthonormalization step and lends itself naturally to decentralization. And lastly, theory and experiments support the viability of this approach for continued research.

\section{PCA as an Eigen-Game}
\label{derivation}

We adhere to the following notation. Vectors and matrices meant to approximate principal components (equivalently eigenvectors) are designated with hats, $\hat{v}$ and $\hat{V}$ respectively, whereas true principal components are $v$ and $V$. Subscripts indicate which eigenvalue a vector is associated with. For example, $v_i$ is the $i$th largest eigenvector. In this work, we will assume each eigenvalue is distinct. By an abuse of notation, $v_{j<i}$ refers to the set of vectors $\{v_j | j \in \{1, \ldots, i-1\}\}$ and are also referred to as the parents of $v_i$ ($v_i$ is their child). Sums over indices should be clear from context, e.g., $\sum_{j < i} = \sum_{j=1}^{i-1}$. The Euclidean inner product is written $\langle u, v \rangle = u^\top v$. We denote the unit sphere by $\mathcal{S}^{d-1}$ and simplex by $\Delta^{d-1}$ in $d$-dimensional ambient space.

\paragraph{Outline of derivation} As argued in the introduction, the PCA problem is often \emph{mis}-interpreted as learning a projection of the data into a subspace that captures maximum variance (equiv. maximizing the trace of a suitable matrix $R$ introduced below). This is in contrast to the original goal of learning the \emph{principal components}. We first develop the intuition for deriving our utility functions by \textcolor{red}{(i)} showing that only maximizing the trace of $R$ is not sufficient for recovering \textbf{all} principal components (equiv. eigenvectors), and \textcolor{blue}{(ii)} showing that minimizing off-diagonal terms in $R$ is a complementary objective to maximizing the trace and can recover \textbf{all} components. We then consider learning only the top-$k$ and construct utilities that are consistent with findings in \textcolor{red}{(i)} and \textcolor{blue}{(ii)}, equal the true eigenvalues at the Nash of the game we construct, and result in a game that is amenable to analysis.

\paragraph{Derivation of player utilities.} \label{sec:utility}
The \emph{eigenvalue} problem for a symmetric matrix $X^\top X = M \in \sR^{d\times d}$ is to find a matrix of $d$ orthonormal column vectors $V$ (implies $V$ is full-rank) such that $MV = V \Lambda$ with $\Lambda$ diagonal. Given a solution to this problem, the columns of $V$ are known as eigenvectors and corresponding entries in $\Lambda$ are eigenvalues. By left-multiplying by $V^\top$ and recalling $V^\top V = VV^\top = I$ by orthonormality (i.e., $V$ is unitary), we can rewrite the equality as
\begin{align}
    V^\top M V &= V^\top V \Lambda \stackrel{\text{unitary}}{=} \Lambda. \label{evp}
\end{align}

Let $\hat{V}$ denote a guess or estimate of the true eigenvectors $V$ and define $R(\hat{V}) \myeq \hat{V}^\top M \hat{V}$. The PCA problem is often posed as maximizing the trace of $R$ (equiv. minimizing reconstruction error):

\begin{align}
    \max_{\hat{V}^\top \hat{V} = I} \Big\{ \textcolor{red}{\sum_{i} R_{ii}} = \Tr(R) = \Tr(\hat{V}^\top M \hat{V}) = \Tr(\hat{V} \hat{V}^\top M) \,\textcolor{red}{= \Tr(M)} \Big\}. \label{max_trace}
\end{align}
Surprisingly, the objective in (\ref{max_trace}) is independent of $\hat{V}$, so it cannot be used to recover \textbf{all} (i.e., $k=d$) the eigenvectors of $M$\textemdash \textcolor{red}{(i)}. Alternatively, \Eqref{evp} implies the \emph{eigenvalue problem} can be phrased as ensuring all off-diagonal terms of $R$ are zero, thereby ensuring $R$ is diagonal\textemdash \textcolor{blue}{(ii)}:
\begin{align}
    \min_{\hat{V}^\top \hat{V} = I} \sum_{i \ne j} R_{ij}^2. \label{min_offdiag}
\end{align}
It is worth further examining the entries of $R$ in detail. Diagonal entries $R_{ii} = \langle \hat{v}_i, M \hat{v}_i \rangle$ are recognized as \emph{Rayleigh quotients} because $||\hat{v}_i|| = 1$ by the constraints. Off-diagonal entries $R_{ij} = \langle \hat{v}_i, M \hat{v}_j \rangle$ measure alignment between $\hat{v}_i$ and $\hat{v}_j$ under a generalized inner product $\langle \cdot , \cdot \rangle_M$.

So far, we have considered learning all the eigenvectors. If we repeat the logic for the top-$k$ eigenvectors with $k < d$, then by \Eqref{evp}, $R$ must still be diagonal. $V$ is not square, so $VV^\top \ne I$, but assuming $V$ is orthonormal as before, we have $VV^\top=P$ is a projection matrix. Left-multiplying \Eqref{evp} by $V$ now reads $(PM) V = V \Lambda$ so we are solving an \emph{eigenvalue problem} for a subspace of $M$.

If we only desire the top-$k$ eigenvectors, maximizing the trace encourages learning a subspace spanned by the top-$k$ eigenvectors, but does not recover the eigenvectors themselves. On the other hand, \Eqref{min_offdiag} places no preference on recovering large over small eigenvectors, but does enforce the columns of $\hat{V}$ to actually be eigenvectors. The preceding exercise is intended to introduce minimizing the off-diagonal terms of $R$ as a possible complementary objective for solving top-$k$ PCA. Next, we will use these two objectives to construct utility functions for each eigenvector $\hat{v}_i$.

We want to combine the objectives to take advantage of both their strengths. A valid proposal is
\begin{align}
    \max_{\hat{V}^\top \hat{V} = I} \sum_{i} R_{ii} - \sum_{i \ne j} R_{ij}^2. \label{valid_approach}
\end{align}
However, this objective ignores the natural hierarchy of the top-$k$ eigenvectors. For example, $\hat{v}_1$ is penalized for aligning with $\hat{v}_k$ and vice versa, but $\hat{v}_1$, being the estimate of the largest eigenvector, should be free to search for the direction that captures the most variance independent of the locations of the other vectors. Instead, first consider solving for the top-$1$ eigenvector, $v_1$, in which case $R=\begin{bmatrix} \langle \hat{v}_1, M \hat{v}_1 \rangle \end{bmatrix}$ is a $1 \times 1$ matrix. In this setting, \Eqref{min_offdiag} is not applicable because there are no off-diagonal elements, so $\max_{\hat{v}_1^\top \hat{v}_1 = 1} \langle \hat{v}_1, M \hat{v}_1 \rangle$ is a sensible utility function for $\hat{v}_1$.

If considering the top-$2$ eigenvectors, $\hat{v}_1$'s utility remains as before, and we introduce a new utility for $\hat{v}_2$. \Eqref{min_offdiag} is now applicable, so $\hat{v}_2$'s utility is
\begin{align}
    &\max_{\hat{v}_2^\top \hat{v}_2=1, \textcolor{green}{\hat{v}_1^\top \hat{v}_2=0}} \langle \hat{v}_2, M \hat{v}_2 \rangle - \frac{\langle \hat{v}_2, M \hat{v}_1 \rangle^2}{\langle \hat{v}_1, M \hat{v}_1 \rangle} \label{vec_2}
\end{align}
where we have divided the off-diagonal penalty by $\langle v_1, M v_1 \rangle$ so a) the two terms in \Eqref{vec_2} are on a similar scale and b) for reasons that ease analysis. Additionally note that the constraint $\textcolor{green}{\hat{v}_1^\top \hat{v}_2=0}$ may be redundant at the optimum $(\hat{v}^*_1=v_1,\hat{v}^*_2=v_2)$ because the second term, $\langle \hat{v}^*_2, M \hat{v}^*_1 \rangle^2 = \langle v_2, M v_1 \rangle^2 = \Lambda_{11}^2 \langle v_2, v_1 \rangle^2$, already penalizes such deviations ($\Lambda_{ii}$ is the $i$th largest eigenvector). These reasons motivate the following set of objectives (utilities), one for each vector $i \in \{1, \ldots, k\}$:
\begin{align}
    \max_{\hat{v}_i^\top \hat{v}_i = 1} \Big\{ u_i(\hat{v}_i \vert \hat{v}_{j < i}) &= \hat{v}_i^\top M \hat{v}_i - \sum_{j < i} \frac{(\hat{v}_i^\top M \hat{v}_j)^2}{\hat{v}_j^\top M \hat{v}_j} = || X \hat{v}_i ||^2 - \sum_{j < i} \frac{\langle X \hat{v}_i, X \hat{v}_j \rangle^2}{\langle X \hat{v}_j, X \hat{v}_j \rangle} \Big\} \label{obj}
\end{align}
where the notation $u_i(a_i \vert b)$ emphasizes that player $i$ adjusts $a_i$ to maximize a utility conditioned on $b$.

It is interesting to note that by incorporating knowledge of the natural hierarchy (see Figure~\ref{fig:dag}), we are immediately led to constructing asymmetric utilities, and thereby, inspired to formulate the PCA problem as a game, rather than a direct optimization problem as in~\Eqref{valid_approach}.
\begin{figure}[!t]
    \centering
    \includegraphics[scale=0.25]{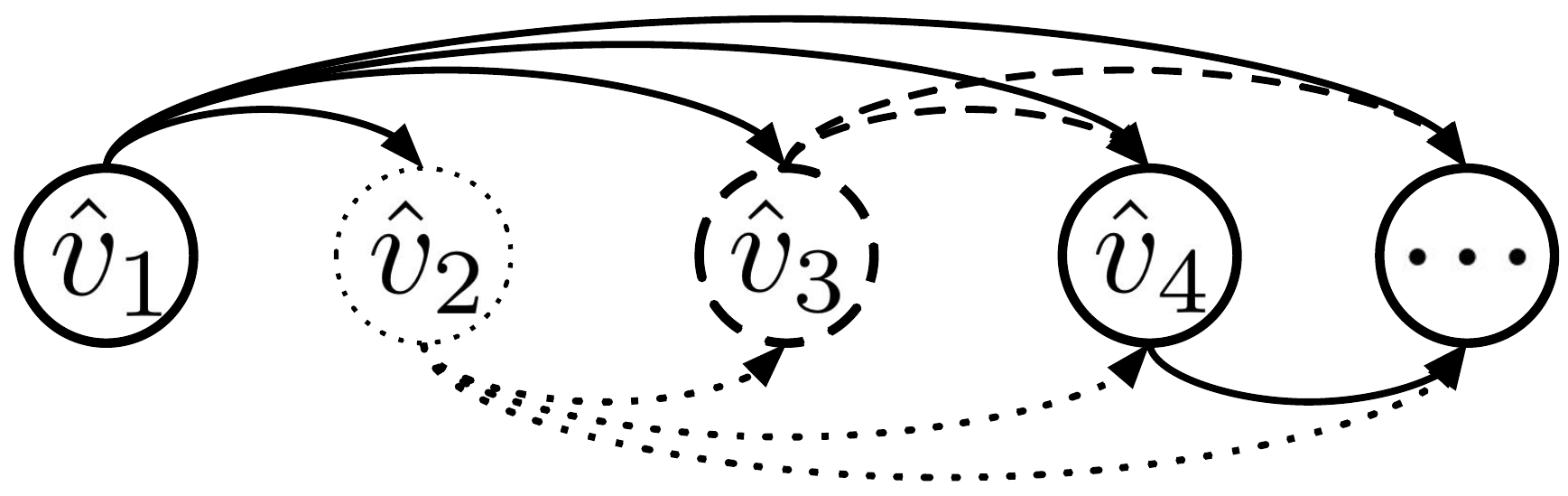}
    \vspace{-5pt}
    \caption[Each player $i$'s utility function depends on its parents represented here by a directed acyclic graph. Each parent must broadcast its vector, ``location'', down the hierarchy in a fixed order.]{Each player $i$'s utility function depends on its parents represented here by a directed acyclic graph. Each parent must broadcast its vector, ``location'', down the hierarchy in a fixed order.}
    \label{fig:dag}
\end{figure}

A key concept in games is a Nash equilibrium. A Nash equilibrium specifies a variable for each player from which no player can unilaterally deviate and improve their outcome. In this case, $\hat{V}$ is a (strict-)Nash equilibrium if and only if for all $i$, $u_i(\hat{v}_i \vert \hat{v}_{j<i}) > u_i(z_i \vert \hat{v}_{j<i})$ for all $z_i \in \mathcal{S}^{d-1}$.

\begin{theorem}[\textbf{PCA Solution is the Unique strict-Nash Equilibrium}]
\label{eigvec_opt}
Assume that the top-$k$ eigenvalues of $X^\top X$ are positive and distinct. Then the top-$k$ eigenvectors form the unique strict-Nash equilibrium of the proposed game in \Eqref{obj}.\footnotemark \emph{ The proof is deferred to Appendix~\ref{sec:appendix_nash}.}
\end{theorem}

Solving for the Nash of a game is difficult in general. Specifically, it belongs to the class of PPAD-complete problems~\citep{gilboa1989nash,daskalakis2009complexity}. However, because the game is hierarchical and each player's utility only depends on its parents, it is possible to construct a sequential algorithm that is convergent by solving each player's optimization problem in sequence.

\footnotetext{Unique up to a sign change; this is expected as both $v_i$ and $-v_i$ represent the same principal component.}

\section{Method}
\paragraph{Utility gradient.}
In Section~\ref{sec:utility}, we mentioned that normalizing the penalty term from \Eqref{vec_2} had a motivation beyond scaling. Dividing by $\langle \hat{v}_j, M \hat{v}_j \rangle$ results in the following gradient for player $i$:
\begin{align}
    \nabla_{\hat{v}_i} u_i(\hat{v}_i \vert \hat{v}_{j < i}) &= 2M \Big[ \underbrace{\hat{v}_i - \sum_{j < i} \frac{\hat{v}_i^\top M \hat{v}_j}{\hat{v}_j^\top M \hat{v}_j} \hat{v}_j}_{\text{generalized Gram-Schmidt}} \Big] = 2X^\top \Big[ X \hat{v}_i - \sum_{j < i} \frac{\langle X \hat{v}_i, X \hat{v}_j \rangle}{\langle X \hat{v}_j, X \hat{v}_j \rangle} X \hat{v}_j \Big]. \label{ui_grad}
\end{align}

The resulting gradient with normalized penalty term has an intuitive meaning. It consists of a single generalized Gram-Schmidt step followed by the standard matrix product found in power iteration and Oja's rule. Also, notice that applying the gradient as a fixed point operator in sequence ($\hat{v}_i \leftarrow \frac{1}{2} \nabla_{\hat{v}_i} u_i(\hat{v}_i \vert \hat{v}_{j < i})$) on $M=I$ recovers the standard Gram-Schmidt procedure for orthogonalization.

\paragraph{A sequential algorithm.}
\begin{wrapfigure}{r}{0.35\textwidth}
    \vspace{-.85cm}
    \begin{center}
    \includegraphics[width=0.98\textwidth]{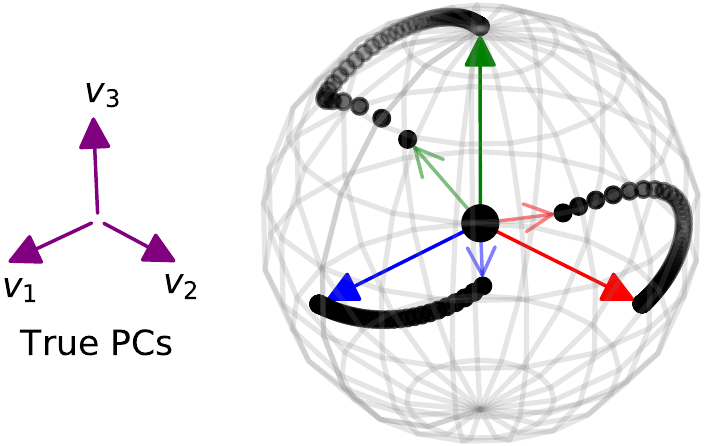}
    \end{center}
    \vspace{-10pt}
    \caption[]{\pcagame{} guides each $\hat{v}_i$ along the unit-sphere from
    \begin{tikzpicture}
    \draw[-{>[scale=1.0]}] (0,0) -- (0,.3);
    \end{tikzpicture}
    to
    \begin{tikzpicture}
    \draw[-{Latex[scale=1.0]}] (0,0) -- (0,.3);
    \end{tikzpicture} in parallel; $M = \texttt{diag}([3, 2, 1])$.}
    \label{eigengame_viz}
\end{wrapfigure}
Each eigenvector can be learned by maximizing its utility. The vectors are constrained to the unit sphere, a non-convex \emph{Riemannian} manifold, so we use \emph{Riemmanian} gradient ascent with gradients given by~\Eqref{ui_grad}. In this case, Riemannian optimization theory simply requires an intermediate step where the gradient, $\nabla_{\hat{v}_i}$, is projected onto the tangent space of the sphere to compute the Riemannian gradient, $\nabla^R_{\hat{v}_i}$. A more detailed illustration can be found in Appendix~\ref{gradr_instability}. Recall that each $u_i$ depends on $\hat{v}_{j < i}$. If any of $\hat{v}_{j < i}$ are being learned concurrently, then $\hat{v}_i$ is maximizing a non-stationary objective which makes a convergence proof difficult. Instead, for completeness, we prove convergence assuming each $\hat{v}_i$ is learned in sequence. \Algref{pcagame_ascent_successive} learns $\hat{v}_i$ given fixed parents $\hat{v}_{j < i}$; we present the convergence guarantee in Section~\ref{sec:theory} and details on setting $\rho_i$ and $\alpha$ in Appendix~\ref{app:conv}.

\begin{minipage}{0.49\textwidth}
    \begin{algorithm}[H]
    \begin{algorithmic}
        \State Given: matrix $X \in \mathbb{R}^{n \times d}$, maximum error tolerance $\rho_i$, initial vector $\hat{v}_i^0 \in \mathcal{S}^{d-1}$, learned approximate parents $\hat{v}_{j < i}$, and step size $\alpha$.
        \State $\hat{v}_i \leftarrow \hat{v}_i^0$
        \State $t_i = \lceil \frac{5}{4} \min(||\nabla_{\hat{v}_i^0} u_i|| / 2, \rho_i)^{-2} \rceil$
        \For{$t = 1: t_i$}
            \State $\texttt{rewards} \leftarrow X \hat{v}_i$
            \State $\texttt{penalties} \leftarrow \sum_{j < i} \frac{\langle X \hat{v}_i, X \hat{v}_j \rangle}{\langle X \hat{v}_j, X \hat{v}_j \rangle} X \hat{v}_j$
            \State $\nabla_{\hat{v}_i} \leftarrow 2 X^\top \Big[ \texttt{rewards} - \texttt{penalties} \Big]$
            \State $\nabla^R_{\hat{v}_i} \leftarrow \nabla_{\hat{v}_i} - \langle \nabla_{\hat{v}_i}, \hat{v}_i \rangle \hat{v}_i$
            \State $\hat{v}_i' \leftarrow \hat{v}_i + \alpha \nabla^R_{\hat{v}_i}$
            \State $\hat{v}_i \leftarrow \frac{\hat{v}_i'}{|| \hat{v}_i' ||}$
        \EndFor
        \State return $\hat{v}_i$
    \end{algorithmic}
    \caption{\pcagame{}$^R$-Sequential}
    \label{pcagame_ascent_successive}
    \end{algorithm}
\end{minipage}
\begin{minipage}{0.49\textwidth}
    \begin{algorithm}[H]
    \begin{algorithmic}
        \State Given: stream, $X_t \in \mathbb{R}^{m \times d}$, total iterations $T$, initial vector $\hat{v}_i^0 \in \mathcal{S}^{d-1}$, and step size $\alpha$.
        \State $\hat{v}_i \leftarrow \hat{v}_i^0$
        \For{$t = 1: T$}
            \State $\texttt{rewards} \leftarrow X_t \hat{v}_i$
            \State $\texttt{penalties} \leftarrow \sum_{j < i} \frac{\langle X_t \hat{v}_i, X_t \hat{v}_j \rangle}{\langle X_t \hat{v}_j, X_t \hat{v}_j \rangle} X_t \hat{v}_j$
            \State $\nabla_{\hat{v}_i} \leftarrow 2 X_t^\top \Big[ \texttt{rewards} - \texttt{penalties} \Big]$
            \State $\nabla^R_{\hat{v}_i} \leftarrow \nabla_{\hat{v}_i} - \langle \nabla_{\hat{v}_i}, \hat{v}_i \rangle \hat{v}_i$
            \State $\hat{v}_i' \leftarrow \hat{v}_i + \alpha \nabla^R_{\hat{v}_i}$
            \State $\hat{v}_i \leftarrow \frac{\hat{v}_i'}{|| \hat{v}_i' ||}$
            \State \texttt{broadcast}($\hat{v}_{i}$)
        \EndFor
        \State return $\hat{v}_i$
    \end{algorithmic}
    \caption{\pcagame{}$^R$ (\pcagame{}\textemdash update with $\nabla_{\hat{v}_i}$ instead of $\nabla^R_{\hat{v}_i}$)}
    \label{pcagame_ascent}
    \end{algorithm}
\end{minipage}

\paragraph{A decentralized algorithm.} While \Algref{pcagame_ascent_successive} enjoys a convergence guarantee, learning every parent $\hat{v}_{j < i}$ before learning $\hat{v}_i$ may be unnecessarily restrictive. Intuitively, as parents approach their respective optima, they become quasi-stationary, so we do not expect maximizing utilities in parallel to be problematic in practice. To this end, we propose \Algref{pcagame_ascent} visualized in Figure~\ref{eigengame_viz}.

In practice we can assign each eigenvector update to its own device (e.g.\ a GPU or TPU).  Systems with fast interconnects may facilitate tens, hundreds or thousands of accelerators to be used. In such settings, the overhead of \texttt{broadcast}($\hat{v}_{i}$) is minimal. We can also specify that the data stream is co-located with the update so $\hat{v}_{i}$ updates with respect to its own $X_{i,t}$. This is a standard paradigm for e.g.\ data-parallel distributed neural network training. We provide further details in Section~\ref{sec:experiments}.

\paragraph{Message Passing on a DAG.}
Our proposed utilities enforce a strict hierarchy on the eigenvectors. This is a simplification that both eases analysis (see Appendix~\ref{no_hier}) and improves convergence\footnote{\pcagame{} sans order learns max 1 PC and sans order+normalization 5 PCs on data in Figure~\ref{fig:synthetic_results}\subref{fig:synth}.}, however, it is not optimal. We assume vectors are initialized randomly on the sphere and, for instance, $\hat{v}_k$ may be initialized closer to $v_1$ than even $\hat{v}_1$ and vice versa. The hierarchy shown in Figure~\ref{fig:dag} enforces a strict graph structure for broadcasting information of parents to the childrens' utilities.

To our knowledge, our utility formulation in~\Eqref{obj} is novel. One disadvantage is that stochastic gradients of~\Eqref{ui_grad} are biased. This is mitigated with large batch sizes (further discussion in Appendix~\ref{gradient_bias}).

\begin{figure}[!t]
    \centering
    \begin{subfigure}[b]{.49\textwidth}
    \includegraphics[width=0.49\textwidth]{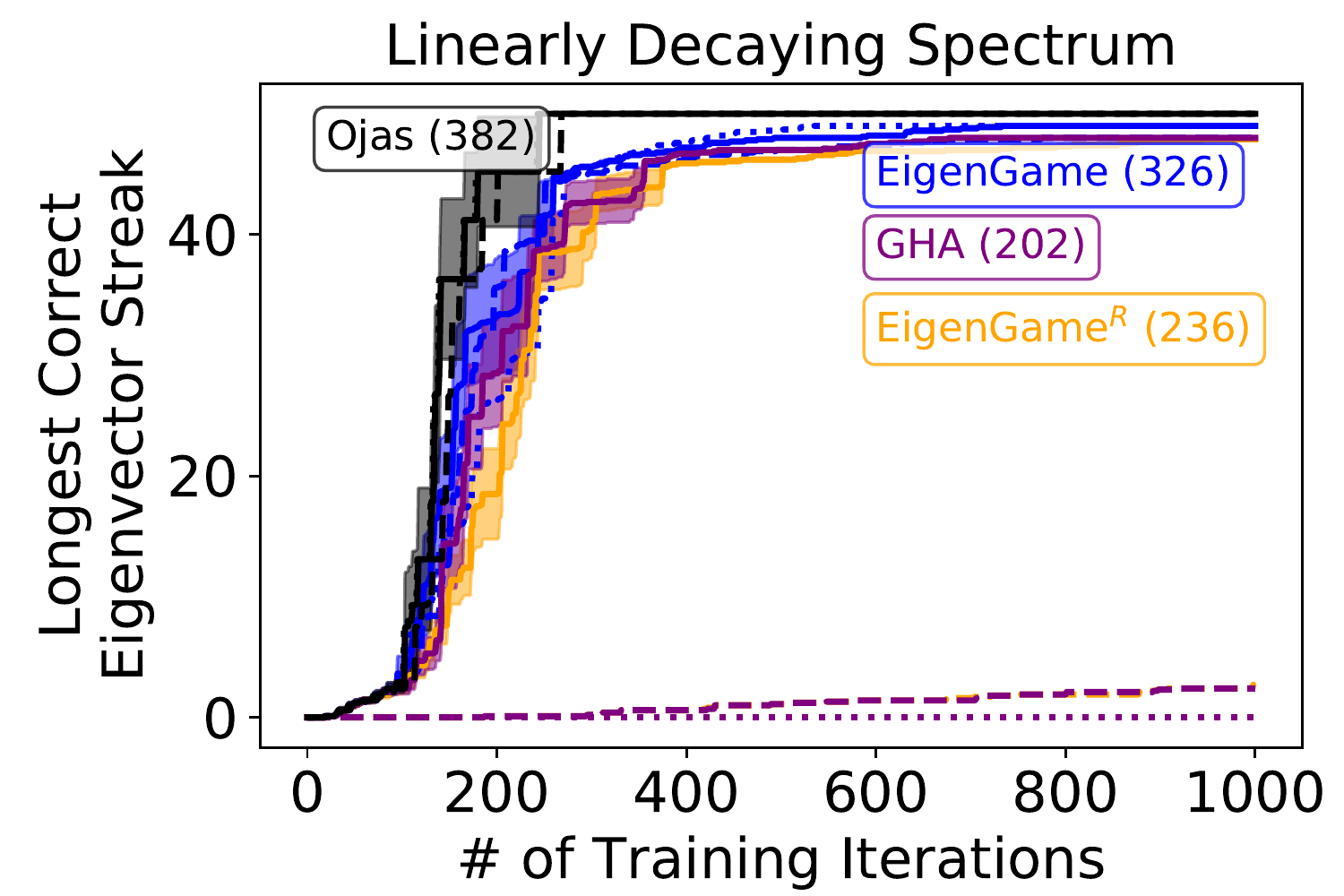}
    \includegraphics[width=0.49\textwidth]{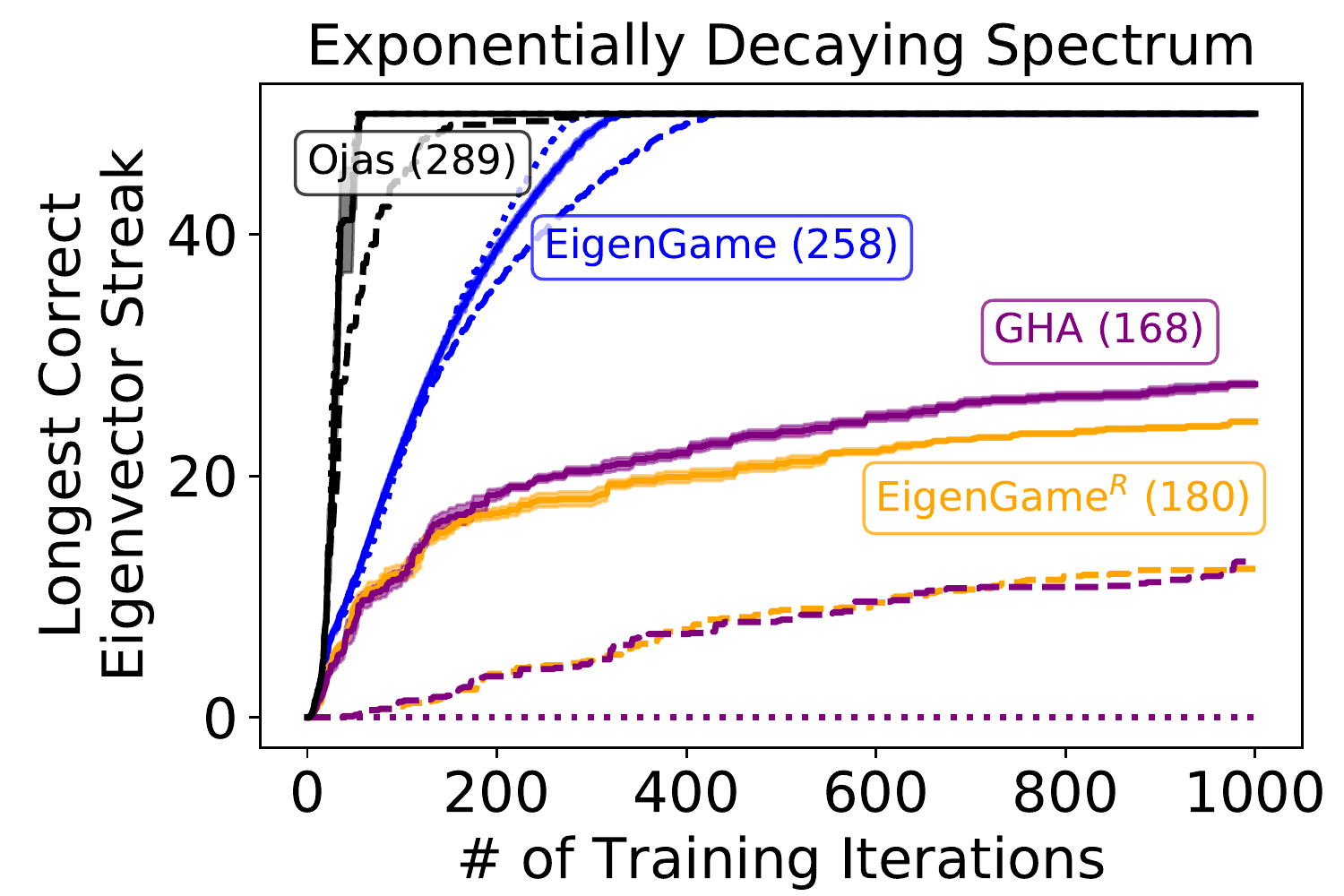}
    \caption{Synthetic data \label{fig:synth}}
    \end{subfigure}
    \begin{subfigure}[b]{.49\textwidth}
    \includegraphics[width=0.49\textwidth]{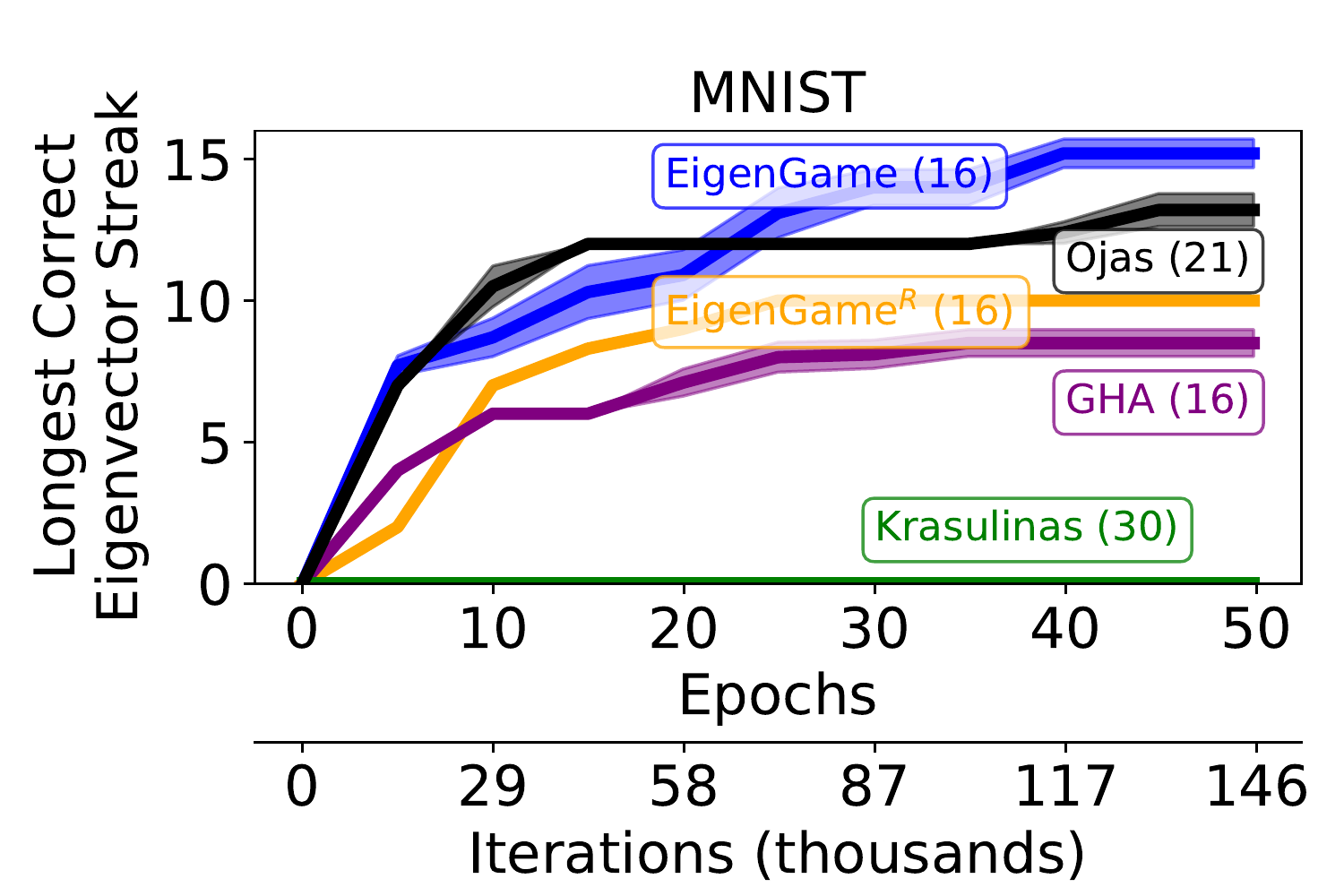}
    \includegraphics[width=0.49\textwidth]{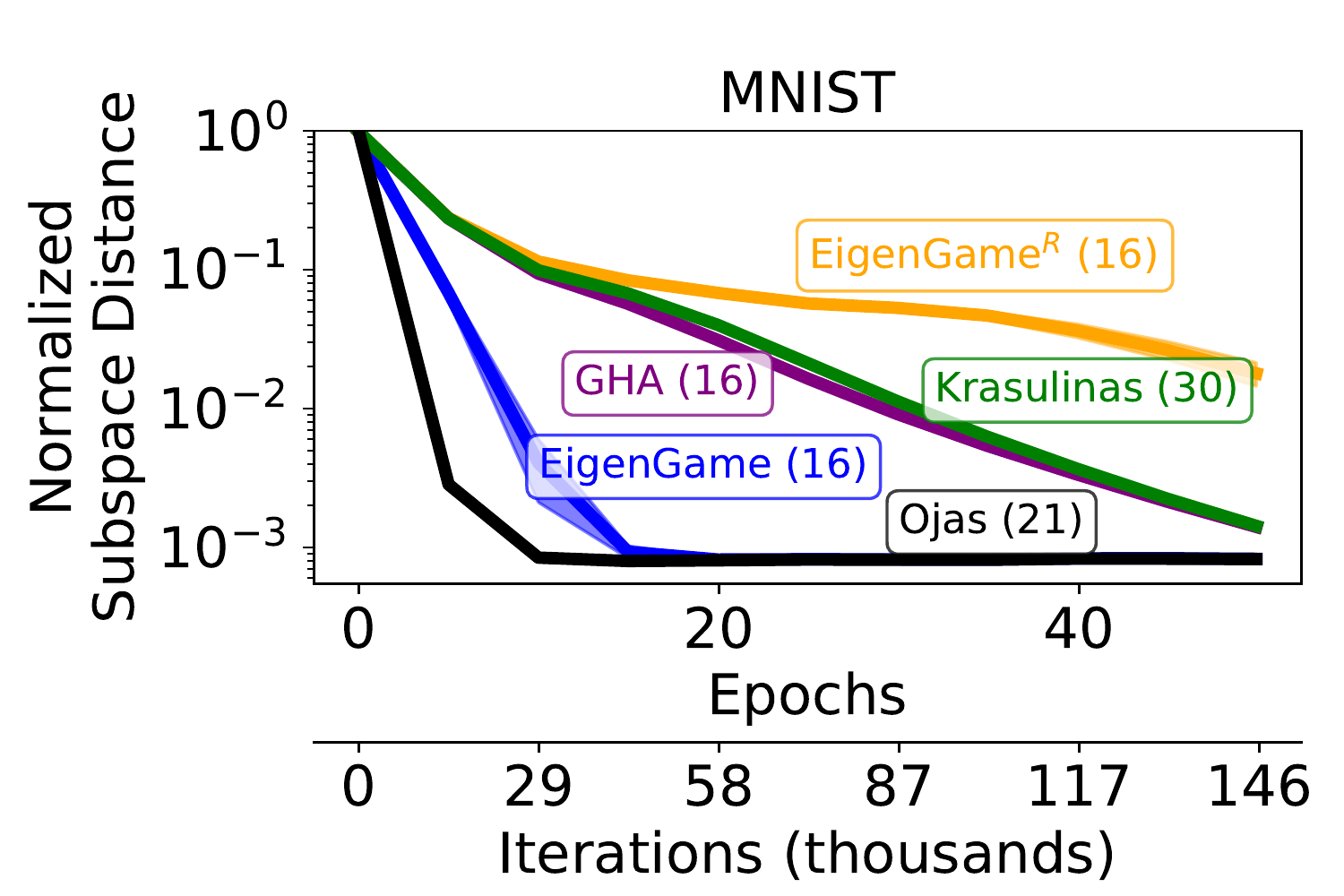}
    \caption{MNIST \label{fig:mnist}}
    \end{subfigure}
    \vspace{-5pt}
    \caption[(\subref{fig:synth}) The longest streak of consecutive vectors with angular error less than $\frac{\pi}{8}$ radians is plotted versus algorithm iterations for a matrix $M \in \mathbb{R}^{50 \times 50}$ with a spectrum decaying from $1000$ to $1$ linearly and exponentially. Average runtimes are reported in milliseconds next to the method names. We omit Krasulina's as it is only designed to find the top-$k$ subspace. Both \pcagame{} variants and GHA achieve similar asymptotes on the linear spectrum. (\subref{fig:mnist}) Longest streak and subspace distance on MNIST with average runtimes reported in seconds. (\subref{fig:synth},\subref{fig:mnist}) Learning rates were chosen from $\{10^{-3},\ldots,10^{-6}\}$ on $10$ held out runs. Solid lines denote results with the best performing learning rate. Dotted and dashed lines denote results using the best learning rate $\times$ $10$ and $0.1$. All plots show means over $10$ trials. Shaded regions highlight $\pm$ standard error of the mean for the best performing learning rates.]{(\subref{fig:synth}) The longest streak of consecutive vectors with angular error less than $\frac{\pi}{8}$ radians is plotted versus algorithm iterations for a matrix $M \in \mathbb{R}^{50 \times 50}$ with a spectrum decaying from $1000$ to $1$ linearly and exponentially. Average runtimes are reported in milliseconds next to the method names\footnotemark. We omit Krasulina's as it is only designed to find the top-$k$ subspace. Both \pcagame{} variants and GHA achieve similar asymptotes on the linear spectrum. (\subref{fig:mnist}) Longest streak and subspace distance on MNIST with average runtimes reported in seconds. (\subref{fig:synth},\subref{fig:mnist}) Learning rates were chosen from $\{10^{-3},\ldots,10^{-6}\}$ on $10$ held out runs. Solid lines denote results with the best performing learning rate. Dotted and dashed lines denote results using the best learning rate $\times$ $10$ and $0.1$. All plots show means over $10$ trials. Shading highlights $\pm$ standard error of the mean for the best learning rates.}
    \label{fig:synthetic_results}
\end{figure}
\footnotetext{\pcagame{} runtimes are longer than those of \pcagame{}$^R$ in the synthetic experiments despite strictly requiring fewer FLOPS; apparently this is due to low-level floating point arithmetic specific to the experiments.}

\section{Convergence of \pcagame{}}
\label{sec:theory}
Here, we first show that \Eqref{obj} has a simple form such that any local maximum of $u_i$ is also a global maximum. Player $i$'s utility depends on its parents, so we next explain how error in the parents propagates to children through mis-specification of player $i$'s utility. Using the first result and accounting for this error, we are then able to give global, finite-sample convergence guarantees in the full-batch setting by leveraging recent non-convex Riemannian optimization theory.

\paragraph{The utility landscape and parent-to-child error propagation.}
\Eqref{obj} is abstruse, but we prove that the shape of player $i$'s utility is simply sinusoidal in the angular deviation of $\hat{v}_i$ from the optimum. The amplitude of the sinusoid varies with the direction of the angular deviation along the unit-sphere and is dependent on the accuracy of players $j < i$. In the special case where players $j < i$ have learned the top-$(i-1)$ eigenvectors exactly, player $i$'s utility simplifies (see Lemma~\ref{player_i_obj}) to
\begin{align}
    u_i(\hat{v}_i, \{v_{j < i}\}) &= \Lambda_{ii} - \sin^2(\theta_i) \Big( \Lambda_{ii} - \sum_{l > i} z_l \Lambda_{ll} \Big) \label{sinusoidal}
\end{align}
where $\theta_i$ is the angular deviation and $z \in \Delta^{d-1}$ parameterizes the deviation direction.
Note that $\sin^2$ has period $\pi$ instead of $2\pi$, which simply reflects the fact that $v_i$ and $-v_i$ are both eigenvectors.

An error propagation analysis reveals that it is critical to learn the parents to a given degree of accuracy. The angular distance between $v_i$ and the maximizer of player $i$'s utility with approximate parents has $\tan^{-1}$ dependence (i.e., a soft step-function; see Lemma~\ref{arctan_error_prop} and Figure~\ref{fig:arctan_err_dep} in Appendix~\ref{app:err_prop}).

\begin{theorem}[\textbf{Global convergence}]
\label{conv_theorem}
\Algref{pcagame_ascent_successive} achieves finite sample convergence to within $\theta_{tol}$ angular error of the top-$k$ principal components, \textbf{independent of initialization}. Furthermore, if each $\hat{v}_i$ is initialized to within $\frac{\pi}{4}$ of $v_i$, \Algref{pcagame_ascent_successive} returns the components with angular error less than $\theta_{\text{tol}}$ in $T = \Big\lceil \mathcal{O} \Big( k \Big[ \frac{(k-1)!}{\theta_{\text{tol}}} \prod\limits_{i=1}^k \big( \frac{16\Lambda_{11}}{g_{i}} \big) \Big]^2 \Big) \Big\rceil$ iterations. \emph{Proofs are deferred to Appendices~\ref{ui_suff_conv} and~\ref{simp_conv_rate}.}
\end{theorem}

Angular error is defined as the angle between $\hat{v}_i$ and $v_i$: $\theta_i = \sin^{-1}(\sqrt{1-\langle v_i, \hat{v}_i \rangle^2})$. The first $k$ in the formula for $T$ appears from a naive summing of worst case bounds on the number of iterations required to learn each $\hat{v}_{j < k}$ individually. The constant $16$ arises from the error propagation analysis; parent vectors, $\hat{v}_{j < i}$, must be learned to under $1/16$th of a canonical error threshold, $\frac{g_i}{(i-1)\Lambda_{11}}$, for the child $\hat{v}_i$ where $g_i = \Lambda_{ii} - \Lambda_{i+1,i+1}$. The Riemannian optimization theory we leverage dictates that $\frac{1}{\rho^2}$ iterations are required to meet a $\mathcal{O}(\rho)$ error threshold. This is why the squared inverse of the error threshold appears here. Breaking down the error threshold itself, the ratio $\Lambda_{11}/g_i$ says that more iterations are required to distinguish eigenvectors when the difference between them (summarized by the gap $g_i$) is small relative to the scale of the spectrum, $\Lambda_{11}$. The $(k-1)!$ term appears because learning smaller eigenvectors requires learning a much more accurate $\hat{v}_1$ higher up the DAG.

Lastly, the utility function for each $\hat{v}_i$ is sinusoidal, and it is possible that we initialize $\hat{v}_i$ with initial utility arbitrarily close to the trough (bottom) of the function where gradients are arbitrarily small. This is why the global convergence rate depends on the initialization in general. Note that \Algref{pcagame_ascent_successive} effectively detects the trough by measuring the norm of the initial gradient ($\nabla_{\hat{v}_i^0} u_i$) and scales the number of required iterations appropriately. A complete theorem that considers the probability of initializing $\hat{v}_i$ within $\frac{\pi}{4}$ of $v_i$ is in Appendix~\ref{app:conv}, but this possibility shrinks to zero in high dimensions. 

We would also like to highlight that these theoretical findings are strong relative to some other claims. For example, the exponential convergence guarantee for Matrix Krasulina requires the initial guess at the eigenvectors capture the top-$(k-1)$ subspace~\citep{tang2019exponentially}, unlikely when $d \gg k$. A similar condition is required in~\citep{shamir2016fast}. These guarantees are given for the mini-batch setting while ours is for the full-batch, however, we provide global convergence without restrictions on initialization.

\section{Related work}
\label{related_work}
PCA is a century-old problem and a massive literature exists \citep{jolliffe2002principal,golub2012matrix}. The standard solution to this problem is to compute the SVD, possibly combined with randomized algorithms, to recover the top-$k$ components as in~\citep{halko2011finding} or with Frequent Directions~\citep{ghashami2016frequent} which combines sketching with SVD.

In neuroscience, Hebb's rule~\citep{hebb2005organization} refers to a connectionist rule that solves for the top eigenvector of a matrix $M$ using additive updates of a vector $v$ as  $v \leftarrow v + \eta M v$. Likewise, Oja's rule~\citep{oja1982simplified, shamir2015stochastic} refers to a similar update $v \leftarrow v + \eta (I - vv^\top) Mv$. In machine learning, using a normalization step of $v \leftarrow v/||v||$ with Hebb's rule is somewhat confusingly referred to as Oja's algorithm~\citep{shamir2015stochastic}, the reason being that the subtractive term in Oja's rule can be viewed as a regularization term for implicitly enforcing the normalization. In the limit of infinite step size, $\eta \rightarrow \infty$, Oja's algorithm effectively becomes the well known Power method. If a normalization step is added to Oja's rule, this is referred to as Krasulina's algorithm~\citep{krasulina1969method}. In the language of Riemannian manifolds, $v/||v||$ can be recognized as a retraction and $(I - vv^\top)$ as projecting the gradient $Mv$ onto the tangent space of the sphere~\citep{absil2009optimization}.

Many of the methods above have been generalized to the top-$k$ components. Most generalizations involve adding an orthonormalization step after each update, typically accomplished with a \qr{} factorization plus some minor sign accounting (e.g., see \Algref{alg:ojas} in Appendix~\ref{oja_disambig}). 
An extension of Krasulina's algorithm to the top-$k$ setting, termed Matrix Krasulina~\citep{tang2019exponentially}, was recently proposed in the machine learning literature. This algorithm can be recognized as projecting the gradient onto the Stiefel manifold (the space of orthonormal matrices) followed by a \qr{} step to maintain orthonormality, which is a well known retraction.

Maintaining orthonormality via \qr{} is computationally expensive. \citet{amid2019implicit} propose an alternative Krasulina method which does not require re-orthonormalization but instead requires inverting a $k \times k$ matrix; in a streaming setting restricted to minibatches of size $1$ ($X_t \in \sR^{d}$), Sherman-Morrison~\citep{golub2012matrix} can be used to efficiently replace the inversion step. \citet{raja2020distributed} develop a data-parallel distributed algorithm for the top eigenvector. Alternatively, the Jacobi eigenvalue algorithm explicitly represents the matrix of eigenvectors as a Givens rotation matrix using $\sin$'s and $\cos$'s and rotates $M$ until it is diagonal~\citep{golub2000eigenvalue}.

In contrast, other methods extract the top components in sequence by solving for the $i$th component using an algorithm such as power iteration or Oja's, and then enforcing orthogonality by removing the learned subspace from the matrix, a process known as \emph{deflation}. Alternatively, the deflation process may be intertwined with the learning of the top components. The generalized Hebbian algorithm~\citep{sanger1989optimal} (GHA) works this way as do Lagrangian inspired formulations~\citep{ghojogh2019eigenvalue} as well as our own approach. We make the connection between GHA and our algorithm concrete in Prop.~\ref{gha_equiv_eigengame}. Note, however, that the GHA update is not the gradient of any utility (Prop.~\ref{gha_not_grad}) and therefore, lacks a clear game interpretation.

Of these, Oja's algorithm has arguably been the most extensively studied~\citep{shamir2016convergence,allen2017first}\footnote{See Table 1 in \citep{allen2017first}.} 
Note that Oja's algorithm converges to the actual principal components~\citep{allen2017first} and Matrix Krasulina~\citep{tang2019exponentially} converges to the top-$k$ subspace. However, neither can be obviously decentralized. GHA~\citep{sanger1989optimal} converges to the principal components asymptotically and can be decentralized~\citep{gang2019fast}. Each of these is applicable in the streaming $k$-PCA setting.

\section{Experiments} \label{sec:experiments}
We compare our approach against GHA, Matrix Krasulina, and Oja's algorithm\footnote{A detailed discussion of Frequent Directions~\citep{ghashami2016frequent} can be found in Appendix~\ref{freqdirs}.}. We present both \pcagame{} and \pcagame{}$^R$ which projects the gradient onto the tangent space of the sphere each step. We measure performance of methods in terms of principal component accuracy and subspace distance. We measure principal component accuracy by the number of consecutive components, or \emph{longest streak}, that are estimated within an angle of $\frac{\pi}{8}$ from ground truth. For example, if the angular errors of the $\hat{v}_i$'s returned by a method are, in order, $[\theta_1, \theta_2, \theta_3, \ldots] = [\frac{\pi}{16}, \frac{\pi}{4}, \frac{\pi}{10}, \ldots]$, then the method is credited with a streak of only $1$ regardless of the errors $\theta_{i > 2}$. For Matrix Krasulina, we first compute the optimal matching from $\hat{v}_i$ to ground truth before measuring angular error. 
We present the longest streak as opposed to “\# of eigenvectors found” because, in practice, no ground truth is available and we think the user should be able to place higher confidence in the larger eigenvectors being correct. If an algorithm returns $k$ vectors, $\frac{k}{2}$ of which are accurate components but does not indicate which, this is less helpful.
We measure normalized subspace distance using $1 - \frac{1}{k} \cdot \Tr(U^*P) \in [0, 1]$ where $U^*=VV^{\dagger}$ and $P=\hat{V}\hat{V}^{\dagger}$ similarly to~\cite{tang2019exponentially}.

\paragraph{Synthetic data.}
Experiments on synthetic data demonstrate the viability of our approach (Figure~\ref{fig:synth}). Oja’s algorithm performs best on synthetic experiments because strictly enforcing orthogonalization with an expensive \qr{} step greatly helps when solving for \textbf{all} eigenvectors. \pcagame{} is able to effectively parallelize this over $k$ machines and the advantage of \qr{} diminishes in Figure~\ref{fig:mnist}. The remaining algorithms perform similarly on a linearly decaying spectrum, however, \pcagame{} performs better on an exponentially decaying spectrum due possibly to instability of Riemannian gradients near the equilibrium (see Appendix~\ref{gradr_instability} for further discussion). GHA and \pcagame{}$^R$ are equivalent under specific conditions (see Proposition~\ref{gha_equiv_eigengame}).

Figure~\ref{fig:synth_extra} shows \pcagame{} solves for the eigenvectors up to a high degree of accuracy $\frac{\pi}{32}$, i.e. the convergence results in Figure~\ref{fig:synth} are not the result of using a \emph{loose} tolerance of $\frac{\pi}{8}$. With the lower tolerance, all algorithms take slightly more iterations to learn the eigenvectors of the linear spectrum; it is difficult to see any performance change for the exponential spectrum. Although Theorem~\ref{conv_theorem} assumes distinct eigenvalues, Figure~\ref{fig:bubble} supports the claim that \pcagame{} does not require distinct eigenvalues for convergence. We leave proving convergence in this setting to future work.

\begin{figure}[!t]
    \centering
    \begin{subfigure}[b]{.49\textwidth}
    \includegraphics[width=0.49\textwidth]{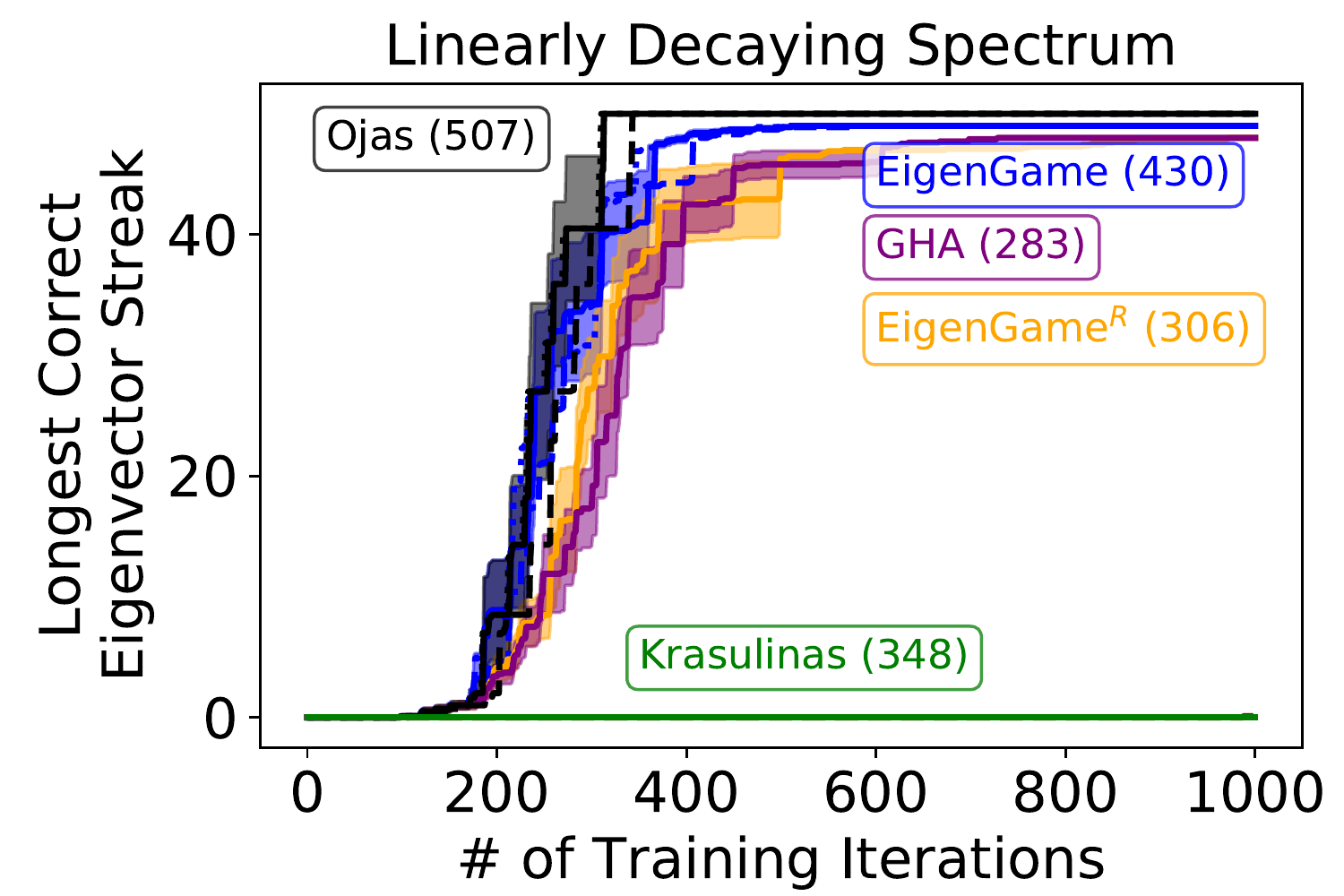}
    \includegraphics[width=0.49\textwidth]{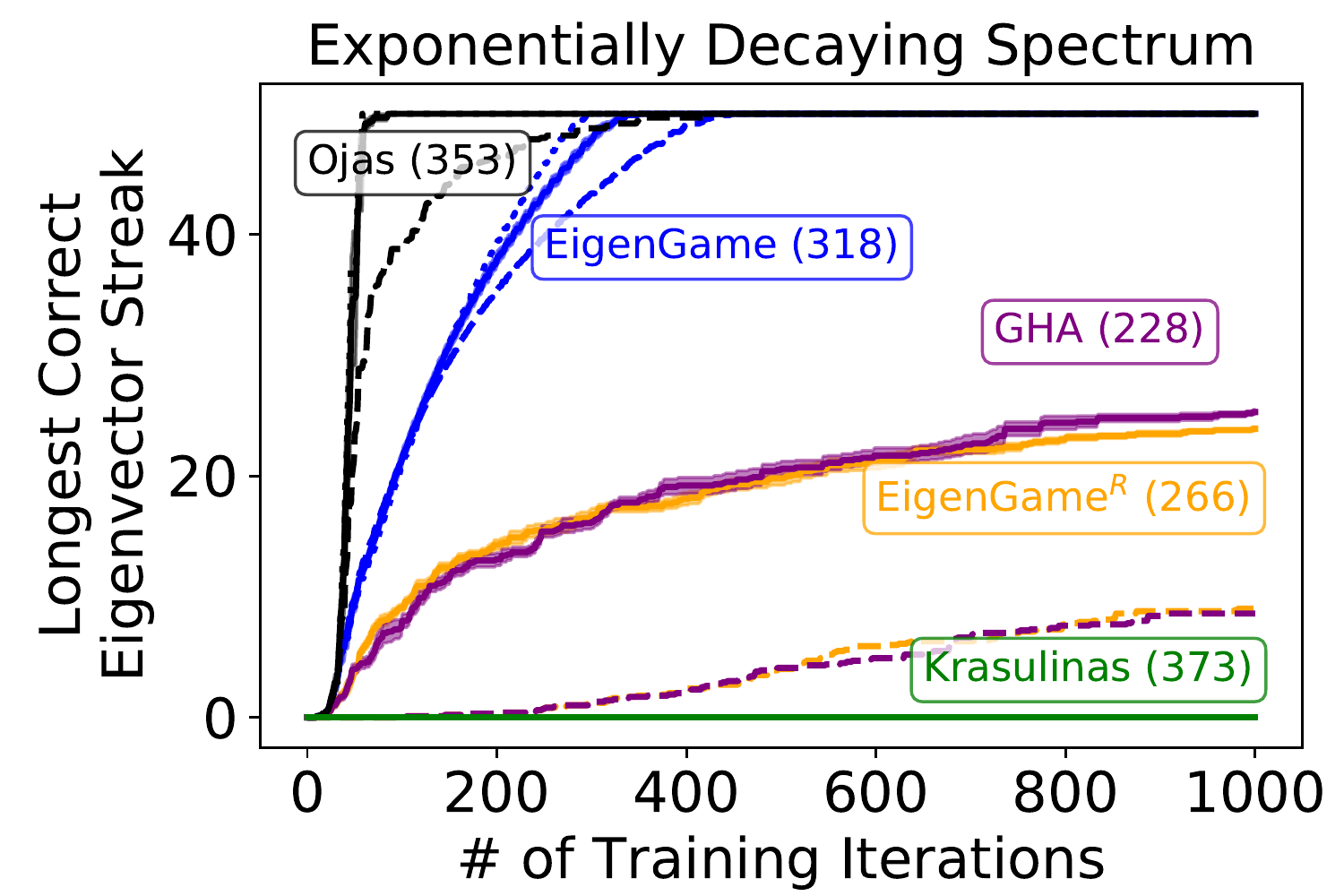}
    \caption{Synthetic data: Stricter Tolerance \label{fig:synth_extra}}
    \end{subfigure}
    \begin{subfigure}[b]{.49\textwidth}
    \includegraphics[width=0.49\textwidth]{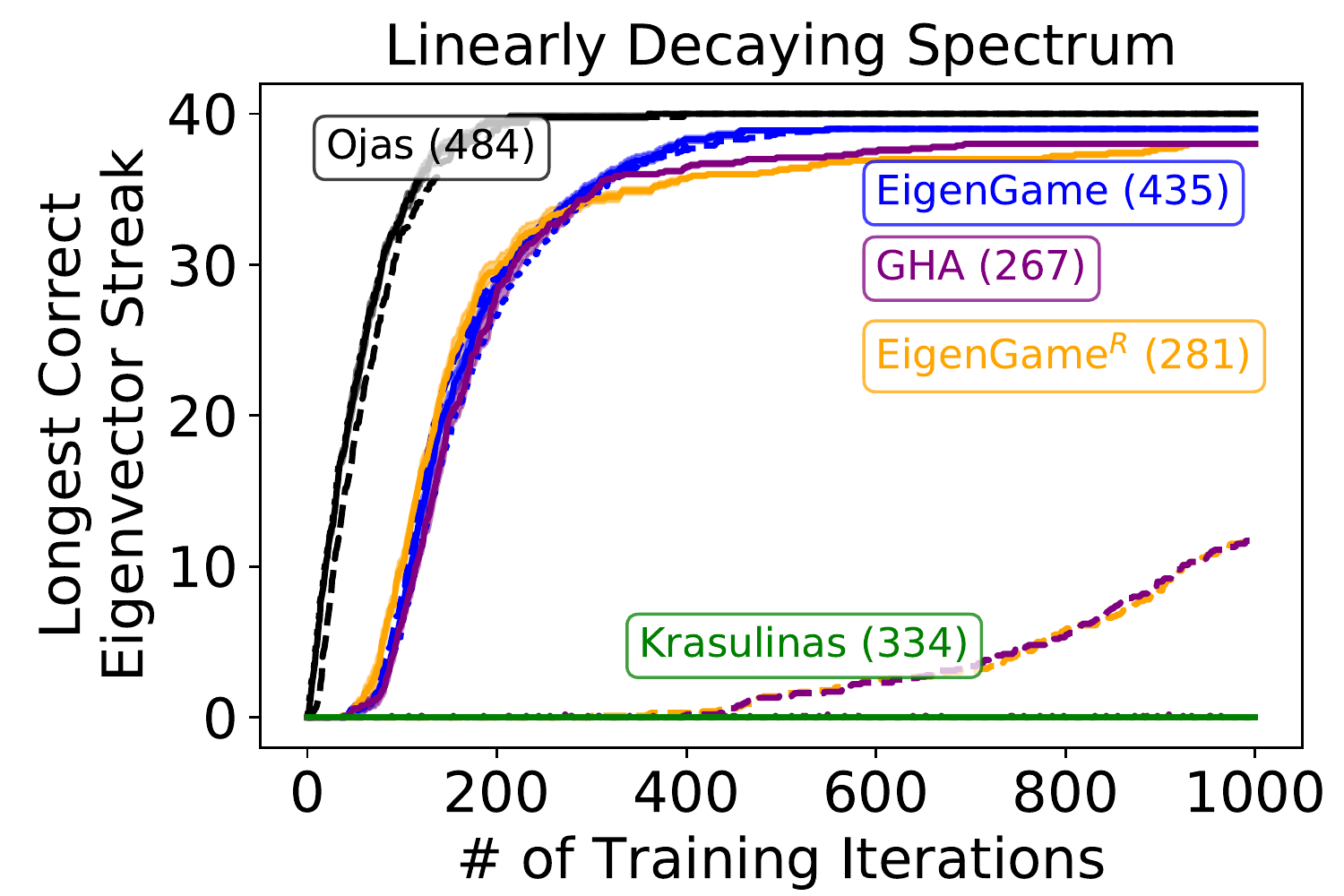}
    \includegraphics[width=0.49\textwidth]{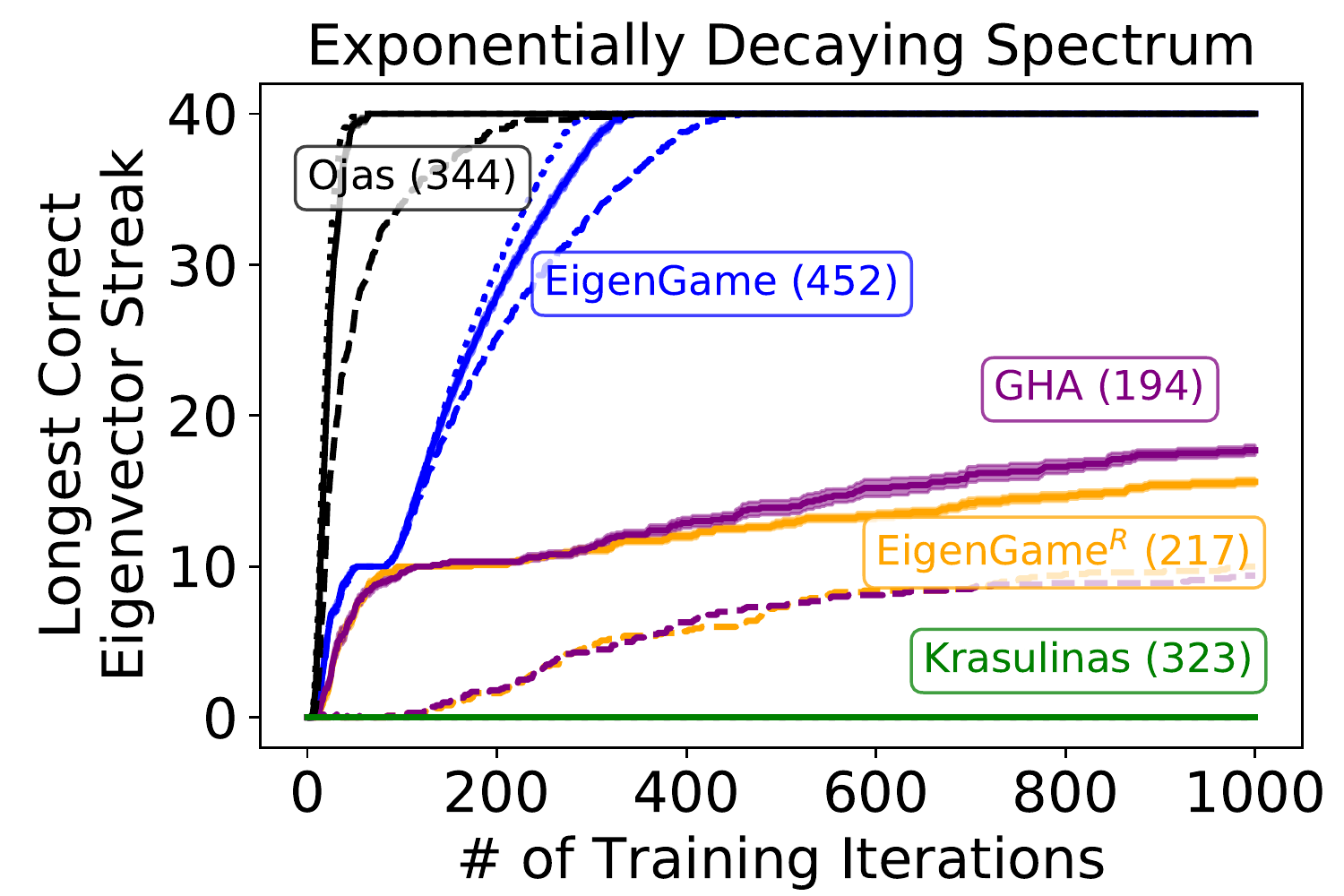}
    \caption{Synthetic: Repeated Eigenvalues \label{fig:bubble}}
    \end{subfigure}
    \vspace{-5pt}
    \caption[(\subref{fig:synth}) Repeats analysis of Figure~\ref{fig:synth} but for a lower angular tolerance of $\frac{\pi}{32}$. (\subref{fig:bubble}) Repeats analysis of Figure~\ref{fig:synth} with an angular tolerance of $\frac{\pi}{8}$ as before, but with eigenvalues $10-19$ of the ordered spectrum overwritten with $\lambda_{10}$ of the original spectrum. We compute angular error for the eigenvectors on either side of this ``bubble" to show that \pcagame{} finds these eigenvectors despite repeated eigenvalues in the spectrum.]{(\subref{fig:synth}) Repeats analysis of Figure~\ref{fig:synth} but for a lower angular tolerance of $\frac{\pi}{32}$. (\subref{fig:bubble}) Repeats analysis of Figure~\ref{fig:synth} with an angular tolerance of $\frac{\pi}{8}$ as before, but with eigenvalues $10-19$ of the ordered spectrum overwritten with $\lambda_{10}$ of the original spectrum. We compute angular error for the eigenvectors on either side of this ``bubble" to show that \pcagame{} finds these eigenvectors despite repeated eigenvalues in the spectrum; note $40/50$ is optimal in this experiment.}
    \label{fig:synthetic_results_xtra}
\end{figure}

\paragraph{\mnist{} handwritten digits.}
We compare \pcagame{} against GHA, Matrix Krasulina, and Oja's algorithm on the \mnist{} dataset (Figure~\ref{fig:mnist}). We flatten each image in the training set to obtain a $60,000\times 784$ dimensional matrix. \pcagame{} is competitive with Oja's in a high batch size regime (1024 samples per mini-batch). The performance gap between \pcagame{} and the other methods shrinks as the mini-batch size is reduced (see Appendix~\ref{gradient_bias}), expectedly due to biased gradients.

\begin{figure}
    \begin{subfigure}[b]{.54\textwidth}
        \centering
        \includegraphics[width=0.98\textwidth]{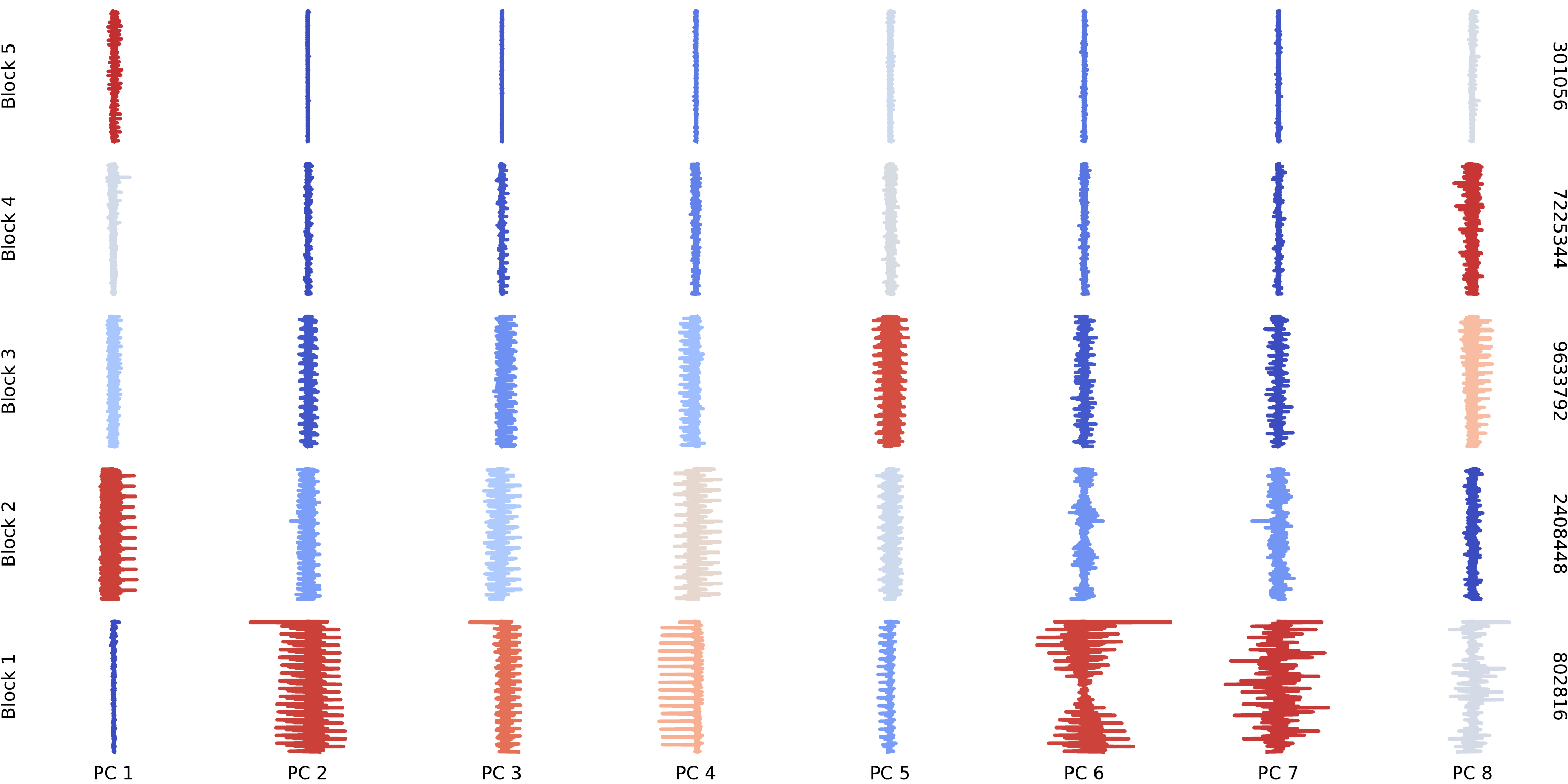}
        \caption{Principal Components \label{fig:imagenet}}
    \end{subfigure}
    \begin{subfigure}[b]{.45\textwidth}
        \centering
        \includegraphics[width=0.98\textwidth]{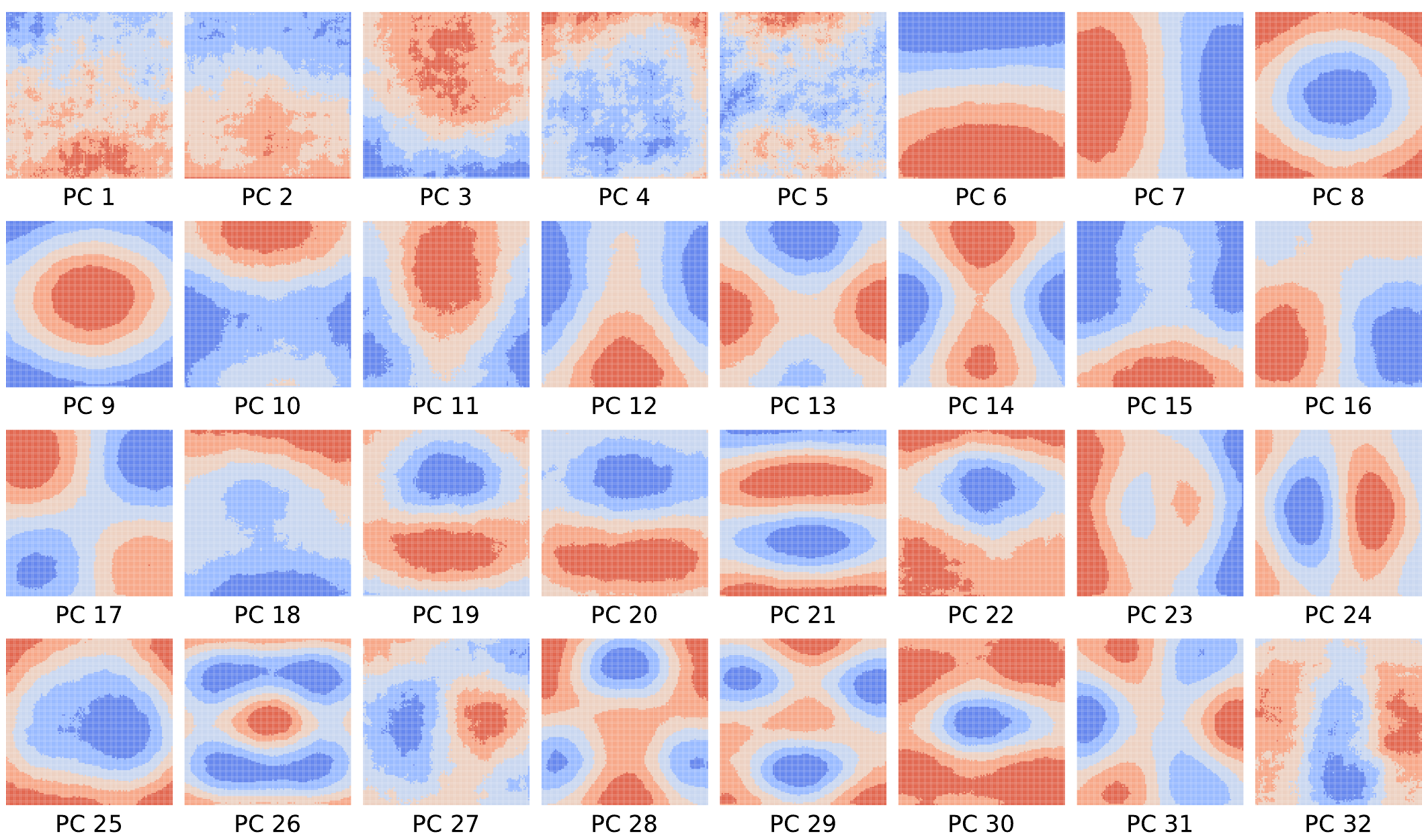}
        \caption{Block-1 Mean Filter Maps \label{fig:imagenet_maps}}
    \end{subfigure}
    \caption{(\subref{fig:imagenet}) Top-$8$ principal components of the activations of a \resnet-$200$ on \imagenet{} ordered block-wise by network topology (dimension of each block on the right $y$-axis). Block 1 is closest to input and Block 5 is the output of the network. Color coding is based on relative variance between blocks across the top-$8$ PCs from blue (low) to red (high). (\subref{fig:imagenet_maps}) Block 1 mean activation maps of the top-$32$ principal components of \resnet-$200$ on \imagenet{} computed with \pcagame{}.}
    \label{fig:resnet200}
\end{figure}

\paragraph{The principal components of \resnet-$200$ activations on \imagenet{} are edge filters.}
A primary goal of PCA is to obtain \emph{interpretable} low-dimensional representations. To this end we present an example of using \pcagame{} to compute the top-$32$ principal components of the activations of a pretrained \resnet-$200$ on the \imagenet{} dataset. We concatenate the flattened activations from the output of each residual block resulting in a $d \approx 20\text{M}$ dimensional vector representation for each of the roughly $1.2\text{M}$ input images. It is not possible to store the entire $195$TB matrix in memory, nor incrementally compute the Gram/covariance matrix.

We implemented a data-and-model parallel version of \pcagame{} in \textsc{Jax} \citep{jax2018github} where each $\hat{v}_i$ is assigned to it's own TPU 
\citep{jouppi2017datacenter}. Each device keeps a local copy of the \resnet{} parameters and the \imagenet{} datastream. Sampling a mini-batch (of size 128), computing the network activations and updating $\hat{v}_i$ are all performed locally. The \texttt{broadcast($\hat{v}_i$)} step is handled by the \texttt{pmap} and \texttt{lax.all\_gather} functions. Computing the top-$32$ principal components takes approximately nine hours on $32$ \texttt{TPUv3}s. 

Figure \ref{fig:imagenet} shows the top principal components of the activations of the trained network organized by network topology (consisting of five residual blocks). Note that \pcagame{} is \emph{not} applied block-wise, but on all $20$M dimensions. We do not assume independence between blocks and the eigenvector has unit norm across all blocks. We observe that Block 1 (closest to input) of PC 1 has very small magnitude activations relative to the other PCs. This is because PC 1 should capture the variance which discriminates most between the classes in the dataset. Since Block 1 is mainly concerned with learning low-level image filters, it stands to reason that although these are important for good performance, they do not necessarily extract abstract representations which are useful for classification. Conversely, we see that PC 1 has larger relative activations in the later blocks.

We visualize the average principal activation in Block 1\footnote{The activations in Block 1 result from convolving $64$ filters over the layer's input. We take the mean over the $64$ channels and plot the resulting $112 \times  112$ image.} in Figure~\ref{fig:imagenet_maps}. The higher PCs learn distinct filters (Gabor filters, Laplacian-of-Gaussian filters c.f.\ \citep{bell1997independent}).

\section{Conclusion}
\renewcommand{\epigraphsize}{\footnotesize}
\renewcommand{\epigraphrule}{0.0pt}
\renewcommand{\textflush}{flushright} \renewcommand{\sourceflush}{flushright}
\setlength{\beforeepigraphskip}{-.5\baselineskip}
\setlength{\afterepigraphskip}{-.5\baselineskip}

\setlength{\epigraphwidth}{.55\textwidth}

\let\originalepigraph\epigraph 
\renewcommand\epigraph[2]{\originalepigraph{\textit{#1}}{{\footnotesize #2}}}

\vspace{-10pt}
\epigraph{\epigraphsize It seems easier to train a bi-directional LSTM with attention\\ than to compute the SVD of a large matrix. --Chris Re}
{NeurIPS 2017 Test-of-Time Award, Rahimi and Recht \citep{rahimi2017reflections}.}

In this work we motivated PCA from the perspective of a multi-player game. This inspired a decentralized algorithm which enables large-scale principal components estimation. To demonstrate this we used \pcagame{} to analyze a large neural network through the lens of PCA. To our knowledge this is the first academic analysis of its type and scale (for reference, \citep{tang2019exponentially} compute the top-$6$ PCs of the $d=2300$ outputs of VGG). \pcagame{} also opens a variety of research directions.

\textbf{Scale.} In experiments, we broadcast across all edges in Figure~\ref{fig:dag} every iteration. Introducing lag or broadcasting with dropout may improve efficiency. Can we further reduce our memory footprint by storing only scalars of the losses and avoiding congestion through online bandit or reinforcement learning techniques? Our decentralized algorithm may have implications for federated and privacy preserving learning as well~\citep{heinze2016dual,heinze2018preserving,bonawitz2019towards}.

\textbf{Games.} \pcagame{} has a unique Nash equilibrium due to the fixed DAG structure, but vectors are initialized randomly so $\hat{v}_k$ may start closer to $v_1$ than $\hat{v}_1$ does. Adapting the DAG could make sense, but might also introduce spurious fixed points or suboptimal Nash. Might replacing vectors with populations accelerate extraction of the top principal components?

\textbf{Core ML.} \pcagame{} could be useful as a diagnostic or for accelerating training~\citep{desjardins2015natural,krummenacher2016scalable}; similarly, spectral normalization has shown to be a valuable tool for stabilizing GAN training \citep{miyato2018spectral}.

Lastly, GANs~\citep{goodfellow2014generative} recently reformulated learning a generative model as a two-player zero-sum game. Here, we show how another fundamental unsupervised learning task can be formulated as a $k$-player game. While two-player, zero-sum games are well understood, research on $k$-player, general-sum games lies at the forefront in machine learning. We hope that marrying a fundamental, well-understood task in PCA with the relatively less understood domain of many player games will help advance techniques on both ends.

\section*{Acknowledgements}
We are grateful to Trevor Cai for his help scaling the \textsc{Jax} implementation of \pcagame{} to handle the large \imagenet{} experiment and to Daniele Calandriello for sharing his expert knowledge of related work and advice on revising parts of the manuscript.

{\footnotesize
\bibliographystyle{plainnat}
\bibliography{references}
}

\newpage
\appendix
\section{Experiment Details}
\label{exp_details}

In the synthetic experiments, $\hat{V}$ is initialized randomly so $M \in \mathbb{R}^{50 \times 50}$ is constructed as a diagonal matrix without loss of generality. The linear spectrum ranges from $1$ to $1000$ with equal spacing. The exponential spectrum ranges from $10^3$ to $10^0$ with equal spacing on the exponents.

\subsection{Clarification of Oja Variants}
\label{oja_disambig}
As discussed in Section~\ref{related_work}, it is easy to confuse the various Oja methods. In our experiments, Oja's algorithm refers to applying Hebb's rule $v_i \leftarrow v_i + \eta M v_i$ followed by an orthonormalization step computed with \qr{} as in \Algref{alg:ojas}:
    
\begin{figure}[ht!]
\begin{algorithm}[H]
\begin{algorithmic}
    \State Given: data stream, $X_t \in \mathbb{R}^{m \times d}$, $T$, $\hat{V}^0 \in$ $\mathcal{S}^{d-1}$$\times$$\ldots$$\times$$\mathcal{S}^{d-1}$, step size $\eta$
    \State $\hat{V} \leftarrow \hat{V}^0$
    \State $\texttt{mask} \leftarrow \texttt{LT}(2I_k - \mathbf{1}_k)$
    \For{$t = 1: T$}
        \State $\hat{V} \leftarrow \hat{V} + \eta X_t^\top X_t \hat{V}$
        \State $Q, R \leftarrow \qr{}(\hat{V})$
        \State $S = \texttt{sign}(\texttt{sign}(\texttt{diag}(R)) + 0.5)$
        \State $\hat{V} = Q S$
    \EndFor
    \State return $\hat{V}$
\end{algorithmic}
\caption{Oja's Algorithm}
\label{alg:ojas}
\end{algorithm}
\end{figure}

where $\mathbf{1}_k$ is a $k \times k$ matrix of all ones, \texttt{LT} returns the lower-triangular part of a matrix (includes the diagonal), and $\texttt{sign} = \begin{cases} -1 & \text{ if } x < 0 \\ 0 & \text{ if } x = 0 \\ 1 & \text{ if } x > 0 \end{cases}$. Oja's algorithm is the standard nomenclature for this variant in the machine learning literature~\citep{allen2017first}.

In the scaled-down \resnet{} experiments (see Section~\ref{small_resnet}), we use Hebb's rule with deflation, also sometimes referred to as Oja's. Deflation is accomplished by directly subtracting out the parent vectors from the dataset. In detail, each batch of data samples, $X_t \in \mathbb{R}^{m \times d}$, is preprocessed as $X_{(i),t} \leftarrow X_t (I - \sum_{j < i} \hat{v}_j \hat{v}_j^\top)$. Then to learn each $\hat{v}_i$, we repeatedly apply Hebb's rule with $M_t = X_{(i),t}^\top X_{(i),t}$ and then $\hat{v}_i \leftarrow \frac{\hat{v}_i}{||\hat{v}_i||}$ to project $\hat{v}_i$ back to the unit-shere. After several iterations $t$ and once $\hat{v}_i$'s Rayleigh quotient appears to have stabilized, we move on to $\hat{v}_{i+1}$.

\section{Spectrum of \resnet{} Activations}
\begin{figure}[ht!]
    \centering
    \includegraphics[width=0.5\textwidth]{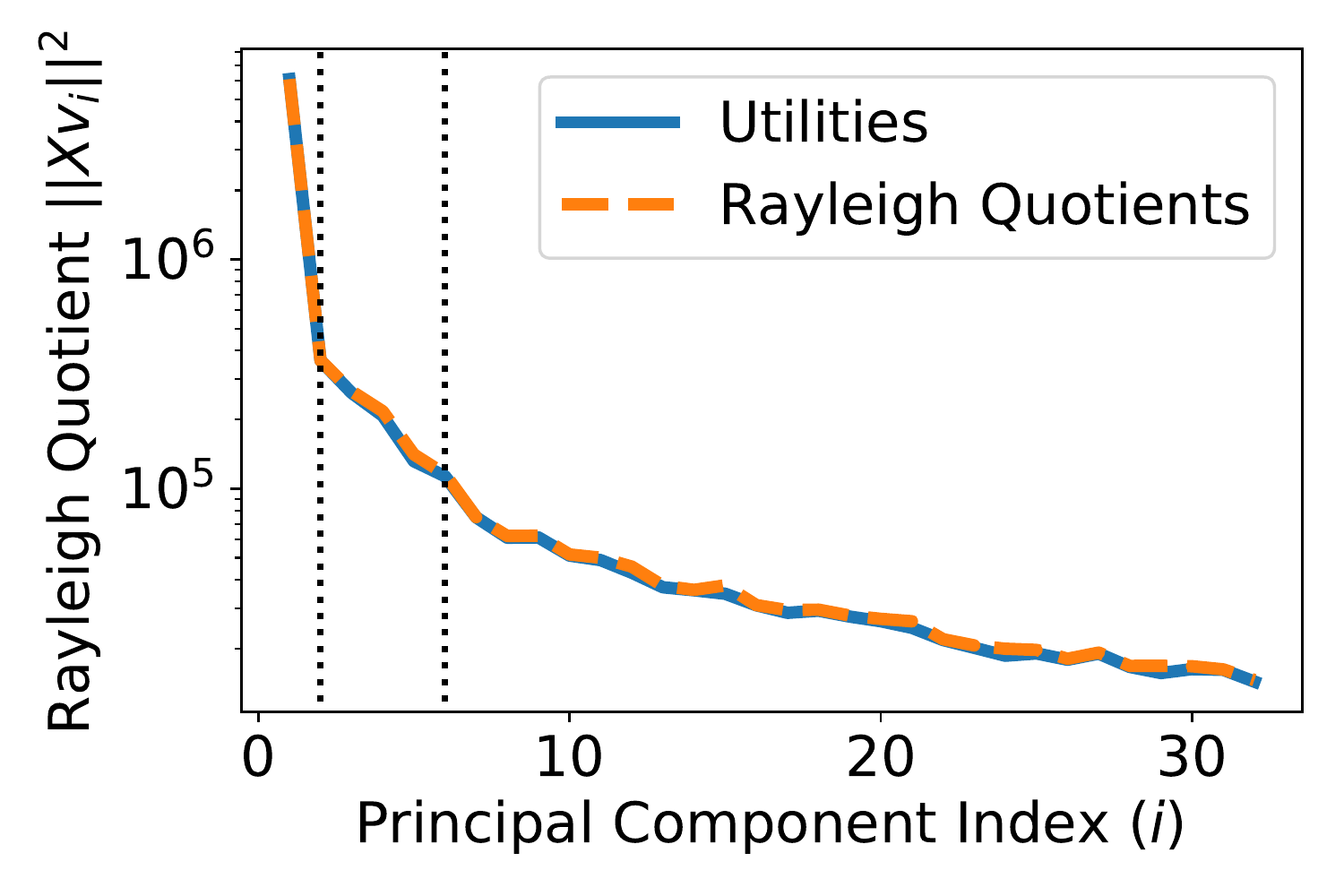}
    \caption{Approximate Eigenvalue Spectrum of \resnet{}-200 Activations.}
    \label{fig:imagenet_scree}
\end{figure}

Figure~\ref{fig:imagenet_scree} shows a scree plot of the Rayleigh quotients recovered by \pcagame{} and the respective utility achieved by each player. The two curves almost perfectly overlap. The mean relative magnitude of the penalty terms to the respective Rayleigh quotient in the utility is $0.025$ indicating that the solutions of each player are close to orthogonal with respect to the generalized inner product (\Eqref{obj}). This implies that that the solutions are indeed eigenvectors. The scree plot has two distinct elbows at PC2 and PC6, corresponding to the differences in filters observed in Figure~\ref{fig:imagenet_maps}.

\section{Synthetic Experiments\textemdash Figures Enlarged}
\begin{figure}[!ht]
    \centering
    \begin{subfigure}[b]{.49\textwidth}
    \includegraphics[width=0.98\textwidth]{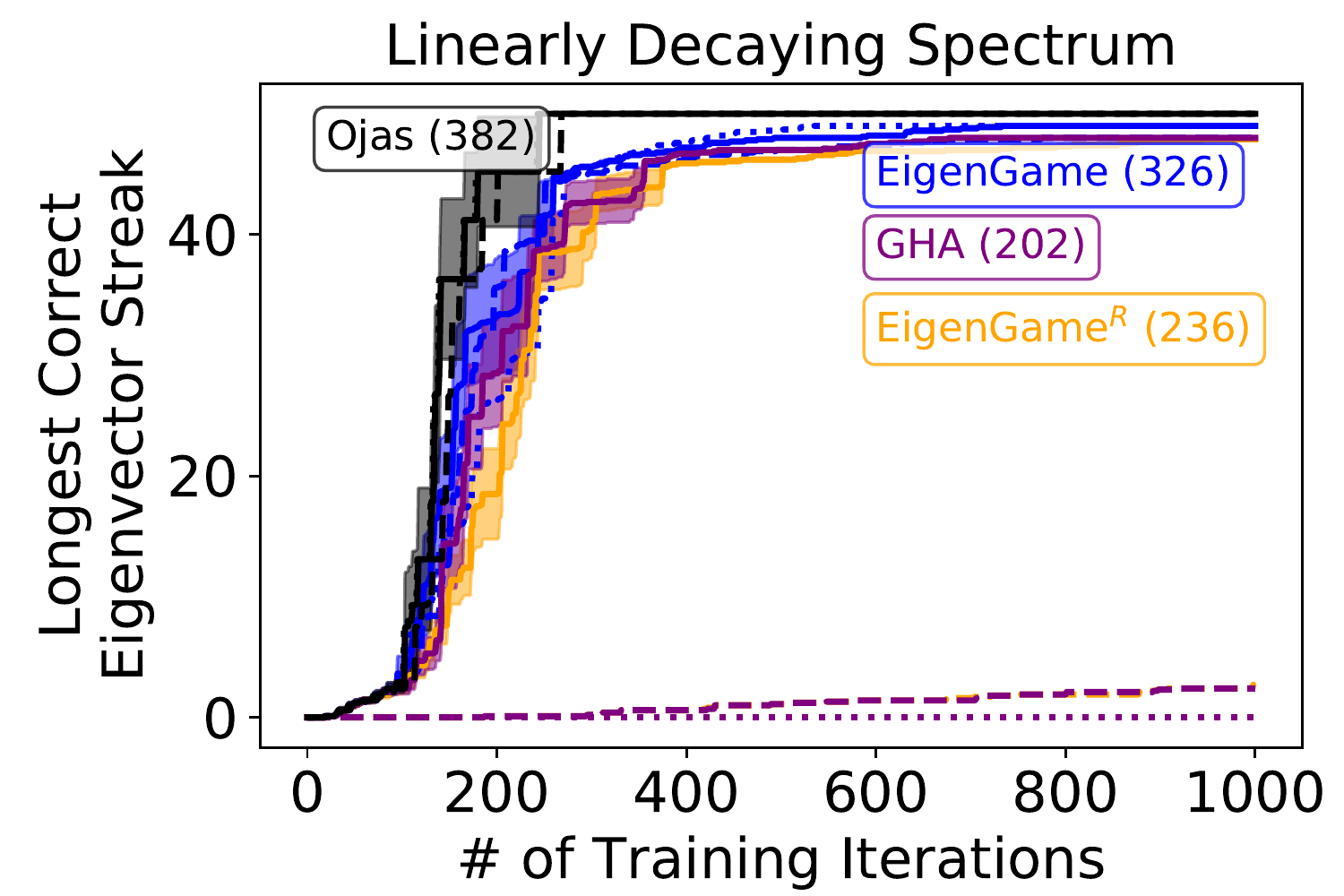}
    \caption{Linear Spectrum \label{fig:synth_lin_big}}
    \end{subfigure}
    \begin{subfigure}[b]{.49\textwidth}
    \includegraphics[width=0.98\textwidth]{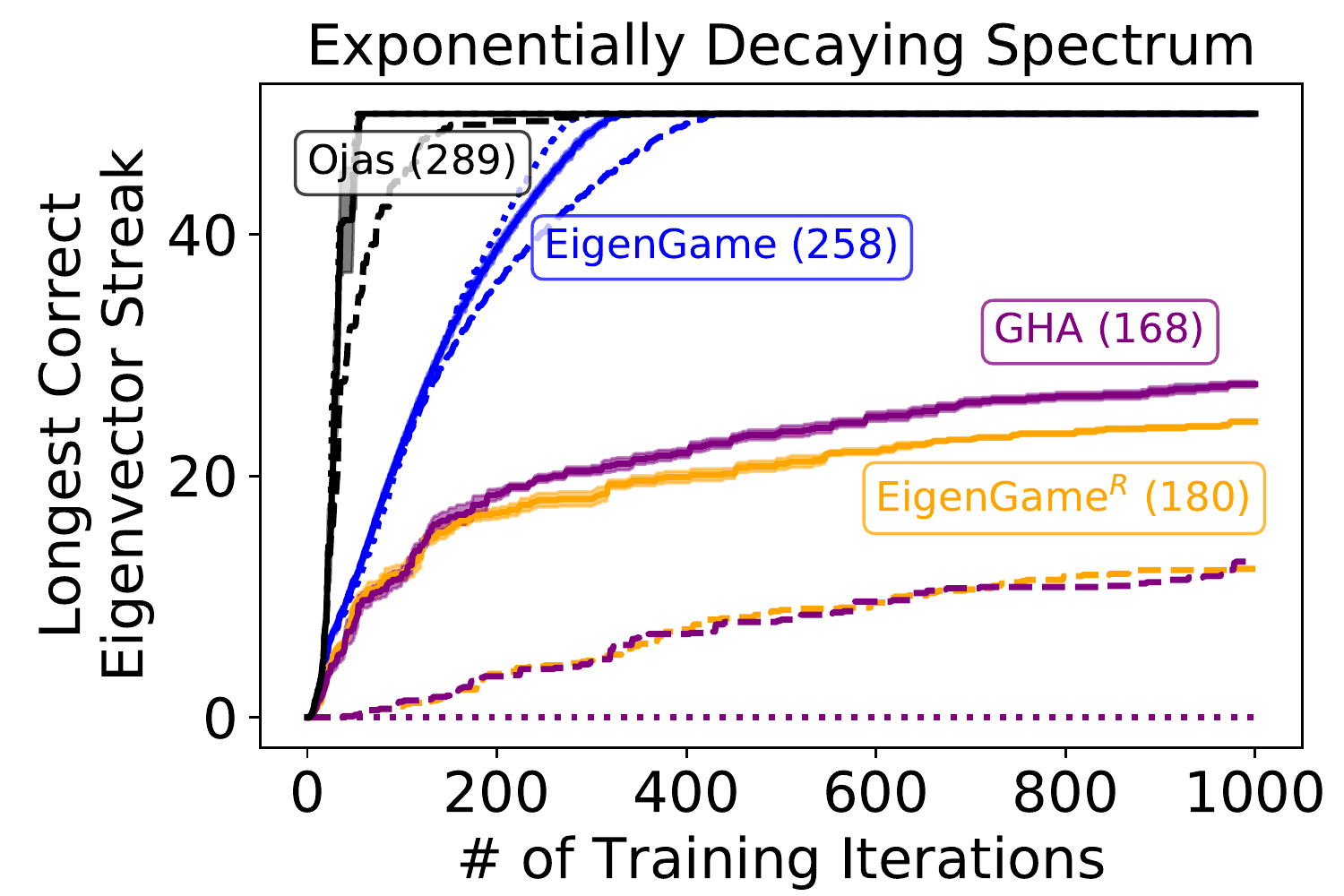}
    \caption{Exponential Spectrum \label{fig:synth_exp_big}}
    \end{subfigure}
    \vspace{-5pt}
    \caption[The longest streak of consecutive vectors with angular error less than $\frac{\pi}{8}$ radians is plotted versus algorithm iterations for a matrix $M \in \mathbb{R}^{50 \times 50}$ with a spectrum decaying from $1000$ to $1$ linearly (\subref{fig:synth_lin_big}) and exponentially (\subref{fig:synth_exp_big}). Average runtimes are reported in milliseconds next to the method names. We omit Krasulina's as it is only designed to find the top-$k$ subspace. Both \pcagame{} variants and GHA achieve similar asymptotes on the linear spectrum. Learning rates were chosen from $\{10^{-3},\ldots,10^{-6}\}$ on $10$ held out runs. Solid lines denote results with the best performing learning rate. Dotted and dashed lines denote results using the best learning rate $\times$ $10$ and $0.1$. All plots show means over $10$ trials. Shaded regions highlight $\pm$ standard error of the mean for the best performing learning rates.]{The longest streak of consecutive vectors with angular error less than $\frac{\pi}{8}$ radians is plotted versus algorithm iterations for a matrix $M \in \mathbb{R}^{50 \times 50}$ with a spectrum decaying from $1000$ to $1$ linearly (\subref{fig:synth_lin_big}) and exponentially (\subref{fig:synth_exp_big}). Average runtimes are reported in milliseconds next to the method names.\footnotemark{} We omit Krasulina's as it is only designed to find the top-$k$ subspace. Both \pcagame{} variants and GHA achieve similar asymptotes on the linear spectrum. Learning rates were chosen from $\{10^{-3},\ldots,10^{-6}\}$ on $10$ held out runs. Solid lines denote results with the best performing learning rate. Dotted and dashed lines denote results using the best learning rate $\times$ $10$ and $0.1$. All plots show means over $10$ trials. Shaded regions highlight $\pm$ standard error of the mean for the best performing learning rates.}
    \label{fig:synthetic_results_big}
\end{figure}
\footnotetext{\pcagame{} runtimes are longer than those of \pcagame{}$^R$ in the synthetic experiments despite strictly requiring fewer FLOPS; apparently this is due to low-level floating point arithmetic specific to the experiments.}

\section{\mnist{} Experiments\textemdash Figures Enlarged}
See Appendix~\ref{gradient_bias}.

\newpage
\section{\resnet{}-200 Experiments\textemdash Figures Enlarged}
Figures~\ref{fig:resnet200_pcs_large} and~\ref{fig:resnet200_maps_large} show enlarged versions of Figures~\ref{fig:imagenet} and~\ref{fig:imagenet_maps} from the main body.

\begin{figure}[!ht]
    \centering
    \includegraphics[width=0.98\textwidth]{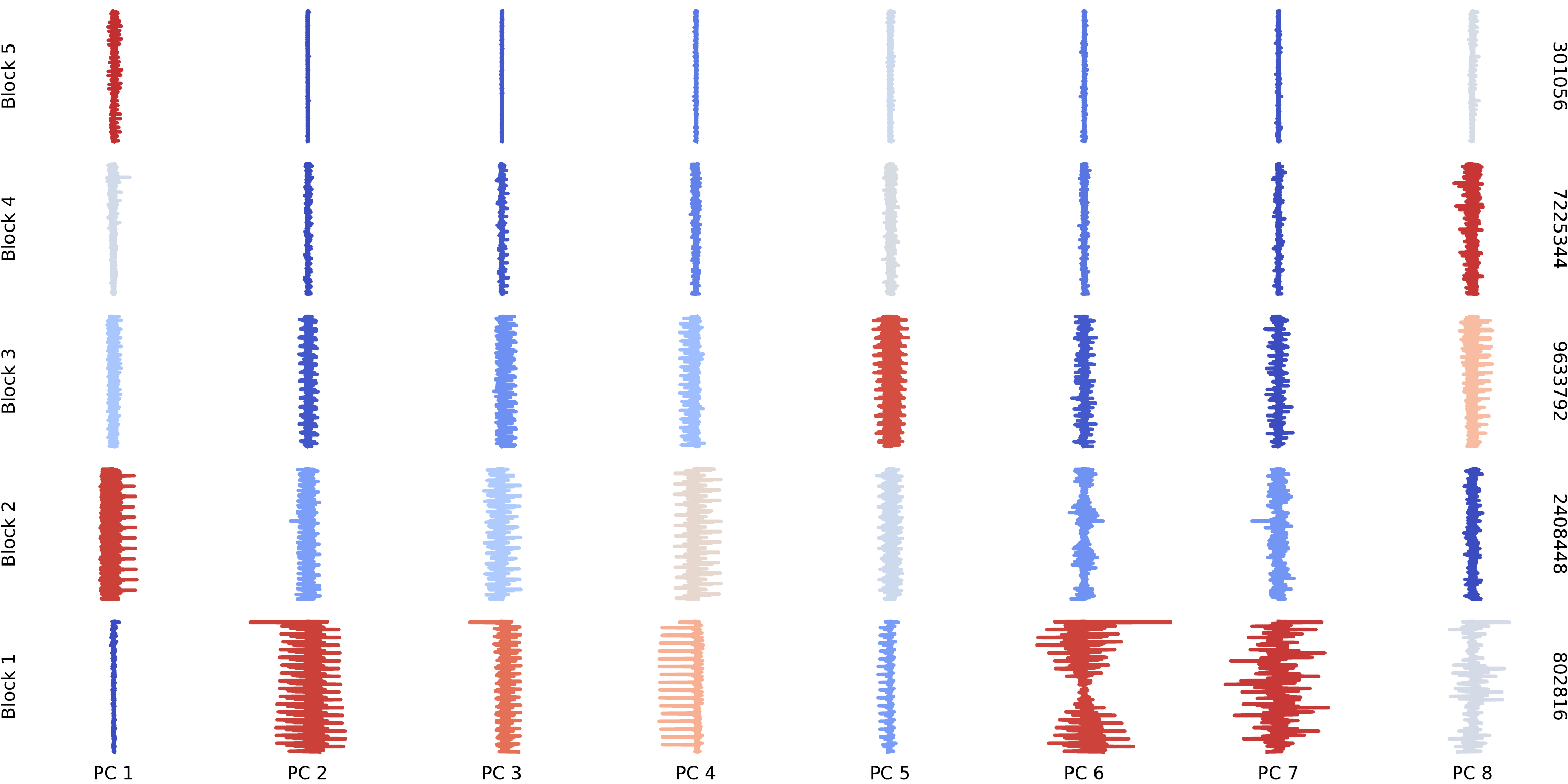}
    \caption{Top-$8$ principal components of the activations of a \resnet-$200$ on \imagenet{} ordered block-wise by network topology (dimension of each block on the right $y$-axis). Block 1 is closest to input and Block 5 is the output of the network. Color coding is based on relative variance between blocks across the top-$8$ PCs from blue (low) to red (high).}
    \label{fig:resnet200_pcs_large}
\end{figure}

\begin{figure}[!ht]
    \centering
    \includegraphics[width=0.98\textwidth]{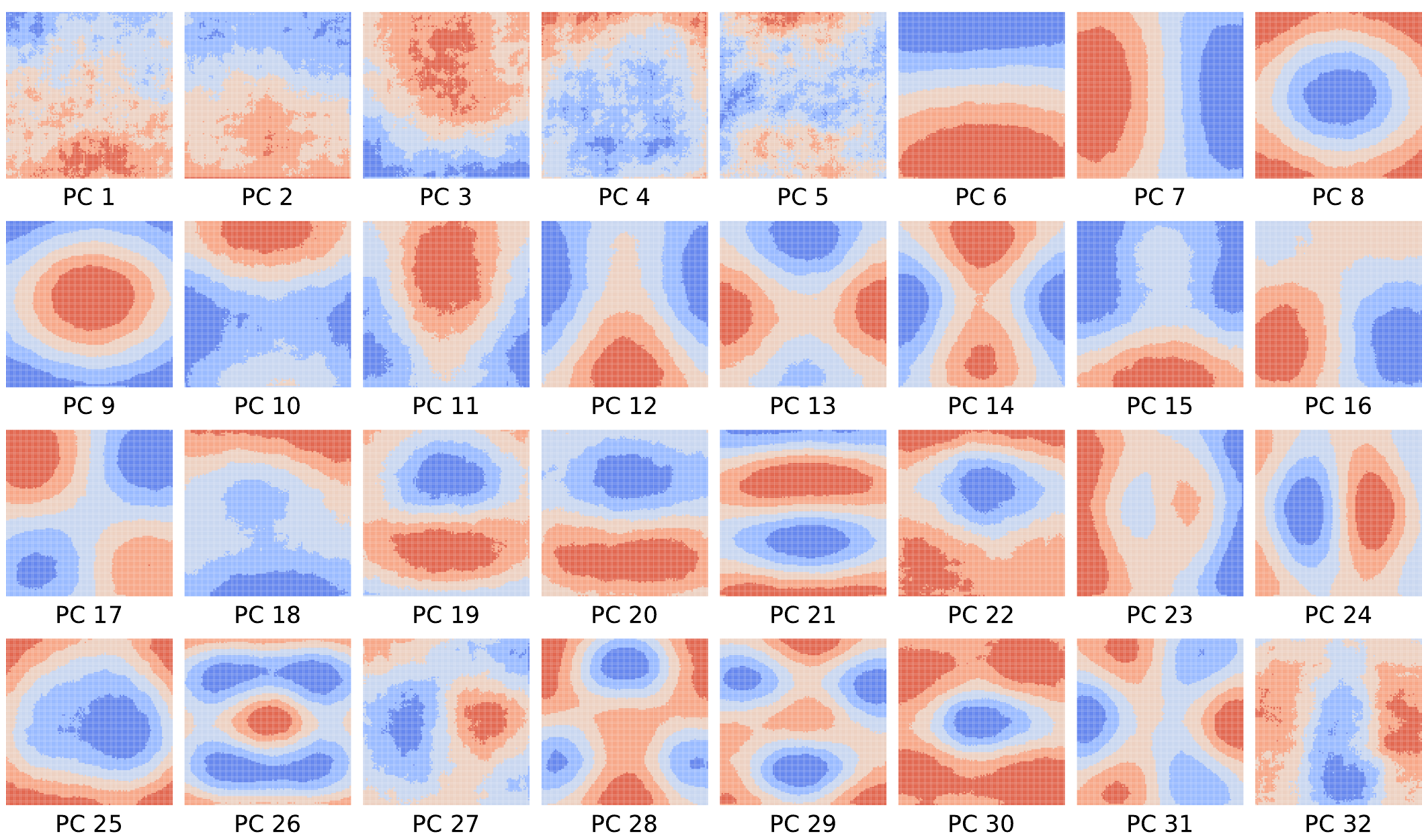}
    \caption{Block 1 mean activation maps of the top-$32$ principal components of \resnet-$200$ on \imagenet{} computed with \pcagame{}.}
    \label{fig:resnet200_maps_large}
\end{figure}

\section{\pcagame{} Vectorized for CPU}
\Algref{pcagame_ascent_vec} presents \Algref{pcagame_ascent} in a vectorized form for implementation on a CPU. \texttt{LT} returns the lower-triangular part of a matrix (includes the diagonal). \texttt{sum}$(A, \texttt{dim}=0)$ sums over the rows of $A$. \texttt{norm}$(A, \texttt{dim}=0)$ returns an array with the $L_2$-norm of each column of $A$. $\odot$ denotes elementwise multiplication. $\mathbf{1}_k$ is a square $k \times k$ matrix of all ones. $I_k$ is the $k \times k$ identity matrix. When dividing a matrix by a vector ($A / v$), we assume broadcasting. Specifically, $v$ is interpreted as a row-vector and stacked vertically to match the dimensions of $A$; the two matrices are then divided element wise.
\begin{figure}[ht!]
\begin{algorithm}[H]
\begin{algorithmic}
    \State Given: data stream, $X_t \in \mathbb{R}^{m \times d}$, $T$, $\hat{V}^0 \in$ $\mathcal{S}^{d-1}$$\times$$\ldots$$\times$$\mathcal{S}^{d-1}$, step size $\alpha$
    \State $\hat{V} \leftarrow \hat{V}^0$
    \State $\texttt{mask} \leftarrow \texttt{LT}(2I_k - \mathbf{1}_k)$
    \For{$t = 1: T$}
        \State $R \leftarrow (X_t\hat{V})^\top (X_t\hat{V})$
        \State $R_{\text{norm}} \leftarrow R / \texttt{diag}(R)$
        \State $G_s \leftarrow \hat{V} (R_{\text{norm}} \odot \texttt{mask})^\top$
        \State $\nabla_{\hat{V}} \leftarrow X_t^\top (X_t G_s)$
        \State $\nabla^R_{\hat{V}} \assign{-}{=} \hat{V} \texttt{sum}(\nabla_{\hat{V}} \odot \hat{V}, \texttt{dim}=0)$
        \State $\hat{V} \leftarrow \hat{V} + \alpha \nabla^R_{\hat{V}}$
        \State $\hat{V} \leftarrow \hat{V} / \texttt{norm}(\hat{V}, \texttt{dim}=0)$
    \EndFor
    \State return $\hat{V}$
\end{algorithmic}
\caption{\pcagame{} \& \pcagame{}$^R$\textemdash Vectorized}
\label{pcagame_ascent_vec}
\end{algorithm}
\end{figure}

\section{Smallest Eigenvectors}
\pcagame{} can be used to recover the $k$ smallest eigenvectors as well. Simply use \pcagame{} to estimate the top eigenvector with eigenvalue $\Lambda_{11}$. Then run \pcagame{} on the matrix $M' = \Lambda_{11} I - M$. The top-$k$ eigenvectors of $M$' are the bottom-$k$ eigenvectors of $M$. For example, the $d$th eigenvector of $M$, $v_{d}$, is the largest eigenvector of $M'$: $M' v_{d} = \Lambda_{11} v_d - M v_d = (\Lambda_{11} - \Lambda_{dd}) v_d$.

\section{Frequent Directions}
\label{freqdirs}
A reviewer from a previous submission of this work requested a comparison and discussion with Frequent Directions~\citep{ghashami2016frequent}, another decentralized subspace-error minimizing $k$-PCA algorithm. Frequent Directions (FD) is a streaming algorithm that maintains an overcomplete sketch matrix with the goal of capturing the subspace of maximal variance within the span of its vectors. Each step of FD operates by first replacing a row of the sketch matrix with a single data sample. It then runs SVD on the sketch matrix and uses the resulting decomposition to construct a new sketch. Note that FD relies on SVD as a core inner step. In theory, \pcagame{} could replace SVD, however, we do not explore that direction here.

\subsection{Recovering Principal Components from Principal Subspace}
FD returns a sketch $B=\hat{V}^\top$ of size $\mathbb{R}^{2l \times d}$ where $l \ge k$. The rows of FD are not principal components, but they should approximate the top-$k$ subspace of the dataset. To recover approximate principal components, the optimal rotation of the vectors can be computed with $Q \leftarrow SVD(XB^\top)$. This can be shown by inspecting $R$ (as defined in Section~\ref{derivation}) with rotated vectors:
\begin{align}
    (\hat{V}Q)^\top M (\hat{V}Q) &= Q^\top \hat{V}^\top M \hat{V} Q = Q^\top (X\hat{V})^\top (X\hat{V}) Q = Q^\top M' Q.
\end{align}
By inspection, the problem of computing the optimal $Q$ reduces to computing the eigenvectors of $M' \in \mathbb{R}^{k \times k}$. This requires projecting the dataset into the principal subspace, $(X\hat{V})$, to compute $M'$ however, this is typically a desired step anyways when performing PCA.

\subsection{Complexity Analysis}
We base our analysis on Section 3.1 of~\citep{ghashami2016frequent} which discusses parallelizing FD. Let $b$ be number of shards to split the original dataset $X \in \mathbb{R}^{n \times d}$ into, each shard being in $\mathbb{R}^{\frac{n}{b} \times d}$. Let $k$ be the number of principal components sought. Finally, let $l = \lceil k + \frac{1}{\epsilon} \rceil$ be the sketch size where $\epsilon \ll 1$ is a desired tolerance on the Frobenius norm of the subspace approximation error.

The runtime of FD is $\mathcal{O}(nld)$; call this $Anld$ for some $A$. To decentralize FD,~\citep{ghashami2016frequent} instructs to
\begin{enumerate}
    \item Split $X$ into $b$ shards and run FD on each individually in parallel.
    \begin{itemize}
        \item total runtime: $A (\frac{n}{b})ld = Anld (\frac{1}{b})$
        \item output: $b$ sketches ($B_i \in \mathbb{R}^{2l \times d}$)
    \end{itemize}
    \item Merge sketches and run FD on the merged sketch to produce sketch $B$.
    \begin{itemize}
        \item total runtime: $A (2lb) ld = Anld (\frac{2bl}{n})$
        \item output: 1 sketch ($B \in \mathbb{R}^{2l \times d}$)
    \end{itemize}
\end{enumerate}
Finally, normalize the rows of $B$, project the dataset $Y \leftarrow XB^\top$, compute the right-singular vectors of the projected dataset, $Q \in \mathbb{R}^{2l \times 2l} \leftarrow SVD(Y)$, compute $\hat{V} \leftarrow B^\top Q$, and compute the corresponding Rayleigh quotients $\hat{V}^\top M \hat{V} = (YQ)^\top (YQ)$ to determine the top-$k$ eigenvectors with error within the desired tolerance. We assume this final step takes negligible runtime because we assume $2l \ll d$, however, for datasets with many samples (large $n$), this step could be nonnegligible without further approximation.

Using the runtimes listed above, we can determine the potential runtime multiplier from decentralization is $(\frac{1}{b} + \frac{2bl}{n})$ which is convex in $b$. If we minimize this w.r.t. $b$ for the optimal number of shards, we find $b^* = \sqrt{\frac{n}{2l}}$. Plugging this back in gives an optimal runtime multiplier of $2\sqrt{2} \sqrt{\frac{l}{n}}$.

The analysis above only considers one recursive step. Step 1) can be decentralized as well. For simplicity, we assume the computation is dominated by Step 2), the merge step. Note these relaxations result in a lower bound on FD runtime, i.e., they favor FD in a comparison with \pcagame{}.

\subsection{Small ImageNet Experiments}
\label{small_resnet}

Consider running on a scaled down \resnet-50 experiment which has approximately $1.2M$ images ($n = 1.2 \times 10^6$, 24TB) and searching for the top-$25$ eigenvectors ($k=25$). Using a modest $\epsilon = \frac{0.25}{k}$ implies $l = 5k = 125$ with optimal batch size $b^* \approx 70$. Therefore, running FD on $\frac{n}{b}$ samples with a sketch size of $125$ should give a rough lower bound on the runtime for an optimally decentralized FD implementation. The runtime obtained was $9$ hours for FD vs $2$ hours for \pcagame{} which actually processes the full dataset $3$ times.

\begin{figure}[ht!]
    \centering
     \includegraphics[width=0.9 \textwidth]{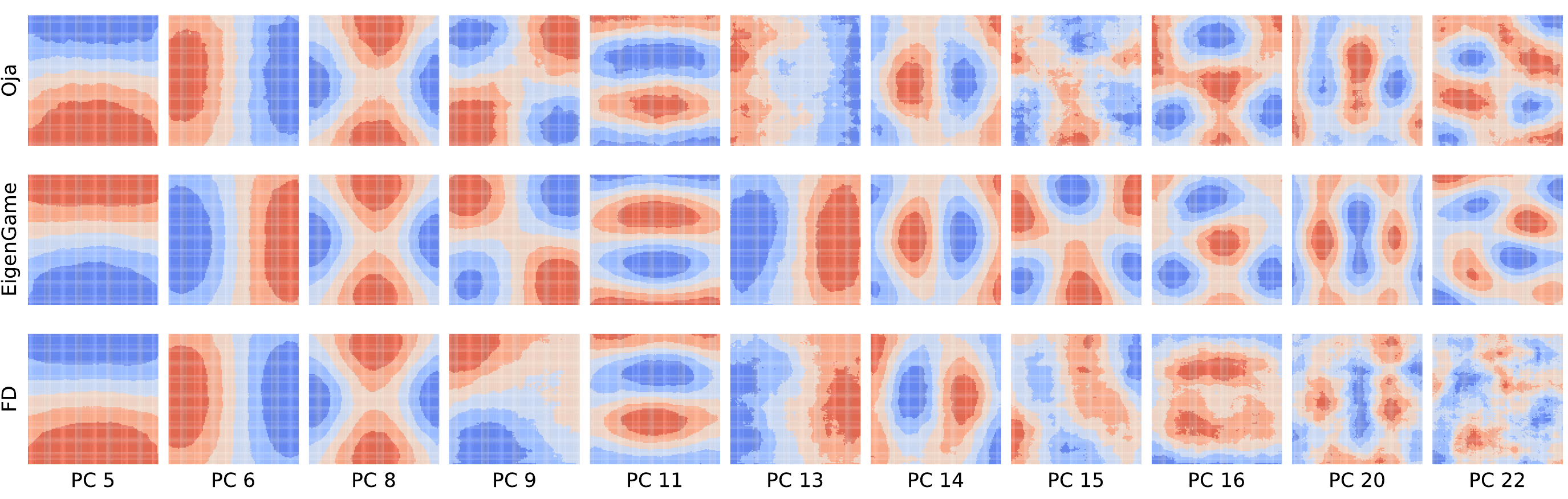}
     \caption{Comparison of mean activation maps between Oja's with deflation, \pcagame{}, and FD for a section of the top principal components of \resnet-50 on \imagenet{}.}
    \label{fig:imagenet_maps_small}
\end{figure}

The reason we run FD on a scaled down \resnet-50 experiment as opposed to the \resnet-200 is that the algorithm requires a final SVD step to recover the actual eigenvectors and we were not able to run SVD on a sketch of size $k \times d$ where $d = 20 \times 10^6$ for the full scale experiment. That is to say FD is not applicable in this extremely large data regime. In contrast, \pcagame{} handles this setting without modification.

To obtain an approximate ``ground truth'' solution for the principal components we run Oja's algorithm with a low learning rate with a batch size of $128$ for $3$ epochs to extract the first eigenvector. We find successive eigenvectors using deflation. By running each step for many iterations and monitoring the convergence of the Rayleigh quotient (eigenvalue) $v_i^\top M v_i$, we can control the quality of the recovered eigenvectors. This is the simplest and most reliable approach to creating ground truth on a problem where no solution already exists. See Section~\ref{oja_disambig} for further details.

\section{Gradient Bias}
\label{gradient_bias}
As expected, Figure~\ref{fig:gradient_bias} shows the performance of \pcagame{} degrades in the low batch size regime. This is expected because we use the same minibatch for all inner products in the gradient which contains products and ratios of random variables. GHA, on the other hand, is linear in the matrix $M$ and as such is naturally unbiased. However, GHA does not appear to readily extend to more general function approximators, whereas \pcagame{} should. Instead we look to reduce the bias of \pcagame{} gradients using larger batch sizes (current hardware easily supports batches of 1024 for \mnist{} and 128 for \imagenet{}). Further reducing bias is left to future work.
\begin{figure}[ht]
    \centering
    \begin{subfigure}[b]{.49\textwidth}
    \centering
    \includegraphics[width=0.98\textwidth]{figures/mnist/longest_streak/ls_k16_mb1024_krashelp.pdf}
    \includegraphics[width=0.98\textwidth]{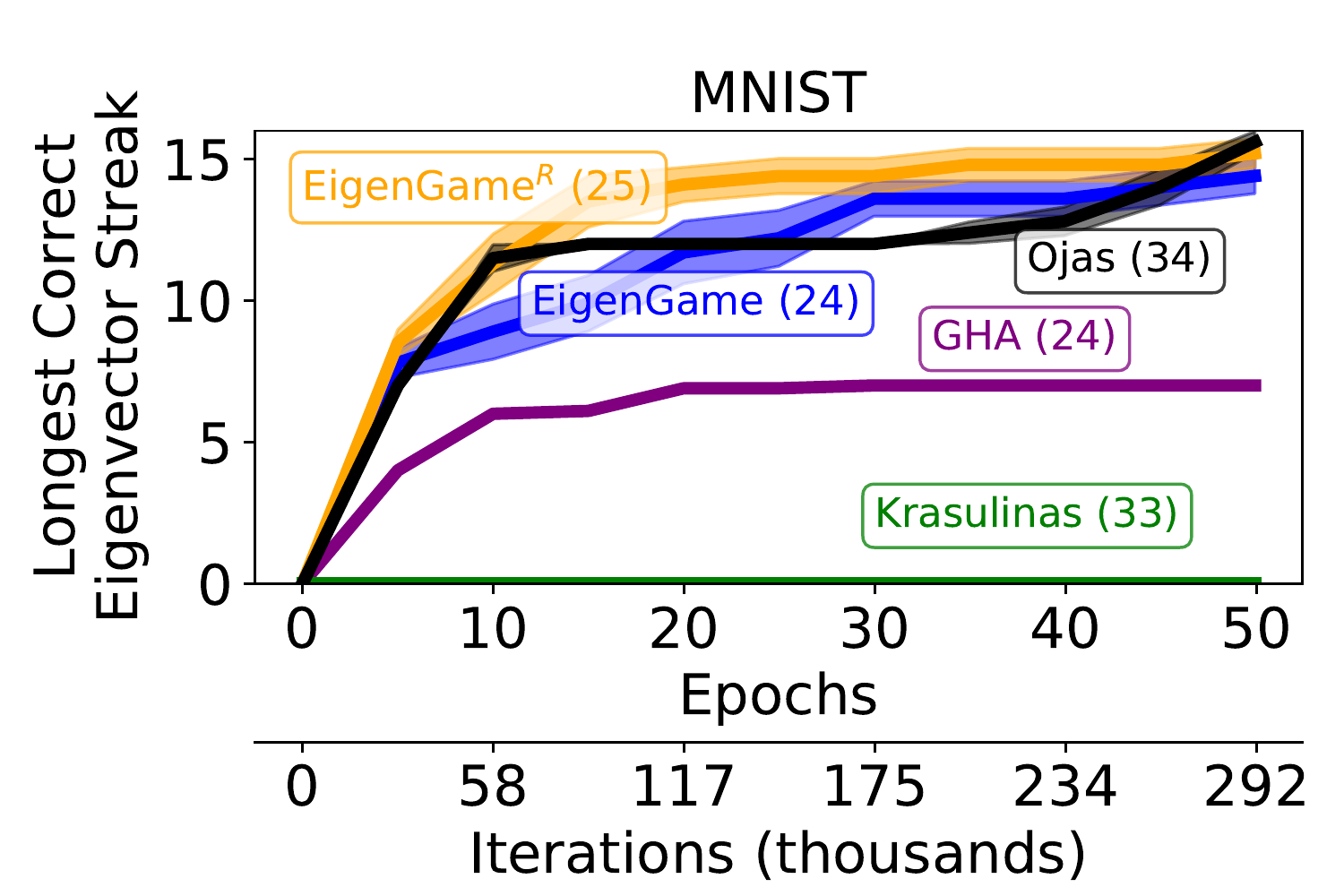}
    \includegraphics[width=0.98\textwidth]{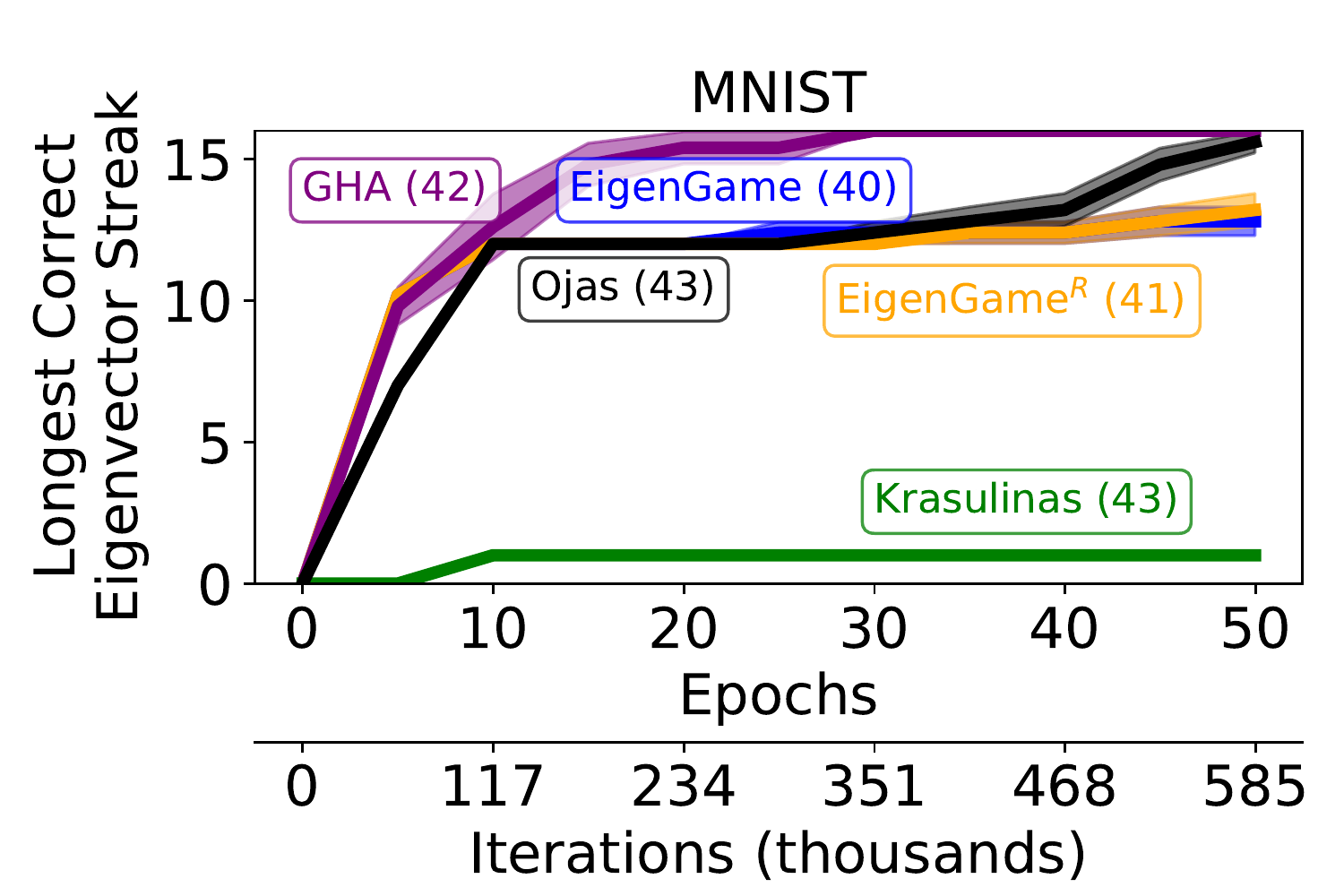}
    \caption{Longest Streak \label{fig:ls_small}}
    \end{subfigure}
    \begin{subfigure}[b]{.49\textwidth}
    \centering
    \includegraphics[width=0.98\textwidth]{figures/mnist/neurips_loss/nl_k16_mb1024_krashelp.pdf}
    \includegraphics[width=0.98\textwidth]{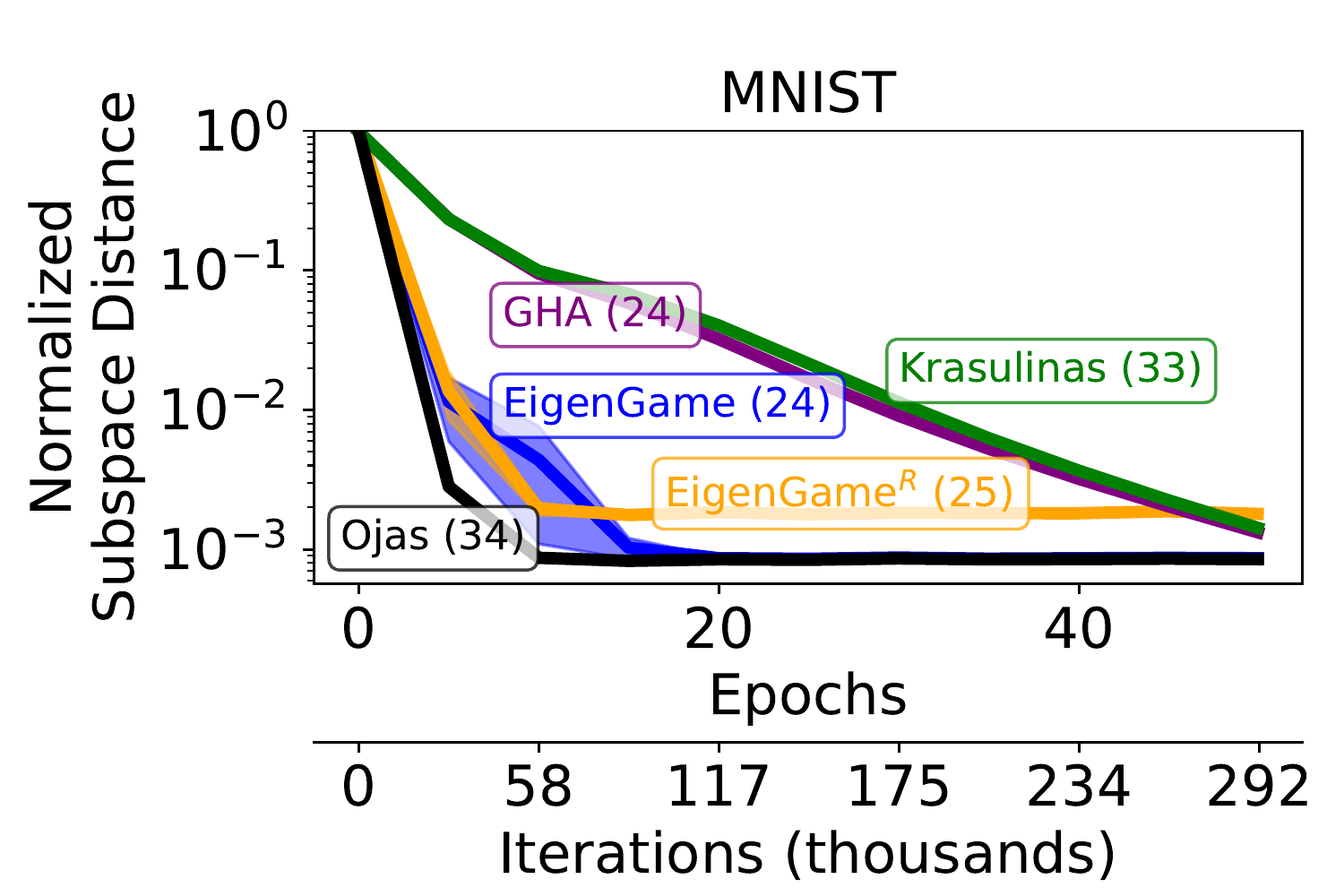}
    \includegraphics[width=0.98\textwidth]{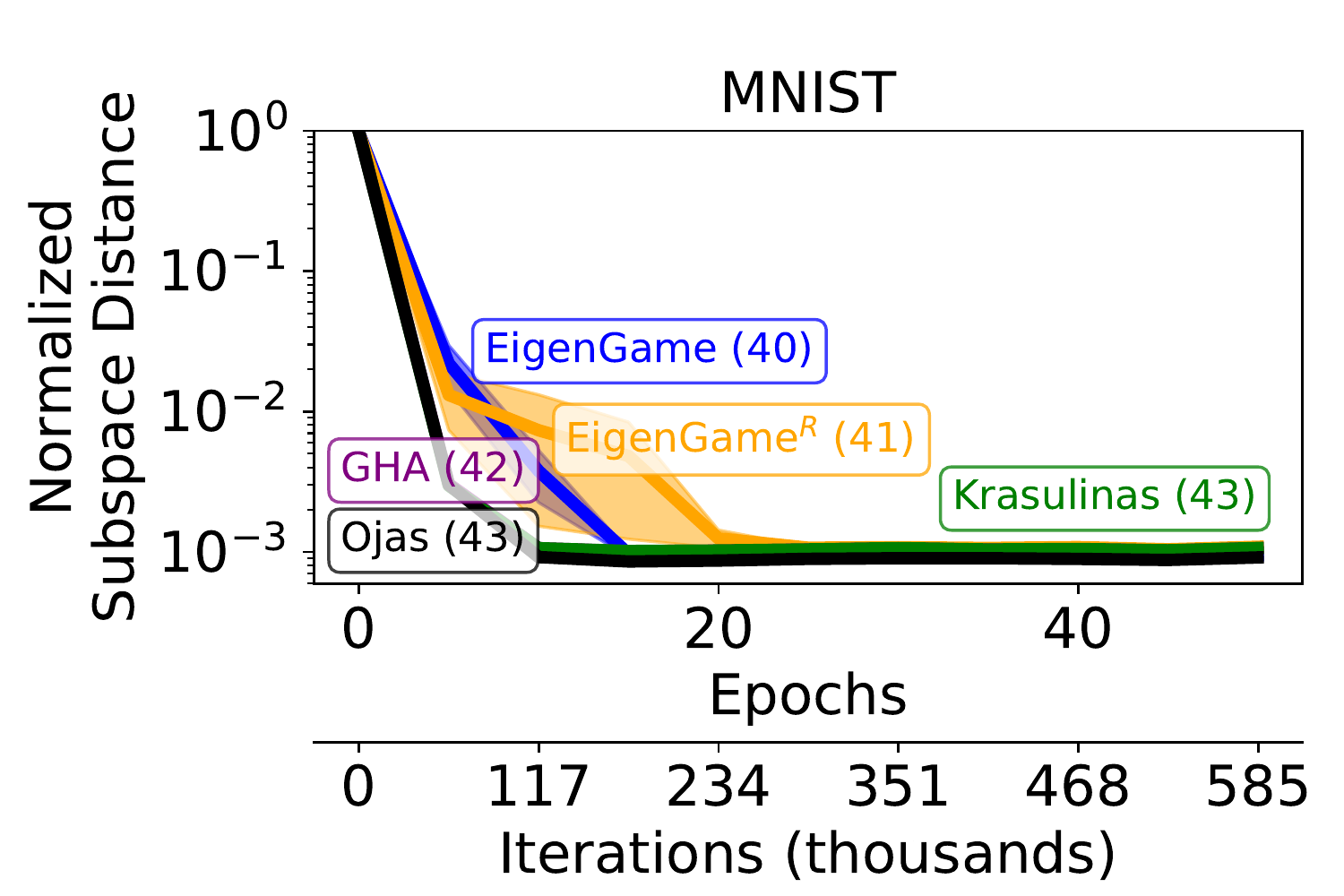}
    \caption{Subspace Distance \label{fig:sd_small}}
    \end{subfigure}
    \vspace{-5pt}
    \caption{(\subref{fig:ls_small}) The longest streak of consecutive vectors with angular error less than $\frac{\pi}{8}$ radians is plotted vs algorithm iterations on \mnist{} for minibatch sizes of $1024$ (top), $512$ (middle), and $256$ (bottom). Shaded regions highlight $\pm$ standard error of the mean for the best performing learning rates. Average runtimes are reported in seconds next to the method names. (\subref{fig:sd_small}) Subspace distance on \mnist{}. (\subref{fig:ls_small},\subref{fig:sd_small}) Learning rates were chosen from $\{10^{-3},\ldots,10^{-6}\}$ on $10$ held out runs. All plots show means over $10$ trials.}
    \label{fig:gradient_bias}
\end{figure}

\section{To project or not to project?}
\label{gradr_instability}
Projecting the update direction onto the unit-sphere, as suggested by Riemannian optimization theory, can result in much larger update steps. This effect is due to the composition of the retraction ($z' \leftarrow \tilde{z}/||\tilde{z}||$) and update step ($\tilde{z} \leftarrow z + \Delta z$). Omitting the projection can actually mimic modulating the learning rate, decaying it near an equilibrium and improving stability.

\begin{figure}[ht!]
    \centering
    \begin{subfigure}[b]{.59\textwidth}
    \includegraphics[scale=0.4]{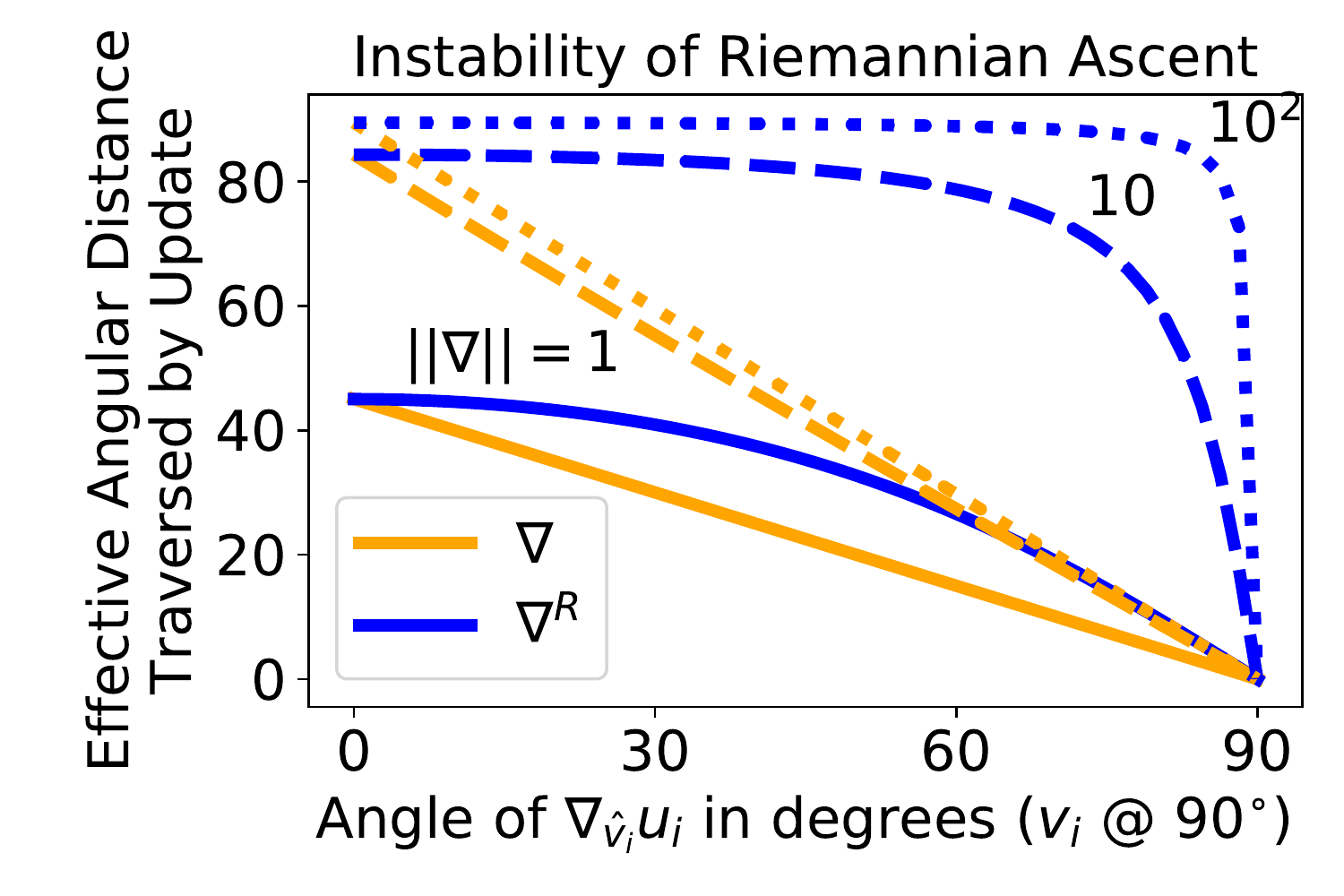}
    \caption{Gradient Instability Near Optimum \label{fig:instability}}
    \end{subfigure}
    \begin{subfigure}[b]{.39\textwidth}
    \includegraphics[scale=0.3]{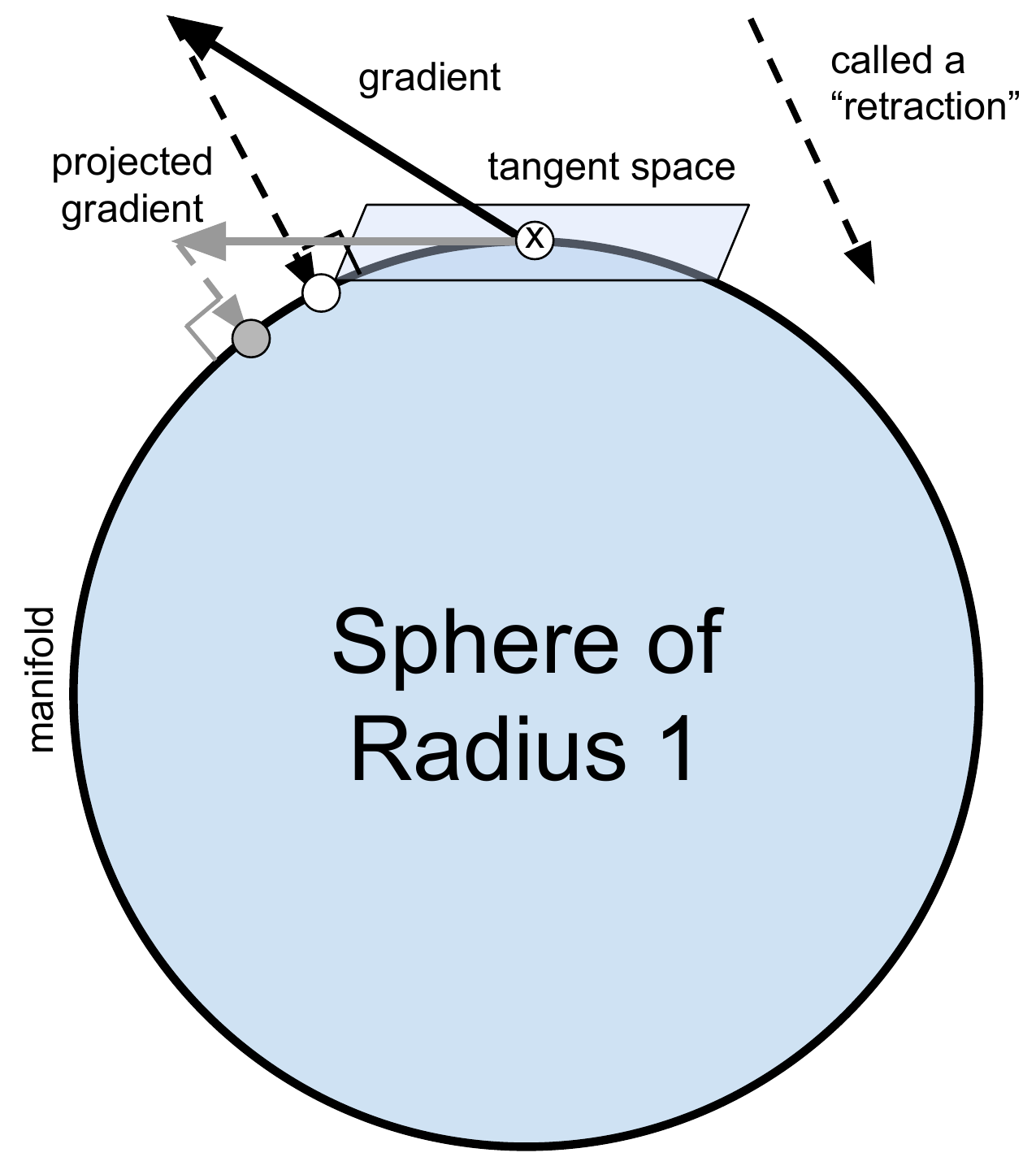}
    \caption{Riemannian Terminology \label{fig:riemann}}
    \end{subfigure}
    \vspace{-5pt}
    \caption{(\subref{fig:instability}) When the $\hat{v}_i$ is near the optimum of its utility and its gradient is nearly orthogonal to the sphere, pointing directly away from the center (@ $90^{\circ}$), the combination of updating using the projected gradient ($\nabla^R$) and the retraction can result in a large update, possibly moving $\hat{v}_i$ away from the optimum. (\subref{fig:riemann}) Diagram presenting Riemannian optimization terminology. The retraction is not a projection in general although our specific choice appears that way for the sphere. A retraction applied at $\hat{v}_i$ takes as input a scaled projected gradient and returns a vector on the manifold: $\hat{v}_i' \leftarrow R_{\hat{v}_i}(\alpha \nabla^R)$.}
    \label{fig:gradient_stability}
\end{figure}

\section{Theoretical comparison with GHA}

\begin{proposition}
\label{gha_equiv_eigengame}
When the first $i-1$ eigenvectors have been learned exactly, GHA on $\hat{v}_i$ is equivalent to projecting the first term in $\nabla_{\hat{v}_i} u_i$ onto the sphere, but omitting to project the second set of penalty terms.
\end{proposition}

\begin{proof}
The GHA update is
\begin{align}
    \Delta \hat{v}_i &= 2 \Big[ M\hat{v}_i - (\hat{v}_i^\top M \hat{v}_i) \hat{v}_i - \sum_{j < i} (\hat{v}_i^\top M \hat{v}_j) \hat{v}_j \Big].
\end{align}
Plugging $v_{j<i}$ for $\hat{v}_{j < i}$ into the GHA update, we find
\begin{align}
    \Delta_i &= 2 \Big[ M\hat{v}_i - (\hat{v}_i^\top M \hat{v}_i) \hat{v}_i - \sum_{j < i} (\hat{v}_i^\top M v_j) v_j \Big]
    \\ &= 2 \Big[ M\hat{v}_i - (\hat{v}_i^\top M \hat{v}_i) \hat{v}_i - \sum_{j < i} \Lambda_{jj} (\hat{v}_i^\top v_j) v_j \Big].
\end{align}
Likewise for the gradient with the first term projected onto the tangent space of sphere:
\begin{align}
    2 \Big[ (I - \hat{v}_i \hat{v}_i^\top) M \hat{v}_i - M \sum_{j < i} \frac{\hat{v}_i^\top M v_j}{v_j^\top M v_j} v_j \Big] &= 2 \Big[ (I - \hat{v}_i \hat{v}_i^\top) M \hat{v}_i - M \sum_{j < i} (\hat{v}_i^\top v_j) v_j \Big]
    \\ &= 2 \Big[ M \hat{v}_i - (\hat{v}_i^\top M \hat{v}_i) \hat{v}_i - \sum_{j < i} \Lambda_{jj} (\hat{v}_i^\top v_j) v_j \Big].
\end{align}
\end{proof}

\begin{proposition}
\label{gha_not_grad}
The GHA update for $\hat{v}_i$ is not the gradient of any function.
\end{proposition}

\begin{proof}
The Jacobian of $\Delta \hat{v}_i$ w.r.t. $\hat{v}_i$ is
\begin{align}
    Jac(\Delta \hat{v}_i) &= 2 \Big[ M - (\hat{v}_i^\top M \hat{v}_i) I - 2 \hat{v}_i \hat{v}_i^\top M - \sum_{j < i} \hat{v}_j \hat{v}_j^\top M \Big].
\end{align}
The sum of the $\hat{v} \hat{v}^\top M$ terms are not, in general, symmetric, therefore, the Jacobian is not symmetric. The Jacobian of a gradient is the Hessian and the Hessian of a function is necessarily symmetric, therefore, the GHA update is not the gradient of any function.
\end{proof}

\subsection{Design Decisions}
We made a number of algorithmic design decisions that led us to the proposed algorithm. The first to note is that a naive utility that simply subtracts off $\sum_{j < i} \langle \hat{v}_i, \hat{v}_j \rangle$ will not solve PCA. This is because large $\langle \hat{v}_i, M \hat{v}_i \rangle$ (read eigenvalues) can drown out these penalties. The intuition is that including $M$ in the inner product gives the right boost to create a natural balance among terms. Next, it is possible to formulate the utilities without normalizing the terms as we did, however, this is harder to analyze and is akin to minimizing $(err)^4$ instead of $(err)^2$ which generally has better convergence properties near optima. Also, while updates formed using the standard Euclidean Gram-Schmidt procedure will solve the PCA problem, they are not the gradients of any utility function. Lastly, our formulation consists entirely of generalized inner products: $\langle \hat{v}_i, M \hat{v}_j \rangle = \langle X \hat{v}_i, X \hat{v}_j \rangle$. Each $X \hat{v}_i$ can be thought of as a shallow function approximator with weights $\hat{v}_i$. This means that our formulation is readily extended to more general function approximation, i.e., $X \hat{v}_i \rightarrow f_i(X)$\footnote{Empirically, replacing $||\hat{v}_i||=1$ with $||\hat{v}_i||\le1$ does not harm performance while the latter is easier to enforce on neural networks for example~\citep{virmaux2018lipschitz}.}. Note that any formulation that operates on $\langle \hat{v}_i, \hat{v}_j \rangle$ instead is not easily generalized.

\section{Nash Proof}
\label{sec:appendix_nash}

Let $\hat{V}$ be a matrix of arbitrary unit-length column vectors ($\hat{v}_j$) and let $M$ (symmetric) be diagonalized as $U \Lambda U^\top$ with $U$ a unitary matrix. Then,
\begin{align}
    R &\myeq \hat{V}^\top M \hat{V} = \hat{V}^\top U \Lambda U^\top \hat{V} = (U^\top \hat{V})^\top \Lambda (U^\top \hat{V}) = Z^\top \Lambda Z \label{simplified}
\end{align}
where $Z$ is also a matrix of unit-length column  vectors because unitary matrices preserve inner products ($\langle U^\top \hat{v}_i, U^\top \hat{v}_i \rangle = \hat{v}_i^\top U U^\top \hat{v}_i = \hat{v}_i^\top \hat{v}_i = 1$). Therefore, rather than considering the action of an arbitrary matrix $\hat{V}$ on $M$, we can consider the action of an arbitrary matrix $Z$ on $\Lambda$. This simplifies the analysis.

In light of this reduction,~\Eqref{eig_deflation} of Theorem~\ref{eigvec_opt_appendix} can be rewritten as
\begin{align}
    u_i(\hat{v}_i \vert v_{j < i}) &= w^\top \Lambda_{jj \ge ii} w
    \\ &= \hat{v}_i^\top \Lambda_{jj \ge ii} \hat{v}_i
\end{align}
because $V$ is identity w.l.o.g. Therefore, player $i$'s problem is simply to find the maximum eigenvector of a transformed matrix $\Lambda_{jj \ge ii}$, i.e., $\Lambda$ with the first $i-1$ eigenvalues removed.

\label{nash_proof}
\begin{theorem}[PCA Solution is the Unique strict-Nash Equilibrium]
\label{eigvec_opt_appendix}
Assume that the top-$k$ eigenvalues of $X^\top X$ are positive and distinct. Then the top-$k$ eigenvectors form the unique strict-Nash equilibrium of the proposed game in \Eqref{obj}.
\end{theorem}

\begin{proof}
In what follows, let $p, q = \{1, \ldots, d\}$ and $i \in \{1, \ldots, k\}$. We will prove optimality of $v_i$ by induction. Clearly, $v_1$ is the optimum of $u_1$ because $u_1 = \langle v_1, M v_1 \rangle = \frac{\langle v_1, M v_1 \rangle}{\langle v_1, v_1 \rangle} = \Lambda_{11}$ is the \emph{Rayleigh} quotient which is known to be maximized for the maximal eigenvalue~\citep{horn2012matrix}. Now, Consider $\hat{v}_i = \sum_{p=1}^d w_p v_p$ as a linear combination of the true eigenvectors. To ensure $||\hat{v}_i||=1$, we require $||w||=1$. Then,
\begin{align}
    u_i(\hat{v}_i \vert v_{j < i}) &= \hat{v}_i^\top M \hat{v}_i - \sum_{j < i} \frac{(\hat{v}_i^\top M v_j)^2}{v_j^\top M v_j} = \hat{v}_i^\top M \hat{v}_i - \sum_{j < i} \frac{(\hat{v}_i^\top M v_j)^2}{\Lambda_{jj}} \\
    &= \Big( \sum_p \sum_q w_p w_q v_p^\top M v_q \Big) - \sum_{j < i} \Big( \sum_p w_p v_p^\top M v_j \Big)^2 / \Lambda_{jj} \\
    &= \Big( \sum_p \sum_q w_p w_q \Lambda_{qq} v_p^\top v_q \Big) - \sum_{j < i} \Big( \sum_p w_p \Lambda_{jj} v_p^\top v_j \Big)^2 / \Lambda_{jj} \\
    &= \sum_q w_q^2 \Lambda_{qq} - \sum_{j < i} \Lambda_{jj} w_j^2 = \sum_{p \ge i} \Lambda_{pp} z_p \label{eig_deflation}
\end{align}
where $z_p = w_p^2$, and $z \in \Delta^{d-1}$
which is a linear optimization problem over the simplex. For distinct $\Lambda_{pp}$ with $\Lambda_{ii} > 0$, $z^* = \argmax(\Lambda_{pp \ge ii}) = e_i$ is unique. Assume each player $i$ plays $e_i$. Any player $j$ that unilaterally deviates from $e_j$ strictly decreases their utility, therefore, the Nash is unique up to a sign change due to $z^* = e_i = w_i^2$. This is expected as both $v_i$ and $-v_i$ are principal components.
\end{proof}

\section{Without the Hierarchy}
\label{no_hier}
In Section~\ref{derivation}, we defined utilities to respect the natural hierarchy of eigenvectors sorted by eigenvalue and mentioned that this eased analysis. Here, we provide further detail as to the difficulty of analyzing the game without the hierarchy. Consider the following alternative definition of the utilities:
\begin{align}
    u_i(\hat{v}_i \vert \hat{v}_{\textcolor{red}{-i}}) &= \hat{v}_i^\top M \hat{v}_i - \sum_{\textcolor{red}{j \ne i}} \frac{(\hat{v}_i^\top M \hat{v}_j)^2}{\hat{v}_j^\top M \hat{v}_j} \label{obj_nohier}
\end{align}
where the sum is now over all $j \ne i$ instead of $j < i$ as in~\Eqref{obj}. With this form, the game is now symmetric across all players $i$. Despite the symmetry of the game, we can easily rule out the existence of a symmetric Nash.

\begin{proposition}
The \pcagame{} defined using symmetric utilities in~\Eqref{obj_nohier} does not contain a symmetric Nash equilibrium (assuming $k \ge 2$ and $rank(M) \ge 2$).
\end{proposition}
\begin{proof}[Proof by Contradiction]
Assume a symmetric Nash exists, i.e., $\hat{v}_i = \hat{v}_j$ for all $i, j$. The utility of a symmetric Nash using equation~\Eqref{obj_nohier} is
\begin{align}
    u_i(\hat{v}_i \vert \hat{v}_{-i}) &= (1-(n-1)) (\hat{v}_i^\top M \hat{v}_i) = (2-n) (\hat{v}_i^\top M \hat{v}_i) \le 0.
\end{align}
Consider a unilateral deviation of $\hat{v}_i$ to a direction orthogonal to $\hat{v}_i$, i.e., $\hat{v}_{\perp} \perp \hat{v}_i$ such that
\begin{align}
    u_i(\hat{v}_\perp, \hat{v}_{-i}) &= (\hat{v}_\perp^\top M \hat{v}_\perp) > 0.
\end{align}
This utility is positive because $rank(M) \ge 2$ and therefore, always greater than the supposed Nash. Therefore, there is no symmetric Nash.
\end{proof}

We can also prove that the true PCA solution is \textbf{a} Nash of this version of \pcagame{}.

\begin{proposition}
The the top-$k$ eigenvectors of $M$ form a strict-Nash equilibrium of the \pcagame{} defined using symmetric utilities in~\Eqref{obj_nohier} (assuming $rank(M) \ge k$).
\end{proposition}
\begin{proof}
Let $\hat{v}_i = v_i$. We will assume this standard ordering, however, the proof follows through for any permutation of the eigenvectors. Clearly, the largest eigenvector is a best response to the spectrum because the penalty term (2nd term in \Eqref{obj_nohier}) cannot be decreased below zero and the Rayliegh term (first term) is maximal, i.e., $v_1 = \argmax_{\hat{v}_1} u_1(\hat{v}_1, v_{-1})$. So assume $v_i$ is another eigenvector and consider representing $\hat{v}_i$ as $\hat{v}_i = \sum_{p=1}^d w_p v_p$ as before in Section~\ref{sec:appendix_nash}. Repeating those same steps, we find
\begin{align}
    u_i(\hat{v}_i, v_{-i}) &= \sum_q w_q^2 \Lambda_{qq} - \sum_{\textcolor{red}{j \ne i}} \Lambda_{jj} w_j^2 = \Lambda_{ii} z_i
\end{align}
where $z_k = w_k^2$, $z \in \Delta^{n-1}$. Assuming $\Lambda_{ii} > 0$, this objective is uniquely maximized for $z_i = 1$ and $z_k = 0$ for all $k \ne i$. Therefore, $v_i = \argmax_{\hat{v}_i} u_i(\hat{v}_i, v_{-i})$.

\end{proof}

However, we were unable to prove that it is the \textbf{only} Nash. It is possible that other Nash equilibria exist. Instead of focusing on determining whether a second Nash equilibrium exists (which is NP-hard~\citep{daskalakis2009complexity,gilboa1989nash}), we learned through experiments that the \pcagame{} variant that incorporates knowledge of the hierarchy is much more performant. We leave determininig uniquess of the PCA solution for the less performant variant as an academic exercise.

\newpage
\section{Error Propagation}
\label{app:err_prop}
\subsection{Generalities}

\paragraph{Notation.} We can parameterize a vector on the sphere using the Riemannian exponential map, Exp, applied to a vector deviation from an anchor point. Formally, let $\hat{v}_j = \text{Exp}_{v_j}(\theta_j\Delta_j) = \cos(\theta_j) v_j + \sin(\theta_j) \Delta_j$ where $v_j$ is the $j$th largest eigenvector and $\Delta_j \in \mathcal{S}^{d-1}$ is such that $\langle \Delta_j, v_j \rangle = 0$. Therefore, $\theta_j$ measures how far $\hat{v}_j$ deviates from $v_j$ in radians and $\Delta_j$ denotes the direction of deviation.

Let $\Lambda_{ii}$ denote the $i$th largest eigenvalue and $v_i$ the associated eigenvector. Also define the eigenvalue gap $g_i = \Lambda_{ii} - \Lambda_{i+1,i+1}$. Finally, let $\kappa_i = \frac{\Lambda_{11}}{\Lambda_{ii}}$ denote the $i$th condition number.

The following Lemma decomposes the utility of a player when the parents have learnt the preceding eigenvectors perfectly. 

\begin{lemma}
\label{player_i_obj}
Let $\hat{v}_i = \cos(\theta_i) v_i + \sin(\theta_i) \Delta_i$ without loss of generality. Then
\begin{align}
    u_i(\hat{v}_i, v_{j < i}) &= u_i(v_i, v_{j < i}) - \sin^2(\theta_i) \Big( \Lambda_{ii} - \sum_{l > i} z_l \Lambda_{ll} \Big).
\end{align}
\end{lemma}

\begin{proof}
Note that $\Delta_i$ can also be decomposed as $\Delta_i = \sum_{l=1}^d w_l v_l, ||w||=1$ without loss of generality and that by Theorem~\ref{eigvec_opt_appendix}, this implies $u_i(\Delta_i, v_{j < i}) = \sum_{l \ge i} z_l \Lambda_{ll}$. This can be simplified further because $\langle \Delta_i, v_i \rangle = 0$ by its definition, which implies that $z_i = 0$. Therefore, more precisely, $u_i(\Delta_i, v_{j < i}) = \sum_{l > i} z_l \Lambda_{ll}$. Continuing we find
\begin{align}
    u_i(\hat{v}_i, v_{j < i}) &= \langle \hat{v}_i, \Lambda \hat{v}_i \rangle - \sum_{j < i} \frac{\langle \hat{v}_i, \Lambda v_j \rangle^2}{\langle v_j, \Lambda v_j \rangle}
    \\ &= \langle \hat{v}_i, \Lambda \hat{v}_i \rangle - \sum_{j < i} \Lambda_{jj} \langle \hat{v}_i, v_j \rangle^2
    \\ &= (\cos^2(\theta_i) \Lambda_{ii} + \sin^2(\theta_i) \langle \Delta_i, \Lambda \Delta_i \rangle) - \sum_{j < i} \Lambda_{jj} \langle \cos(\theta_i) v_i + \sin(\theta_i) \Delta_i, v_j \rangle^2
    \\ &= (\cos^2(\theta_i) \Lambda_{ii} + \sin^2(\theta_i) \langle \Delta_i, \Lambda \Delta_i \rangle) - \sum_{j < i} \Lambda_{jj} \sin^2(\theta_i) \langle \Delta_i, v_j \rangle^2
    \\ &= \Lambda_{ii} - \sin^2(\theta_i) \Lambda_{ii} + \sin^2(\theta_i)\Big[ \langle \Delta_i, \Lambda \Delta_i \rangle - \sum_{j < i} \Lambda_{jj} \langle \Delta_i, v_j \rangle^2 \Big]
    \\ &= u_i(v_i, v_{j < i}) - \sin^2(\theta_i) \Big( \Lambda_{ii} - u_i(\Delta_i, v_{j < i}) \Big)
    \\ &= u_i(v_i, v_{j < i}) - \sin^2(\theta_i) \Big( \Lambda_{ii} - \sum_{l > i} z_l \Lambda_{ll} \Big). \quad \text{[T\ref{eigvec_opt_appendix}]}
\end{align}
\end{proof}

\subsection{Summary of Error Propagation Results}
Player $i$'s utility is sinusoidal in the angular deviation of $\theta_i$ from the optimum. The amplitude of the sinusoid varies with the direction of the angular deviation along the sphere and is dependent on the accuracy of players $j < i$. In the special case where players $j < i$ have learned the top-$(i-1)$ eigenvectors exactly, player $i$'s utility simplifies (see Lemma~\ref{player_i_obj}) to
\begin{align}
    u_i(\hat{v}_i, v_{j < i}) &= \Lambda_{ii} - \sin^2(\theta_i) \Big( \Lambda_{ii} - \sum_{l > i} z_l \Lambda_{ll} \Big).
\end{align}
Note that $\sin^2$ has period $\pi$ as opposed to $2\pi$, which simply reflects the fact that $v_i$ and $-v_i$ are both eigenvectors.

The angular distance between $v_i$ and the maximizer of player $i$'s utility with approximate parents has $\tan^{-1}$ dependence (i.e., a soft step-function; see Lemma~\ref{arctan_error_prop}). Figure~\ref{fig:arctan_err_dep} plots the dependence for a synthetic problem. This dependence reveals that there is an error threshold players $j < i$ must fall below in order for player $i$ to accurately learn the $i$-th eigenvector.
\begin{figure}[ht]
    \centering
    \includegraphics[scale=0.4]{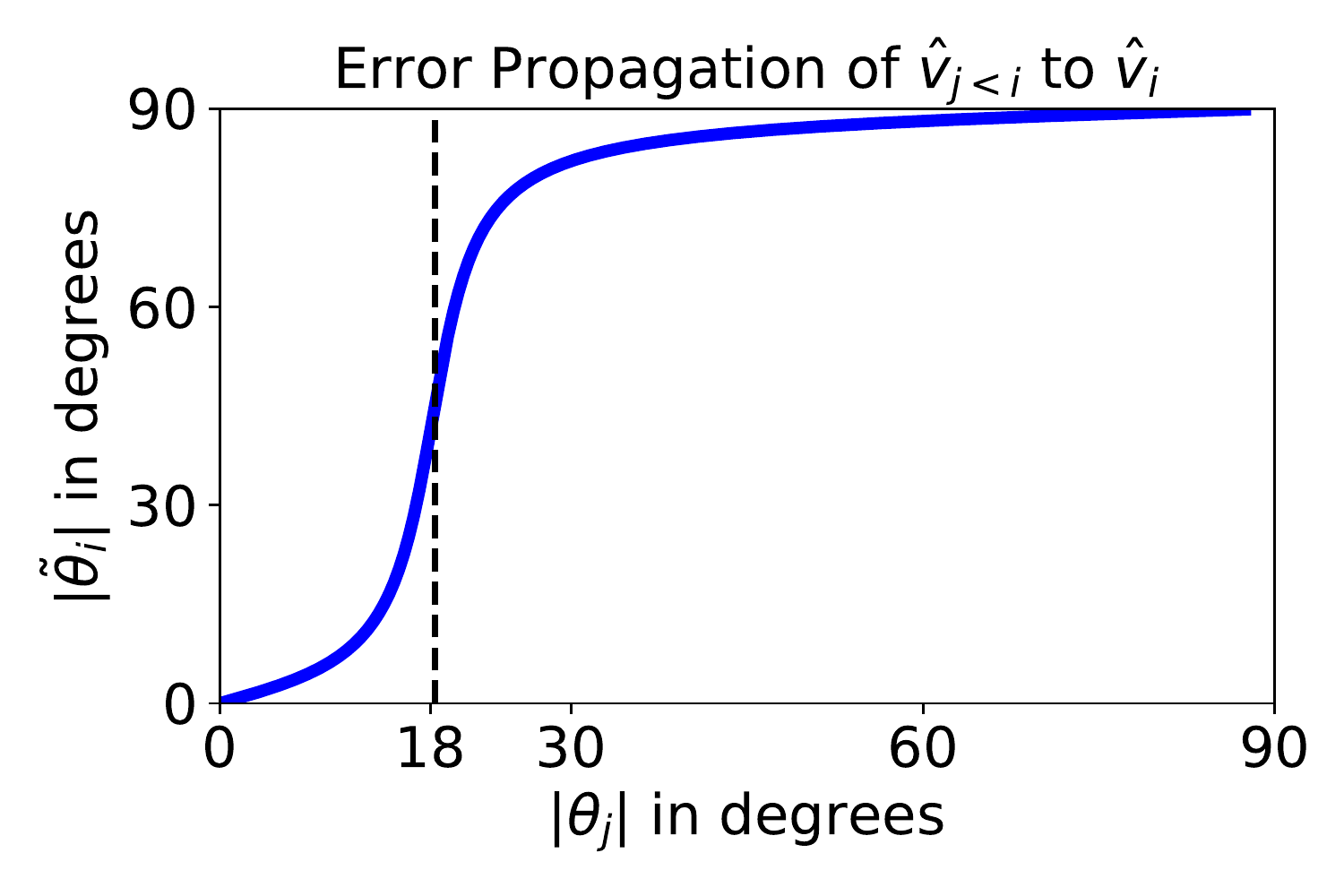}
    \caption{Example~\ref{error_prop_example} demonstrates that the angular error ($x$-axis) in the learned parents $\hat{v}_{j < i}$ must fall below a threshold (e.g., $\approx 18^{\circ}$ here) in order for the maximizer of player $i$'s utility to lie near the true $i$th eigenvector ($y$-axis). The matrix $M$ for this example has a condition number $\kappa_i = \frac{\Lambda_{11}}{\Lambda_{ii}} = 10$.}
    \label{fig:arctan_err_dep}
\end{figure}

\subsection{Theorem and Proofs}

In Theorem~\ref{parent_err_to_child_err}, we prove that given parents close enough to their corresponding true eigenvectors, the angular deviation of a local maximizer of a child's utility from the child's true eigenvector is below a derived threshold. In other words, given accurate parents, a child can succesfully proceed to approximate its corresponding eigenvector (its utility is well posed). We prove this theorem in several steps.

First we show in Lemma~\ref{misspec_loss_function} that the child's utility function can be written as a composition of sinusoids with dependence on the angular deviation from the child's true eigenvector. The amplitude of the sinusoid depends on the directions in which the child and parents have deviated from their true eigenvectors along their spheres. We then simplify the composition of sinusoids to a single sinusoid in Lemma~\ref{ui_sinusoidal}. Any local max of a sinusoid is also a global max. Therefore, to upper bound the angular deviatiation of the child's local maximizer from its true corresponding eigenvector, we consider the worst case direction for the maximizer to deviate from the true eigenvector.

In Lemma~\ref{arctan_error_prop}, we give a closed form solution for the angular deviation of a maximizer of a child's utility given any parents and deviation directions. This dependence is given by the $\arctan$ function which resembles a \emph{soft} step function with a linear regime for small angular deviations, followed by a step, and then another linear regime for large angular deviations. The argument of the $\arctan$ is a ratio of terms, each with dependence on the parents' angular deviations and directions of deviation. We establish two minor lemmas, Lemma~\ref{algebra_trick} and Lemma~\ref{ui_of_deltai}, to help bound the denominator in Lemma~\ref{A_upperbnd}. We then tighten the bounds on the ratio assuming parents with error below a certain threshold (``left'' of the step) in Lemmas~\ref{A_upperbnd_simple},~\ref{B_upperbnd_simple}, and~\ref{tan_arg_bound}. Finally, using these bounds on the argument to the $\arctan$, we are able to bound the angular deviation of any maximizer of the child's utility in Lemma~\ref{parent_err_to_child_err} given any deviation direction for the child or parents.

\begin{theorem}
\label{parent_err_to_child_err}
Assume it is given that $\vert \theta_j \vert \le \frac{c_i g_i}{(i-1)\Lambda_{11}} \le \sqrt{\frac{1}{2}}$ for all $j < i$ with $0 \le c_i \le \frac{1}{16}$. Then
\begin{align}
    \vert \theta_i^* \vert = \vert \argmax_{\theta_i} u_i(\hat{v}_i(\theta_i, \Delta_i), \hat{v}_{j < i}) \vert &\le 8c_i. 
\end{align}
\end{theorem}

\begin{proof}
By Lemma~\ref{tan_arg_bound}, $A < 0$ for $c_i < \frac{1}{8}$. Therefore, $\vert \theta_i^* \vert = \frac{1}{2} \tan^{-1}\Big\vert \frac{B}{A} \Big\vert$ by Lemma~\ref{arctan_error_prop}. Also, note that for $z \le \frac{1}{2}$, $\tan^{-1}(\vert z \vert) \le \vert z \vert$. Setting $c_i \le \frac{1}{16}$ to ensures $z = \vert \frac{B}{A} \vert \le \frac{1}{2}$. Then,
\begin{align}
    \vert \theta_i^* \vert &= \frac{1}{2} \tan^{-1}\Big\vert \frac{B}{A} \Big\vert \le \frac{1}{2} \vert \frac{B}{A} \vert \stackrel{L\textcolor{cobalt}{\ref{tan_arg_bound}}}{\le} \frac{1}{2} \frac{8c}{1 - 8c_i}
    \le 8c_i.
\end{align}
\end{proof}

\begin{lemma}
\label{misspec_loss_function}
Let $\hat{v}_j = \cos(\theta_j) v_j + \sin(\theta_j) \Delta_j$ for all $j \le i$ without loss of generality. Then
\begin{align}
    u_i(\hat{v}_i, \hat{v}_{j < i}) &= \textcolor{orange}{A(\theta_j, \Delta_j, \Delta_i)} \sin^2(\theta_i) - \textcolor{blue}{B(\theta_j, \Delta_j, \Delta_i)} \frac{\sin(2 \theta_i)}{2} + \textcolor{green}{C(\theta_j, \Delta_j, \Delta_i)}
\end{align}
where
\begin{align}
    \textcolor{orange}{A(\theta_j, \Delta_j, \Delta_i)} &= ||\Delta_i||_{\Lambda^{-1}} - \Lambda_{ii}
    \\ &- \sum_{j < i} \frac{\Lambda_{jj}^2 \cos^2(\theta_j) \langle \Delta_i, v_j \rangle^2 - \Lambda_{ii}^2 \sin^2(\theta_j) \langle \Delta_j, v_i \rangle^2 + \sin^2(\theta_j) \langle \Delta_i, \Lambda \Delta_j \rangle^2}{\Lambda_{jj} \cos(\theta_j)^2 + ||\Delta_j||_{\Lambda^{-1}}  \sin^2(\theta_j)}
    \\ &- \sum_{j < i} \frac{\Lambda_{jj} \sin(2\theta_j) \langle \Delta_i, v_j \rangle \langle \Delta_i, \Lambda \Delta_j \rangle}{\Lambda_{jj} \cos(\theta_j)^2 +||\Delta_j||^2_{\Lambda^{-1}}  \sin^2(\theta_j)}
    \\ \textcolor{blue}{B(\theta_j, \Delta_j, \Delta_i)} &= \sum_{j < i} \frac{\Lambda_{ii} \Lambda_{jj} \sin(2\theta_j) \langle \Delta_j, v_i \rangle \langle \Delta_i, v_j \rangle + 2\Lambda_{ii} \sin^2(\theta_j) \langle \Delta_j, v_i \rangle \langle \Delta_i, \Lambda \Delta_j \rangle}{\Lambda_{jj} \cos(\theta_j)^2 + ||\Delta_j||^2_{\Lambda^{-1}}  \sin^2(\theta_j)}
    \\ \textcolor{green}{C(\theta_j, \Delta_j, \Delta_i)} &= \Lambda_{ii} - \sum_{j < i} \frac{\Lambda_{ii}^2 \sin^2(\theta_j) \langle \Delta_j, v_i \rangle^2}{\Lambda_{jj} \cos(\theta_j)^2 + ||\Delta_j||^2_{\Lambda^{-1}}  \sin^2(\theta_j)}.
\end{align}
We abbreviate the above to \textcolor{orange}{A}, \textcolor{blue}{B}, \textcolor{green}{C} to avoid clutter in all upcoming statements and proofs. These functions are dependent on all variables \textbf{except} $\theta_i$.
\end{lemma}

\begin{proof}
Note that the true eigenvectors are orthogonal, so in what follows, any $\langle v_i, v_j \rangle = 0$ where $j \ne i$. Also, recall that $2\sin(z)\cos(z) = \sin(2z)$. We highlight some but not all such simplifications. Finally, we recognize $\langle \Delta_i, \Lambda \Delta_i \rangle = ||\Delta_i||_{\Lambda^{-1}}$ as the generalized norm of $\Delta_i$ or the Mahalanobis distance from the origin.
\begingroup
\allowdisplaybreaks
\begin{align}
    &u_i(\hat{v}_i, \hat{v}_{j < i})
    \\ &= \langle \hat{v}_i, \Lambda \hat{v}_i \rangle - \sum_{j < i} \frac{\langle \hat{v}_i, \Lambda \hat{v}_j \rangle^2}{\langle \hat{v}_j, \Lambda \hat{v}_j \rangle}
    \\ &= \langle \cos(\theta_i) v_i + \sin(\theta_i) \Delta_i, \Lambda \big( \cos(\theta_i) v_i + \sin(\theta_i) \Delta_i \big) \rangle \nonumber
    \\ &- \sum_{j < i} \frac{\langle \cos(\theta_i) v_i + \sin(\theta_i) \Delta_i, \Lambda \big( \cos(\theta_j) v_j + \sin(\theta_j) \Delta_j \big) \rangle^2}{\langle \cos(\theta_j) v_j + \sin(\theta_j) \Delta_j, \Lambda \big( \cos(\theta_j) v_j + \sin(\theta_j) \Delta_j \big) \rangle}
    \\ &= \Lambda_{ii} \cos(\theta_i)^2 + \textcolor{blue}{\langle \Delta_i, \Lambda \Delta_i \rangle} \sin^2(\theta_i) \nonumber
    \\ &- \sum_{j < i} \frac{\langle \cos(\theta_i) v_i + \sin(\theta_i) \Delta_i, \Lambda \big( \cos(\theta_j) v_j + \sin(\theta_j) \Delta_j \big) \rangle^2 }{\Lambda_{jj} \cos(\theta_j)^2 + \langle \Delta_j, \Lambda \Delta_j \rangle \sin^2(\theta_j)}
    \\ &= \textcolor{orange}{\Lambda_{ii} \cos(\theta_i)^2} + \textcolor{blue}{||\Delta_i||^2_{\Lambda^{-1}}} \sin^2(\theta_i) \nonumber
    \\ &- \sum_{j < i} \frac{\big( \Lambda_{jj} \sin(\theta_i) \cos(\theta_j) \langle \Delta_i, v_j \rangle + \Lambda_{ii} \sin(\theta_j) \cos(\theta_i) \langle \Delta_j, v_i \rangle + \sin(\theta_i) \sin(\theta_j) \langle \Delta_i, \Lambda \Delta_j \rangle \big)^2}{\Lambda_{jj} \cos(\theta_j)^2 + ||\Delta_j||^2_{\Lambda^{-1}}  \sin^2(\theta_j)}.
    \end{align}
    \endgroup
Developing the numerator of the fraction, we obtain terms in $\sin $ and in $\sin^2$ that we later regroup to obtain the result:
\begingroup
\begin{align}
    &= \textcolor{orange}{\Lambda_{ii}  - \Lambda_{ii} \sin(\theta_i)^2} + ||\Delta_i||^2_{\Lambda^{-1}} \sin^2(\theta_i) \nonumber
    \\ &- \sum_{j < i} \frac{\Lambda_{jj}^2 \sin^2(\theta_i) \cos^2(\theta_j) \langle \Delta_i, v_j \rangle^2 + \Lambda_{ii}^2 \sin^2(\theta_j) \cos^2(\theta_i) \langle \Delta_j, v_i \rangle^2 + \sin^2(\theta_i) \sin^2(\theta_j) \langle \Delta_i, \Lambda \Delta_j \rangle^2}{\Lambda_{jj} \cos(\theta_j)^2 + ||\Delta_j||^2_{\Lambda^{-1}}  \sin^2(\theta_j)}
    \\ &- \textcolor{blue}{2}\sum_{j < i} \frac{\Lambda_{ii} \Lambda_{jj} \textcolor{blue}{\sin(\theta_i)} \textcolor{green}{\sin(\theta_j)} \textcolor{blue}{\cos(\theta_i)} \textcolor{green}{\cos(\theta_j)} \langle \Delta_j, v_i \rangle \langle \Delta_i, v_j \rangle}{\Lambda_{jj} \cos(\theta_j)^2 + ||\Delta_j||^2_{\Lambda^{-1}}  \sin^2(\theta_j)}
    \\ &- 2\sum_{j < i} \frac{\Lambda_{jj} \sin^2(\theta_i) \sin(\theta_j) \cos(\theta_j) \langle \Delta_i, v_j \rangle \langle \Delta_i, \Lambda \Delta_j \rangle}{\Lambda_{jj} \cos(\theta_j)^2 +||\Delta_j||^2_{\Lambda^{-1}}  \sin^2(\theta_j)}
    \\ &- 2\sum_{j < i} \frac{\Lambda_{ii} \sin(\theta_i) \cos(\theta_i) \sin^2(\theta_j) \langle \Delta_j, v_i \rangle \langle \Delta_i, \Lambda \Delta_j \rangle}{\Lambda_{jj} \cos(\theta_j)^2 + ||\Delta_j||^2_{\Lambda^{-1}}  \sin^2(\theta_j)}
    \\ &= \Lambda_{ii}  - \Lambda_{ii} \sin^2(\theta_i) + ||\Delta_i||^2_{\Lambda^{-1}} \sin^2(\theta_i) \nonumber
    \\ &- \sum_{j < i} \frac{\Lambda_{jj}^2 \sin^2(\theta_i) \cos^2(\theta_j) \langle \Delta_i, v_j \rangle^2 + \Lambda_{ii}^2 \sin^2(\theta_j) \cos^2(\theta_i) \langle \Delta_j, v_i \rangle^2 + \sin^2(\theta_i) \sin^2(\theta_j) \langle \Delta_i, \Lambda \Delta_j \rangle^2}{\Lambda_{jj} \cos(\theta_j)^2 + ||\Delta_j||^2_{\Lambda^{-1}}  \sin^2(\theta_j)}
    \\ &- \frac{1}{2} \sum_{j < i} \frac{\Lambda_{ii} \Lambda_{jj} \textcolor{blue}{\sin(2\theta_i)} \textcolor{green}{\sin(2\theta_j)} \langle \Delta_j, v_i \rangle \langle \Delta_i, v_j \rangle}{\Lambda_{jj} \cos(\theta_j)^2 + ||\Delta_j||^2_{\Lambda^{-1}}  \sin^2(\theta_j)}
    \\ &- \sum_{j < i} \frac{\Lambda_{jj} \sin^2(\theta_i) \sin(2\theta_j) \langle \Delta_i, v_j \rangle \langle \Delta_i, \Lambda \Delta_j \rangle}{\Lambda_{jj} \cos(\theta_j)^2 +||\Delta_j||^2_{\Lambda^{-1}}  \sin^2(\theta_j)}
    \\ &- \sum_{j < i} \frac{\Lambda_{ii} \sin(2\theta_i) \sin^2(\theta_j) \langle \Delta_j, v_i \rangle \langle \Delta_i, \Lambda \Delta_j \rangle}{\Lambda_{jj} \cos(\theta_j)^2 + ||\Delta_j||^2_{\Lambda^{-1}}  \sin^2(\theta_j)}.
\end{align}
\endgroup

Collecting terms, we find
\begin{align}
    &u_i(\hat{v}_i, \hat{v}_{j < i})
    \\ &= \sin^2(\theta_i) \textcolor{orange}{\Big[} ||\Delta_i||^2_{\Lambda^{-1}} - \Lambda_{ii}
    \\ &- \sum_{j < i} \frac{\Lambda_{jj}^2 \cos^2(\theta_j) \langle \Delta_i, v_j \rangle^2 - \Lambda_{ii}^2 \sin^2(\theta_j) \langle \Delta_j, v_i \rangle^2 + \sin^2(\theta_j) \langle \Delta_i, \Lambda \Delta_j \rangle^2}{\Lambda_{jj} \cos(\theta_j)^2 + ||\Delta_j||^2_{\Lambda^{-1}}  \sin^2(\theta_j)}
    \\ &- \sum_{j < i} \frac{\Lambda_{jj} \sin(2\theta_j) \langle \Delta_i, v_j \rangle \langle \Delta_i, \Lambda \Delta_j \rangle}{\Lambda_{jj} \cos(\theta_j)^2 +||\Delta_j||^2_{\Lambda^{-1}}  \sin^2(\theta_j)} \textcolor{orange}{\Big]}
    \\ &- \frac{\sin(2\theta_i)}{2} \textcolor{blue}{\Big[} \sum_{j < i} \frac{\Lambda_{ii} \Lambda_{jj} \sin(2\theta_j) \langle \Delta_j, v_i \rangle \langle \Delta_i, v_j \rangle + 2\Lambda_{ii} \sin^2(\theta_j) \langle \Delta_j, v_i \rangle \langle \Delta_i, \Lambda \Delta_j \rangle}{\Lambda_{jj} \cos(\theta_j)^2 + ||\Delta_j||^2_{\Lambda^{-1}}  \sin^2(\theta_j)} \textcolor{blue}{\Big]}
    \\ &+ \textcolor{green}{\Big[} \Lambda_{ii} - \sum_{j < i} \frac{\Lambda_{ii}^2 \sin^2(\theta_j) \langle \Delta_j, v_i \rangle^2}{\Lambda_{jj} \cos(\theta_j)^2 + ||\Delta_j||^2_{\Lambda^{-1}}  \sin^2(\theta_j)} \textcolor{green}{\Big]}
    \\ &\myeq \textcolor{orange}{A} \sin^2(\theta_i) - \textcolor{blue}{B} \frac{\sin(2\theta_i)}{2} + \textcolor{green}{C}.
\end{align}
\end{proof}

\begin{lemma}
\label{ui_sinusoidal}
The utility function along $\Delta_i$, $\theta:\mapsto u_i(\hat{v}_i(\theta_i, \Delta_i), \hat{v}_{j < i})$, is sinusoidal with period $\pi$:
\begin{align}
    u_i(\hat{v}_i(\theta_i, \Delta_i), \hat{v}_{j < i}) &=\frac{1}{2} \Big[ \sqrt{A^2 + B^2} \cos(2 \theta_i + \phi) + A + 2C \Big]
\end{align}
where $\phi = \tan^{-1}\Big( \frac{B}{A} \Big)$.
\end{lemma}

\begin{proof}
Starting from Lemma~\ref{misspec_loss_function}, we find
\begin{align}
u_i(\hat{v}_i(\theta_i, \Delta_i), \hat{v}_{j < i}) &= A \sin^2(\theta_i) - B \frac{\sin(2 \theta_i)}{2} + C
\\ &= A \frac{1 - \cos(2 \theta_i)}{2} - B \frac{\sin(2 \theta_i)}{2} + C
\\ &= \frac{1}{2} \Big[ -A \cos(2 \theta_i) - B \sin(2 \theta_i) + A + 2C \Big]
\\ &= \frac{1}{2} \Big[ \sqrt{A^2 + B^2} \cos(2 \theta_i + \phi) + A + 2C \Big]
\end{align}
where $\phi = \tan^{-1}\Big( \frac{B}{A} \Big)$.
\end{proof}

\begin{lemma}
\label{arctan_error_prop}
The angular deviation, $\theta_i$, of the vector that maximizes the mis-specified objective, $\argmax_{\theta_i} u_i(\hat{v}_i(\theta_i, \Delta_i), \hat{v}_{j < i})$, is given by
\begin{align}
    \vert \theta_i^* \vert &= \begin{cases}
        \frac{1}{2} \tan^{-1}\Big( \vert \frac{B}{A} \vert \Big) & \text{ if } A < 0
        \\ \frac{\pi}{4} & \text{ if } A = 0
        \\ \frac{1}{2} \Big[ \pi - \tan^{-1}\Big( \vert \frac{B}{A} \vert \Big) \Big] & \text{ if } A > 0
    \end{cases}
\end{align}
where $A$ and $B$ are given by Lemma~\ref{misspec_loss_function}.
\end{lemma}

\begin{proof}
First, we identify the critical points:
\begin{align}
    \frac{\partial}{\partial \theta_i} u_i(\hat{v}_i, \hat{v}_{j < i}) &= 2A \sin(\theta_i) \cos(\theta_i) - B \cos(2\theta_i) = 0
    \\ &= A \sin(2\theta_i) - B \cos(2\theta_i) = 0
    \\ &= \frac{1}{\cos(2\theta_i)} [ \tan(2 \theta_i) A - B] = 0
    \\ \tan(2 \theta_i) &= \frac{B}{A}.
\end{align}

Then we determine maxima vs minima:
\begin{align}
    \frac{\partial^2}{\partial \theta_i} u_i(\hat{v}_i, \hat{v}_{j < i}) &= \frac{2}{\cos(2\theta_i)} [ B \tan(2\theta_i) + A ] = \frac{2}{\cos(2\theta_i)} [ \frac{B^2}{A} + A ],
\end{align}
therefore, $\texttt{sign}(\frac{\partial^2}{\partial \theta_i} u_i) = \texttt{sign}(\cos(2\theta_i)) \texttt{sign}(A) < 0$ for $\theta_i$ to be a local maximum. If $A < 0$, then $\theta_i^*$ must lie within $[-\frac{\pi}{4}, \frac{\pi}{4}]$. If $A > 0$, then $\theta_i^*$ must lie within $[-\frac{\pi}{2}, -\frac{\pi}{4}]$ or $[\frac{\pi}{4}, \frac{\pi}{2}]$. By inspection, if $A=0$, then $u_i$ is maximized at $\theta_i = -\frac{\pi}{4} \texttt{sign}(B)$. In general, we are interested in the magnitude of $\theta_i$, not its sign.
\end{proof}

\begin{lemma}
\label{algebra_trick}
The following relationship is useful for proving Lemma~\ref{A_upperbnd}:
\begin{align}
    \frac{b}{a+c} &= \frac{b}{a} \Big[ 1 - \frac{c}{a+c} \Big]
\end{align}
\end{lemma}
\begin{proof}
\begin{align}
    \frac{b}{a+c} &= \frac{b}{a} + x
    \\ \implies x &= \frac{b}{a+c} - \frac{b}{a} = b \Big[ \frac{1}{a+c} - \frac{1}{a} \Big]
    \\ &= b \Big[ \frac{a - (a + c)}{a(a+c)} \Big] = -\frac{b}{a} \Big[\frac{c}{a+c} \Big].
\end{align}
\end{proof}

\begin{lemma}
\label{ui_of_deltai}
If $\langle \Delta_i, v_i \rangle = 0$, then $u_i(\Delta_i, v_{j < i}) \le \Lambda_{i+1,i+1}$.
\end{lemma}

\begin{proof}
Recall the Nash proof in Appendix~\ref{nash_proof}:
\begin{align}
    u_i(\Delta_i, v_{j < i}) &= \sum_{p \ge i} \Lambda_{pp} z_p
\end{align}
where $z_p = w_p^2, \Delta_i = \sum_{p=1}^d w_p v_p$, and $z \in \Delta^{d-1}$. The fact that $\langle \Delta_i, v_i \rangle = 0$ implies that $z_i = 0$. Therefore, the utility simplifies to
\begin{align}
    u_i(\Delta_i, v_{j < i}) &= \sum_{p \ge i+1} \Lambda_{pp} z_p
\end{align}
which is upper bounded by $\Lambda_{i+1,i+1}$.
\end{proof}

\begin{lemma}
\label{A_upperbnd}
Assume $\vert \theta_j \vert \le \epsilon$ for all $j < i$ (implies $\sin^2(\theta_j) \le \epsilon^2$). Then
\begin{align}
    A &\le -g_i + (i-1) (\Lambda_{11} + \Lambda_{ii}) \frac{\epsilon^2}{1 - \epsilon^2} + 2 (i-1) \Lambda_{11} \frac{\epsilon}{\sqrt{1 - \epsilon^2}}.
\end{align}
\end{lemma}

\begin{proof}
\begingroup
\allowdisplaybreaks
\begin{align}
    &A(\theta_{j<i}) \nonumber
    \\ &= ||\Delta_i||^2_{\Lambda^{-1}} - \Lambda_{ii} \nonumber
    \\ &- \sum_{j < i} \frac{\Lambda_{jj}^2 \cos^2(\theta_j) \langle \Delta_i, v_j \rangle^2 - \Lambda_{ii}^2 \sin^2(\theta_j) \langle \Delta_j, v_i \rangle^2 + \sin^2(\theta_j) \langle \Delta_i, \Lambda \Delta_j \rangle^2}{\Lambda_{jj} \cos(\theta_j)^2 + ||\Delta_j||^2_{\Lambda^{-1}}  \sin^2(\theta_j)} \nonumber
    \\ &- \sum_{j < i} \frac{\Lambda_{jj} \sin(2\theta_j) \langle \Delta_i, v_j \rangle \langle \Delta_i, \Lambda \Delta_j \rangle}{\Lambda_{jj} \cos(\theta_j)^2 +||\Delta_j||^2_{\Lambda^{-1}}  \sin^2(\theta_j)}
    \\ &= ||\Delta_i||^2_{\Lambda^{-1}} - \sum_{j < i} \textcolor{blue}{\frac{\Lambda_{jj}^2 \cos^2(\theta_j) \langle \Delta_i, v_j \rangle^2}{\Lambda_{jj} \cos^2(\theta_j) +||\Delta_j||^2_{\Lambda^{-1}}  \sin^2(\theta_j)}} - \Lambda_{ii} \nonumber
    \\ &- \sum_{j < i} \frac{- \Lambda_{ii}^2 \sin^2(\theta_j) \langle \Delta_j, v_i \rangle^2 + \sin^2(\theta_j) \langle \Delta_i, \Lambda \Delta_j \rangle^2 + \Lambda_{jj} \sin(2\theta_j) \langle \Delta_i, v_j \rangle \langle \Delta_i, \Lambda \Delta_j \rangle}{\Lambda_{jj} \cos^2(\theta_j) + ||\Delta_j||^2_{\Lambda^{-1}}  \sin^2(\theta_j)}
    \\ &\stackrel{\text{[\textcolor{cobalt}{L\ref{algebra_trick}}]}}{=} ||\Delta_i||^2_{\Lambda^{-1}} - \sum_{j < i} \textcolor{blue}{\frac{\Lambda_{jj}^2 \cos^2(\theta_j) \langle \Delta_i, v_j \rangle^2}{\Lambda_{jj} \cos^2(\theta_j)} \Big[ 1 - \frac{||\Delta_j||^2_{\Lambda^{-1}}  \sin^2(\theta_j)}{\Lambda_{jj} \cos^2(\theta_j) + ||\Delta_j||^2_{\Lambda^{-1}}  \sin^2(\theta_j)} \Big]} - \Lambda_{ii} \quad  \nonumber
    \\ &- \sum_{j < i} \frac{- \Lambda_{ii}^2 \sin^2(\theta_j) \langle \Delta_j, v_i \rangle^2 + \sin^2(\theta_j) \langle \Delta_i, \Lambda \Delta_j \rangle^2 + \Lambda_{jj} \sin(2\theta_j) \langle \Delta_i, v_j \rangle \langle \Delta_i, \Lambda \Delta_j \rangle}{\Lambda_{jj} \cos^2(\theta_j) + ||\Delta_j||^2_{\Lambda^{-1}}  \sin^2(\theta_j)}
    \\ &\le \textcolor{orange}{||\Delta_i||^2_{\Lambda^{-1}} - \sum_{j < i} \frac{\Lambda_{jj}^2 \cos^2(\theta_j) \langle \Delta_i, v_j \rangle^2}{\Lambda_{jj} \cos^2(\theta_j)}} + \sum_{j < i} \Big(||\Delta_j||^2_{\Lambda^{-1}}  \sin^2(\theta_j)\Big) \frac{\Lambda_{jj}^2 \cos^2(\theta_j) \langle \Delta_i, v_j \rangle^2}{\Lambda_{jj}^2 \cos^4(\theta_j)} - \Lambda_{ii} \nonumber
    \\ &+ \sum_{j < i} \frac{\Lambda_{ii}^2 \sin^2(\theta_j) \langle \Delta_j, v_i \rangle^2 + 2 \Lambda_{jj} \sqrt{\sin^2(\theta_j)} \sqrt{\cos^2(\theta_j)} \vert \langle \Delta_i, v_j \rangle \vert \vert \langle \Delta_i, \Lambda \Delta_j \rangle \vert}{\Lambda_{jj} \cos^2(\theta_j)}
    \\ &= \textcolor{orange}{u_i(\Delta_i, v_{j < i})} + \sum_{j < i} \Big(||\Delta_j||^2_{\Lambda^{-1}}  \sin^2(\theta_j)\Big) \frac{\langle \Delta_i, v_j \rangle^2}{\cos^2(\theta_j)} - \Lambda_{ii} \nonumber
    \\ &+ \sum_{j < i} \frac{\Lambda_{ii}^2 \sin^2(\theta_j) \langle \Delta_j, v_i \rangle^2 + 2 \Lambda_{jj} \sqrt{\sin^2(\theta_j)} \sqrt{\cos^2(\theta_j)} \vert \langle \Delta_i, v_j \rangle \vert \vert \langle \Delta_i, \Lambda \Delta_j \rangle \vert}{\Lambda_{jj} \cos^2(\theta_j)}
    \\ &\stackrel{\text{[\textcolor{cobalt}{L\ref{ui_of_deltai}]}}}{\le}\textcolor{blue}{\Lambda_{i+1,i+1}} - \Lambda_{ii} + \sum_{j < i} \Big(\textcolor{orange}{||\Delta_j||^2_{\Lambda^{-1}}} \sin^2(\theta_j)\Big) \frac{\cancelto{1}{\langle \Delta_i, v_j \rangle^2}}{\cos^2(\theta_j)}  \nonumber
    \\ &+ \sum_{j < i} \frac{\textcolor{green}{\Lambda_{ii}^2} \sin^2(\theta_j) \cancelto{1}{\langle \Delta_j, v_i \rangle^2} + 2 \Lambda_{jj} \sqrt{\sin^2(\theta_j)} \sqrt{\cos^2(\theta_j)} \vert \langle \Delta_i, v_j \rangle \vert \vert \langle \Delta_i, \Lambda \Delta_j \rangle \vert}{\textcolor{green}{\Lambda_{jj}} \cos^2(\theta_j)}
    \\ &\le \Lambda_{i+1,i+1} - \Lambda_{ii} + \sum_{j < i} \epsilon^2 \frac{\textcolor{orange}{\Lambda_{11}} + \textcolor{green}{\Lambda_{ii}}}{\cos^2(\theta_j)} + 2 \frac{\Lambda_{jj} \sqrt{\sin^2(\theta_j)} \sqrt{\cos^2(\theta_j)} \vert \langle \Delta_i, v_j \rangle \vert \vert \langle \Delta_i, \Lambda \Delta_j \rangle \vert}{\Lambda_{jj} \cos^2(\theta_j)}
    \\ &\le \Lambda_{i+1,i+1} - \Lambda_{ii} + \sum_{j < i} \epsilon^2 \frac{\Lambda_{11} + \Lambda_{ii}}{\cos^2(\theta_j)} + 2 \Lambda_{11} \sqrt{\frac{\sin^2(\theta_j)}{\cos^2(\theta)}}
    \\ &\le \Lambda_{i+1,i+1} - \Lambda_{ii} + (i-1) (\Lambda_{11} + \Lambda_{ii}) \frac{\epsilon^2}{1 - \epsilon^2} + 2 (i-1) \Lambda_{11} \frac{\epsilon}{\sqrt{1 - \epsilon^2}}.
\end{align}
\endgroup

Note $\textcolor{green}{\frac{\Lambda_{ii}^2}{\Lambda_{jj}} < \Lambda_{ii}}$ because $\Lambda_{ii} < \Lambda_{jj}$ for all $j < i$.
\end{proof}

\begin{lemma}
\label{A_upperbnd_simple}
Assume $\epsilon^2 \le \frac{1}{2}$. Then
\begin{align}
    A &\le -g_i + 8 (i-1) \Lambda_{11} \epsilon.
\end{align}
\end{lemma}
Assume $\epsilon^2 \le \frac{1}{2}$ so $\frac{\epsilon}{\sqrt{1-\epsilon^2}} \le 1$. Then
\begin{align}
    A &\le \Lambda_{i+1,i+1} - \Lambda_{ii} + (i-1) (\Lambda_{11} + \Lambda_{ii}) \frac{\epsilon^2}{1 - \epsilon^2} + 2 (i-1) \Lambda_{11} \frac{\epsilon}{\sqrt{1 - \epsilon^2}}
    \\ &\le -g_i + (i-1) \Big[ \frac{\epsilon}{\sqrt{1-\epsilon^2}} \Big] \Big[ 3\Lambda_{11} + \Lambda_{ii} \Big]
    \\ &\le -g_i + 4 (i-1) \Lambda_{11} \frac{\epsilon}{\sqrt{1-\epsilon^2}}
    \\ &\le -g_i + 8 (i-1) \Lambda_{11} \epsilon.
\end{align}

\begin{lemma}
\label{B_upperbnd_simple}
Assume $\epsilon^2 \le \frac{1}{2}$. Then
\begin{align}
    \vert B \vert &\le 8 (i-1) \Lambda_{ii} \kappa_{i-1} \epsilon.
\end{align}
\end{lemma}

\begin{proof}
\begin{align}
    \vert B \vert &= \sum_{j < i} \frac{\vert \Lambda_{ii} \Lambda_{jj} \sin(2\theta_j) \langle \Delta_j, v_i \rangle \langle \Delta_i, v_j \rangle + 2 \Lambda_{ii} \sin^2(\theta_j) \langle \Delta_j, v_i \rangle \langle \Delta_i, \Lambda \Delta_j \rangle \vert}{\Lambda_{jj} \cos(\theta_j)^2 + ||\Delta_j||_{\Lambda^{-1}}  \sin^2(\theta_j)}
    \\ &\le \sum_{j < i} \frac{\Lambda_{ii} \Lambda_{jj} \sqrt{\sin^2(2\theta_j)} + 2 \Lambda_{ii} \sin^2(\theta_j) \Lambda_{11}}{\Lambda_{jj} \cos(\theta_j)^2}
    \\ &\le \sum_{j < i} \frac{\Lambda_{ii} \Lambda_{jj} \sqrt{4\sin^2(\theta_j)\cos^2(\theta_j)} + 2 \Lambda_{ii} \sin^2(\theta_j) \Lambda_{11}}{\Lambda_{jj} \cos(\theta_j)^2}
    \\ &\le 2\sum_{j < i} \frac{\Lambda_{ii} \Lambda_{jj} \epsilon + \Lambda_{ii} \epsilon^2 \Lambda_{11}}{\Lambda_{jj} (1-\epsilon^2)}
    \\ &= 2\Lambda_{ii} \frac{\epsilon}{1-\epsilon^2} \Big( (i-1) + \epsilon \sum_{j < i} \kappa_j \Big)
    \\ &\le 4\Lambda_{ii} \epsilon \Big( (i-1) + \epsilon (i-1) \kappa_{i-1} \Big)
    \\ &= 4(i-1) \Lambda_{ii} \epsilon \Big( 1 + \epsilon \kappa_{i-1} \Big)
    \\ &\le 4(i-1) \Lambda_{ii} \epsilon \Big( 1 + \frac{1}{\sqrt{2}} \kappa_{i-1} \Big)
    \\ &\le 8 (i-1) \Lambda_{ii} \kappa_{i-1} \epsilon.
\end{align}
\end{proof}

\begin{lemma}
\label{tan_arg_bound}
Let $\epsilon_i = \frac{c_i g_i}{(i-1)\Lambda_{11}}$ with $c_i < \frac{1}{8}$. Then
\begin{enumerate}[(i)]
    \item $A\leq 0$, 
    \item $\Big\vert \frac{B}{A} \Big\vert \le \frac{8c_i}{1 - 8c_i}$.
\end{enumerate}
\end{lemma}

\begin{proof}
Plugging in Lemma~\ref{A_upperbnd_simple} and $\epsilon_i$, we find
\begin{align}
    A &\le -g_i + 8c_i \frac{(i-1)\Lambda_{11} g_i}{(i-1)\Lambda_{11}} = -g_i + 8c_i g_i = (8c_i - 1) g_i. \label{A_with_epsilon}
\end{align}
Since we assumed $c_i<1/8$, this proves $(i)$. Plugging in Lemma~\ref{B_upperbnd_simple} and $\epsilon_i$ solves $(ii)$:
\begin{align}
     \text{\Eqref{A_with_epsilon}}\implies \vert A \vert &\ge (1 - 8c_i) g_i \label{abs_A_lower_bnd}
    \\ \vert B \vert &\le 8c_i \frac{(i-1)\Lambda_{ii} \kappa_{i-1} g_i}{(i-1)\Lambda_{11}} = 8c_i g_i \frac{\Lambda_{ii}}{\Lambda_{i-1,i-1}} \le 8c_i g_i
    \\ \implies \vert \frac{B}{A} \vert &\le \frac{8c_i}{1 - 8c_i}.
\end{align}
\end{proof}

\begin{example}
\label{error_prop_example}
We construct the following example in order to concreteley demonstrate the $\arctan$ dependence of a child ($\hat{v}_i$) on a parent ($\hat{v}_1$ in this case).

Let $\Delta_1 = v_i$, $\Delta_i = v_1$, $\Delta_{1 < j < i} = v_{i+1}$ and constrain all parents to have error $\sin(\theta_j)=\epsilon$ for all $j < i$. Then the child's optimum has an angular deviation from the true eigenvector direction of
\begin{align}
    \vert \theta_i^* \vert &= \begin{cases}
        \frac{1}{2} \tan^{-1}\Big( \vert \frac{B}{A} \vert \Big) & \text{ if } A < 0
        \\ \frac{\pi}{4} & \text{ if } A = 0
        \\ \frac{1}{2} \Big[ \pi - \tan^{-1}\Big( \vert \frac{B}{A} \vert \Big) \Big] & \text{ if } A > 0
    \end{cases}
\end{align}
where $\vert \frac{B}{A} \vert = \frac{2\epsilon \sqrt{1 - \epsilon^2}}{\vert 1 - \epsilon^2 (\kappa_i + \frac{1}{\kappa_i}) \vert}$.
\end{example}

\begin{proof}
Note that $\langle \Delta_i, v_{1 < j < i} \rangle$, $\langle \Delta_{1 < j < i}, v_i \rangle$, and $\langle \Delta_i, \Lambda \Delta_j \rangle$ all equal $0$ by design; and $\langle \Delta_i, v_1 \rangle = \langle \Delta_1, v_i \rangle = 1$. Plugging into Lemma~\ref{misspec_loss_function}, all elements of the sum disappear for $j \ge 1$ and only the blue terms survive for $j = 1$. We find
\begin{align}
    A &= ||\Delta_i||_{\Lambda^{-1}} - \Lambda_{ii}
    \\ &- \sum_{j < i} \frac{\textcolor{blue}{\Lambda_{jj}^2 \cos^2(\theta_j) \langle \Delta_i, v_j \rangle^2} - \textcolor{blue}{\Lambda_{ii}^2 \sin^2(\theta_j) \langle \Delta_j, v_i \rangle^2} + \sin^2(\theta_j) \langle \Delta_i, \Lambda \Delta_j \rangle^2}{\Lambda_{jj} \cos^2(\theta_j) + ||\Delta_j||^2_{\Lambda^{-1}}  \sin^2(\theta_j)}
    \\ &- \sum_{j < i} \frac{\Lambda_{jj} \sin(2\theta_j) \langle \Delta_i, v_j \rangle \langle \Delta_i, \Lambda \Delta_j \rangle}{\Lambda_{jj} \cos^2(\theta_j) +||\Delta_j||^2_{\Lambda^{-1}}  \sin^2(\theta_j)}
    \\ & = \Lambda_{11} - \Lambda_{ii} - \frac{\textcolor{blue}{\Lambda_{11}^2 (1-\epsilon^2)} - \textcolor{blue}{\Lambda_{ii}^2 \epsilon^2}}{\Lambda_{11} (1-\epsilon^2) + \Lambda_{11} \epsilon^2}
    \\ &= \Lambda_{11} - \Lambda_{ii} - \frac{\Lambda_{11}^2 (1-\epsilon^2) - \Lambda_{ii}^2 \epsilon^2}{\Lambda_{11}}
    \\ &= \Lambda_{11} - \Lambda_{ii} - \Big[ \Lambda_{11} (1-\epsilon^2) - \frac{\Lambda_{ii}}{\kappa_i} \epsilon^2 \Big]
    \\ &= - \Lambda_{ii} + \epsilon^2 (\Lambda_{11} + \frac{\Lambda_{ii}}{\kappa_i})
\end{align}
and
\begin{align}
    B &= \sum_{j < i} \frac{\textcolor{blue}{\Lambda_{ii} \Lambda_{jj} \sin(2\theta_j) \langle \Delta_j, v_i \rangle \langle \Delta_i, v_j \rangle} + 2 \Lambda_{ii} \sin^2(\theta_j) \langle \Delta_j, v_i \rangle \langle \Delta_i, \Lambda \Delta_j \rangle}{\Lambda_{jj} \cos(\theta_j)^2 + ||\Delta_j||^2_{\Lambda^{-1}}  \sin^2(\theta_j)}
    \\ & = \frac{\textcolor{blue}{\Lambda_{ii} \Lambda_{11} \sin(2\theta_1)}}{\Lambda_{11} \cos(\theta_1)^2 + ||\Delta_1||^2_{\Lambda^{-1}}  \sin^2(\theta_1)}
    \\ &= 2\frac{\Lambda_{ii} \Lambda_{11} \sqrt{\epsilon^2 (1-\epsilon^2)}}{\Lambda_{11} (1-\epsilon^2) + \Lambda_{11}  \epsilon^2}
    \\ &= 2\Lambda_{ii} \epsilon \sqrt{1-\epsilon^2}.
\end{align}
Then
\begin{align}
    \vert \frac{B}{A} \vert &= \frac{2\Lambda_{ii} \epsilon \sqrt{1 - \epsilon^2}}{\vert \Lambda_{ii} - \epsilon^2 (\Lambda_{11} + \frac{\Lambda_{ii}}{\kappa_i}) \vert} = \frac{2\epsilon \sqrt{1 - \epsilon^2}}{\vert 1 - \epsilon^2 (\kappa_i + \frac{1}{\kappa_i}) \vert}.
\end{align}
\end{proof}

\newpage
\section{Convergence Proof}
\label{app:conv}

\subsection{Non-Convex Riemannian Optimization Theory}
We repeat the non-convex Riemannian optimization rates here from~\citep{boumal2019global} for convenience.

\begin{lemma}
\label{riemann_grad_descent}
Under Assumptions~\ref{assump_2p3} and~\ref{assump_2p4}, generic Riemannian descent (\Algref{generic_riem_descent}) returns $x \in \mathcal{M}$ satisfying $f(x) \le f(x_0)$ and $|| \nabla^R f(x)|| \le \rho$ in at most
\begin{align}
    \lceil \frac{f(x_0) - f^*}{\xi} \cdot \frac{1}{\rho^2} \rceil
\end{align}
iterations, provided $\rho \le \frac{\xi'}{\xi}$. If $\rho > \frac{\xi'}{\xi}$, at most $\lceil \frac{f(x_0) - f^*}{\xi'} \cdot \frac{1}{\rho} \rceil$ iterations are required.
\end{lemma}

\begin{proof}
See Theorem~2.5 in~\citep{boumal2019global}.
\end{proof}

\begin{assumption}
\label{assump_2p3}
There exists $f^* > -\infty$ such that $f(x) \ge f^*$ for all $x \in M$. See Assumption~2.3 in~\citep{boumal2019global}.
\end{assumption}

\begin{assumption}
\label{assump_2p4}
There exist $\xi, \xi' > 0$ such that, for all $k \ge 0$, $f(x_k) - f(x_{k+1}) \ge \min(\xi ||\nabla^R f(x_k)||, \xi') ||\nabla^R f(x_k)||$. See Assumption~2.4 in~\citep{boumal2019global}.
\end{assumption}

\begin{algorithm}[ht]
\begin{algorithmic}
    \State Given: $f: \mathcal{M} \rightarrow \mathbb{R}$ differentiable, a retraction Retr on $\mathcal{M}$, $x_0 \in \mathcal{M}$, $\rho > 0$
    \State Init: $k \leftarrow 0$
    \While{$||\nabla^R f(x_k)|| > \rho$}
        \State Pick $\eta_k \in T_{x_k} \mathcal{M}$
    \EndWhile
    \State return $x_k$
\end{algorithmic}
\caption{Generic Riemannian descent algorithm}
\label{generic_riem_descent}
\end{algorithm}

\subsection{Convergence of \pcagame{}}

Theorem~\ref{ui_suff_conv} provides an asymptotic convergence guarantee for~\Algref{pcagame_ascent_successive} (below) to recover the top-$k$ principal components. Assuming $\hat{v}_i$ is initialized within $\frac{\pi}{4}$ of $v_i$ for all $i \le k$, Theorem~\ref{simp_conv_rate} provides a finite sample convergence rate. In particular, it specifies the total number of iterations required to learn parents such that $\hat{v}_k$ can be learned within a desired tolerance.

The proof of Theorem~\ref{ui_suff_conv} proceeds in several steps. First, recall that player $i$'s utility is sinusoidal in its angular deviation from $v_i$ and therefore, technically, non-concave although it is simple in the sense that every local maximum is a global maximum (w.r.t. angular deviation). Also, note that our ascent is not performed on the natural parameters of the sphere ($\theta_i$ and $\Delta_i$), but rather on $\hat{v}_i$ directly with $\hat{v}_i \in \mathcal{S}^{d-1}$, a Riemannian manifold.

We therefore leverage recent results in non-convex optimization, specifically minimization, for Riemannian manifolds~\citep{boumal2019global}, repeated here for convenience (see Theorem~\ref{riemann_grad_descent}). Note, we are maximizing a utility so we simply flip the sign of our utility to apply this theory. The convergence rate guarantee given by this theory is for generic Riemannian descent with a constant step size,~\Algref{generic_riem_descent}, and makes two assumptions. One is a bound on the utility (Lemma~\ref{assump_2p3}) and the other is a smoothness or Lipschitz condition (Lemma~\ref{assump_2p4}). The convergence rate itself states the number of iterations required for the norm of the Riemannian gradient to fall below a given threshold. The theory also guarantees descent in that the solution returned by the algorithm will have lower loss (higher utility) than the vector passed to the algorithm.

The probability of sampling a vector $\hat{v}_i^0$ at angular deviation within $\phi$ of the maximizer is given by
\begin{align}
    P[\vert \theta_i^0 - \theta_i^* \vert \le \phi] &= I_{\sin^2(\phi)}(\frac{d-1}{2}, \frac{1}{2}) = \frac{\texttt{Beta}(\sin^2{\phi}, \frac{d-1}{2}, \frac{1}{2})}{\texttt{Beta}(1, \frac{d-1}{2}, \frac{1}{2})}
\end{align}
where \texttt{Beta} is the incomplete beta function, and $I$ is the normalized incomplete beta function~\citep{li2011concise}. This probability quickly approaches zero for $\phi < \frac{\pi}{2}$ as the dimension $d$ increases. Therefore, for large $d$, it becomes highly probable that $\hat{v}_i$ will be initialized near an angle $\frac{\pi}{2}$ from the true eigenvector\textemdash in other words, all points are far from each other in high dimensions. In this case, $\hat{v}_i$ lies near a trough of the sinusoidal utility where gradients are small. Without a bound on the minimum possible gradient norm, a finite sample rate cannot be constructed (how many iterations are required to escape the trough?). Therefore, we can only guarantee asymptotic convergence in this setting. Next, we consider the fortuitous case where all $\hat{v}_i$ have been initialized within $\frac{\pi}{4}$. This is both to obtain a convergence rate for this setting, but also to highlight the Big-O dependencies. Note that the utility is symmetric across $\frac{\pi}{4}$ and the number of iterations required to escape a trough and reach the $\frac{\pi}{4}$ mark is equal to the number of iterations required to ascend from $\frac{\pi}{4}$ to the same distance from the peak.

In order to ensure this theory can provide meaningful bounds for \pcagame{}, we first show, assuming a child is within $\frac{\pi}{4}$ of its maximizer, that the norm of the Riemannian gradient bounds the angular deviation of a child from this maximizer.

To begin the proof, we relate the error in the parents to a bound on the ambient gradient in Lemma~\ref{grad_lipschitz_eps}. This bound is then tightened assuming parents with error below a certain threshold in Lemma~\ref{grad_lipschitz_eps_parents}. Using the fact that $u_i = \hat{v}_i^\top \nabla_{\hat{v}_i} u_i$, this bound directly translates to a bound on the utility in Corollary~\ref{bound_on_util}, thereby satisfying Assumption~\ref{assump_2p3}. Again, given accurate parents, Lemma~\ref{smoothness_bound} proves Assumption~\ref{assump_2p4} on smoothness is satisfied and derives some of the constants for the ultimate convergence rate.

Recall that we have so far been proving convergence to a local maximizer of a child's utility, which, assuming inaccurate parents, is not the same as the true eigenvector. Lemma~\ref{approx_to_true_err} upper bounds the angular deviation of an approximate maximizer from the true eigenvector using the angular deviation of a maximizer plus the approximate maximizer's approximation error. Lemma~\ref{ui_conv} then provides the convergence rate for the child to approach the true eigenvector given accurate enough parents. Finally, Theorem~\ref{ui_suff_conv} compiles the chain of convergence rates leading up the DAG towards $\hat{v}_1$ and derives a convergence rate for child $k$ given all previous parents have been learned to a high enough degree of accuracy. The number of iterations required for each parent in the chain is provided.

\begin{theorem}
\label{ui_suff_conv}
Assume all spectral gaps are positive, i.e. for $i=1...k, \, g_i>0$. Let $\theta_k$ denote the angular distance (in radians) of $\hat{v}_k$ from the true eigenvector $v_k$. Let the maximum desired error for $\theta_k = \theta_{\text{tol}} \le 1$ radian. Then set $c_k = \frac{\theta_{\text{tol}}}{16}$, $\rho_k = \frac{g_k}{2\pi} \theta_{\text{tol}}$, and
\begin{align}
    \rho_i &= \Big[ \frac{g_i g_{i+1}}{2\pi i \Lambda_{11}} \Big] c_{i+1}
    \\ c_{i} &\le \frac{(i-1)! \prod_{j=i+1}^{k} g_{j}}{(16\Lambda_{11})^{k-i}(k-1)!} c_{k}
\end{align}
for $i < k$ where the $c_{i}$'s are dictated by each $\hat{v}_i$ to its parents and represent fractions of a canonical error threshold; for example, if $\hat{v}_k$ sets $c_{k} = \frac{1}{16}$, then this threshold gets communicated up the DAG to each parent, each time strengthening.

Consider learning $\hat{v}_{i}$ by applying \Algref{pcagame_ascent_successive} successively, i.e., learn $\hat{v}_1$, stop ascent, learn $\hat{v}_2$, and so on, each with step size $\frac{1}{2L}$ and corresponding $\rho_i$ where $L = 4 \Big[ \Lambda_{11} k  + (1 + \kappa_{k-1}) \frac{g_k}{16} \Big]$. Then the top-$k$ principal components will be returned, each within tolerance $\theta_{\text{tol}}$, in the limit.
\end{theorem}

\begin{proof}
In order to learn $\hat{v}_k$, we need $\vert \theta_{j} \vert \le \frac{c_k g_k}{(k-1)\Lambda_{11}}$ with $c_k \le \frac{1}{16}$ for all $j < k$. If this requirement is met, then by Lemma~\ref{approx_to_true_err}, the angular error in $\hat{v}_k$ after running Riemannian gradient ascent is bounded as
\begin{align}
    \vert \theta_{k} \vert &\le \bar{\epsilon} + 8c_k
\end{align}
where $\bar{\epsilon}$ denotes the convergence error and the error propagated by the parents is $8c_k$. The quantity, $\frac{g_k}{(k-1)\Lambda_{11}}$, in the parents bound is $\ll 8$, so the parents must be very accurate to reduce the error propagated to the child. Each parent must then convey this information up the chain, strengthening the requirement each hop.

Let half the error in $\vert \theta_k \vert$ come from mis-specifying the utility with imperfect parents, $\hat{v}_{j < k}$, and the other half from convergence error. 
The error after learning $\hat{v}_{k-1}$ via Riemannian gradient ascent must be less than the threshold required for learning the $k$th eigenvector. Assuming $\hat{v}_{k-1}$'s parents have been learned accurately enough, $\vert \theta_{j < k-1} \vert \le \frac{c_{k-1} g_{k-1}}{(k-2)\Lambda_{11}}$, and that $\hat{v}_{j \le k}$ were initialized within $\frac{\pi}{4}$ of their maximizers, we require:
\begin{align}
    \vert \theta_{k-1} \vert &\stackrel{L\textcolor{cobalt}{\ref{ui_conv}}}{\le} \frac{\pi}{g_{k-1}} \rho_{k-1} + 8c_{k-1} \le \frac{c_{k} g_k}{(k-1) \Lambda_{11}}.
\end{align}
More generally, the error after learning $\hat{v}_{i-1}$ must be less than the threshold for learning any of its successors:
\begin{align}
    \vert \theta_{i-1} \vert &\le \frac{\pi}{g_{i-1}} \rho_{i-1} + 8c_{i-1} \le \min_{i-1 < l \le k} \Big( \frac{c_{l} g_l}{(l-1) \Lambda_{11}} \Big).
\end{align}
Assume for now that the $\argmin$ of the expression is $i$, the immediate child. First we bound the error from $\hat{v}_{i-1}$'s parents:
\begin{align}
   8c_{i-1} &\le \frac{c_i g_{i}}{\textcolor{red}2(i-1)\Lambda_{11}} \label{parent_error}
    \\ \implies c_{i-1} &\le \frac{c_i g_{i}}{16(i-1)\Lambda_{11}}.
\end{align}
Note the $\textcolor{red}{2}$ in the denominator of~\Eqref{parent_error} which appears because we desired half the error to come from the parents (half is an arbitrary choice in the analysis).
Continuing this process recursively implies
\begin{align}
    c_{i-2} &\le \frac{c_{i-1} g_{i-1}}{16(i-2)\Lambda_{11}} \le \frac{c_i g_{i-1} g_{i}}{16^2(i-2)(i-1)\Lambda_{11}^2}
    \\ \implies c_{i-n} &\le \Big[ \frac{(i-n-1)! \prod_{j=i-n+1}^{i} g_{j}}{(16\Lambda_{11})^{n}(i-1)!} \Big] c_{i}.
\end{align}
One can see that $c_{j < i}$ is strictly smaller than $c_{i}$ because each additional term added to the product is strictly less than $1$\textemdash the assumption of the $\argmin$ above is therefore correct.
In particular, this requires the first eigenvector to be learned to very high accuracy to enable learning the $k$th:
\begin{align}
    c_1 &\le \Big[ \frac{\prod_{j=2}^{k} g_{j}}{(16\Lambda_{11})^{k-1}(k-1)!} \Big] c_{k}.
\end{align}
More generally
\begin{align}
    c_{i} &\le \frac{(i-1)! \prod_{j=i+1}^{k} g_{j}}{(16\Lambda_{11})^{k-i}(k-1)!} c_{k}
\end{align}
This completes the requirement for mitigating error in the parents.

The convergence error from gradient ascent must also be bounded as (again, note the $\textcolor{red}{2}$)
\begin{align}
    \frac{\pi}{g_i} \rho_i &\le \frac{c_{i+1} g_{i+1}}{2i\Lambda_{11}}
    \\ \implies \rho_i &\le \Big[ \frac{g_i g_{i+1}}{\textcolor{red}{2}\pi i \Lambda_{11}} \Big] c_{i+1}
\end{align}
which requires at most
\begin{align}
    t_i &= \lceil 5 \Big(\frac{\pi i\Lambda_{11}}{g_i g_{i+1}}\Big)^2 \frac{1}{c_{i+1}^{2}} \rceil
\end{align}
iterations.
Given $\hat{v}_i$ is initialized within $\frac{\pi}{4}$ of its maximizer, it follows that learning each $\hat{v}_{j<k}$ consecutively via Riemannian gradient ascent for at most $\sum_{i=1}^{k-1} t_i$ iterations is sufficient for learning the $k$-th eigenvector. Riemannian gradient ascent on $\hat{v}_k$ then returns (Lemma~\ref{ui_conv})
\begin{align}
    \vert \theta_k \vert &\le \frac{\pi}{g_k} \rho_k + 8c_k \le \frac{\pi}{g_k} \rho_k + \frac{\theta_{\text{tol}}}{2}
\end{align}
after at most
\begin{align}
    t_k &= \Big\lceil \frac{5}{4} \cdot \frac{1}{\rho_k^2} \Big\rceil = \Big\lceil \frac{5\pi^2}{(\theta_{\text{tol}} g_k)^2} \Big\rceil
\end{align}
iterations.

We can relax the assumption that $\hat{v}_i$ is initialized within $\frac{\pi}{4}$ of its maximizer and obtain global convergence. Assume that $\frac{\pi}{2} - \vert \theta_i^0 \vert \le \frac{\pi}{4}$ and let $||\nabla_{\hat{v}_i^0}||$ be the initial norm of the Riemannian gradient. The utility function $u_i(\hat{v}_i, \hat{v}_{j < i})$ is symmetric across $\frac{\pi}{4}$. Therefore, the number of iterations required to ascend to within $\frac{\pi}{4}$ is given by Lemma~\ref{ui_conv}:
\begin{align}
    t_i^+ &= \Big\lceil \frac{5}{4} \big(\frac{\pi}{g_i}\big)^2 \frac{1}{ (\frac{\pi}{2} - \vert \theta_i^0 \vert)^2} \Big\rceil.
\end{align}
Alternatively, simply set the desired gradient norm to be less than the initial. This necessarily requires iterates to ascend to past $\frac{\pi}{4}$. As long as $\hat{v}_i$ is not initialized to exactly $\frac{\pi}{2}$ from the maximum (an event with Lebesgue measure $0$), the ascent process will converge to the maximizer.
\end{proof}

\begin{theorem}
\label{simp_conv_rate}
Apply the algorithm outlined in Theorem~\ref{ui_suff_conv}  with the same assumptions. Then with probability
\begin{align}
    P[\vert \theta_i^0 - \theta_i^* \vert \le \frac{\pi}{4}] &= I_{\frac{1}{2}}(\frac{d-1}{2}, \frac{1}{2})
\end{align}
where $I$ is the normalized incomplete beta function, the max total number of iterations required for learning all vectors to adequate accuracy is
\begin{align}
    T_k &= \Big\lceil \mathcal{O} \Big( k \Big[ \frac{(16 \Lambda_{11}^{k})(k-1)!}{\prod_{j=1}^{k} g_{j}} \frac{1}{\theta_{\text{tol}}} \Big]^2 \Big) \Big\rceil.
\end{align}

\end{theorem}
\paragraph{Discussion. } In other words, assuming all $\hat{v}_i$ are fortuitously initialized within $\frac{\pi}{4}$ of their maximizers, then we can state a finite sample convergence rate. The first $k$ in the Big-$\mathcal{O}$ formula for total iterations appears simply from a naive summing of worst case bounds on the number of iterations required to learn each $\hat{v}_{j < k}$ individually. The constant $16$ is a loose bound that arises from the error propagation analysis. Essentially, parent vectors, $\hat{v}_{j < i}$, must be learned to under $\frac{1}{16}$ a canonical error threshold for the child $\hat{v}_i$, $\frac{g_i}{(i-1)\Lambda_{11}}$. The Riemannian optimization theory we leverage dictates that $\frac{1}{\rho_i^2}$ iterations are required to meet a $\mathcal{O}(\rho_i)$ error threshold. This is why the squared inverse of the error threshold appears here. Breaking down the error threshold itself, the ratio $\frac{\Lambda_{11}}{g_i}$ says that more iterations are required to distinguish eigenvectors when the difference between them (summarized by the gap $g_i$) is small relative to the scale of the spectrum, $\Lambda_{11}$. The $(k-1)!$ term appears because learning smaller eigenvectors requires learning a much more accurate $\hat{v}_1$ higher up the chain.

\begin{proof}
Assume $\hat{v}_i$ is sampled uniformly in $\mathcal{S}^{d-1}$. Note this can be accomplished by normalizing a sample from a multivariate Gaussian. We will prove
\begin{enumerate}[(i)]
    \item the probability of the event that $\hat{v}_i^0$ is within $\frac{\pi}{4}$ of the maximizer of $u_i(\hat{v}_i, \hat{v}_{j < i})$,
    \item an upper bound on the number of iterations required to return all $\hat{v}_i$ with angular error less than $\theta_{\text{tol}}$.
\end{enumerate}

The probability of sampling a vector $\hat{v}_i^0$ at angular deviation within $\frac{\pi}{4}$ of the maximizer is given by twice the probability of sampling from one of the spherical caps around $v_i$ or $-v_i$. This probability is
\begin{align}
    P[\vert \theta_i^0 - \theta_i^* \vert \le \phi] &= I_{\sin^2(\phi)}(\frac{d-1}{2}, \frac{1}{2}) = \frac{\texttt{Beta}(\sin^2(\phi), \frac{d-1}{2}, \frac{1}{2})}{\texttt{Beta}(1, \frac{d-1}{2}, \frac{1}{2})}
\end{align}
where \texttt{Beta} is the incomplete beta function, and $I$ is the normalized incomplete beta function~\citep{li2011concise}. This probability quickly approaches zero for $\phi < \frac{\pi}{2}$ as the dimension $d$ increases. This proves $(i)$.

Plugging the bound on $c_i$
\begin{align}
    c_{i} &\le \frac{(i-1)! \prod_{j=i+1}^{k} g_{j}}{(16\Lambda_{11})^{k-i}(k-1)!} c_{k}
\end{align}
into the bound on iterations
\begin{align}
    t_i &= \lceil 5 \Big(\frac{\pi i\Lambda_{11}}{g_i g_{i+1}}\Big)^2 \frac{1}{c_{i+1}^{2}} \rceil
\end{align}
we find
\begin{align}
    t_i &= \Big\lceil 5 \Big(\frac{\pi i\Lambda_{11}}{g_i g_{i+1}}\Big)^2 \frac{(16\Lambda_{11})^{2(k-i-1)}((k-1)!)^2}{(i!)^2 \prod_{j=i+2}^{k} g_{j}^2} \frac{1}{c_k^2} \Big\rceil
    \\ &= \Big\lceil 5\pi^2 \frac{16^{2(k-i)}\Lambda_{11}^{2(k-i)}((k-1)!)^2}{\Big( \prod_{j=i}^{k} g_{j}^2 \Big) ((i-1)!)^2} \frac{1}{(16c_k)^2} \Big\rceil
    \\ &\le \Big\lceil 5 \pi^2 \Big[ \frac{(16 \Lambda_{11})^{k-1}(k-1)!}{\prod_{j=1}^{k} g_{j}} \frac{1}{16c_k} \Big]^2 \Big\rceil \quad \text{[$\Lambda_{11} \ge g_i \,\, \forall i$]}
    \\ &= \Big\lceil \mathcal{O} \Big( \Big[ \frac{(16 \Lambda_{11})^{k}(k-1)!}{\prod_{j=1}^{k} g_{j}} \frac{1}{16c_k} \Big]^2 \Big) \Big\rceil
\end{align}
which is now in a form independent of $i$ (worst case). It can be shown that $t_k \le t_1$ by taking their $\log$ and applying Jensen's inequality. The total iterations required for learning $\hat{v}_{j < k}$ is at most $k-1$ times this. Therefore,
\begin{align}
    T_k &= \Big\lceil \mathcal{O} \Big( k \Big[ \frac{(16 \Lambda_{11})^{k}(k-1)!}{\prod_{j=1}^{k} g_{j}} \frac{1}{16c_k} \Big]^2 \Big) \Big\rceil.
\end{align}
\end{proof}

\begin{corollary}[PC Convergence $\implies$ Subspace Convergence]
\label{ang_conv_to_subspace_conv}
Convergence of $\hat{V}$ to the top-$k$ principal components of $X$ with maximum angular error $\theta_{tol}$ implies convergence to the top-$k$ subspace of $X$ in the following sense\footnote{See~\cite{allen2017first} for more details on this measure of subspace error.}:
\begin{align}
    ||\hat{V}^\top V_{\lnot k}||_F^2 \le k (d-k) \theta_{tol}^2 .
\end{align}
where the columns of $V_{\lnot k}$ comprise the bottom $d-k$ eigenvectors of $M = X^\top X$.
\end{corollary}
\begin{proof}
Recall that the true principal components, $v_i$, are all orthogonal. If the angle between $\hat{v}_i$ and $v_i$ is less than or equal to $\theta_{tol}$ for every $i$, then the angle between $\hat{v}_i$ and $v_j$ for any $j \ne i$ must be greater than or equal to $\frac{\pi}{2} - \theta_{tol}$. The entries in $\hat{V}^\top V_{\lnot k}$ are equal to the cosines of the angles between each of the columns in $\hat{V}$ and $V_{\lnot k}$. Therefore, all entries are less than or equal to $\vert \cos(\frac{\pi}{2} - \theta_{tol}) \vert = \vert \sin(\theta_{tol}) \vert \le \theta_{tol}$. This implies the squared Frobenius norm of this matrix is less than or equal to the number of entries times the maximum value squared: $k(d-k) \theta_{tol}^2$.
\end{proof}

\begin{lemma}
\label{angle_to_riemann_grad}
Assume $\hat{v}_i$ is within $\frac{\pi}{4}$ of its maximizer, i.e., $\vert \theta_i - \theta_i^* \vert \le \frac{\pi}{4}$. Also, assume that $\vert \theta_{j < i} \vert \le \frac{c_i g_i}{(i-1) \Lambda_{11}} \le \sqrt{\frac{1}{2}}$ with $0 \le c_i \le \frac{1}{16}$. Then the norm of the Riemannian gradient of $u_i$ upper bounds this angular deviation:
\begin{align}
    \vert \theta_i - \theta_i^* \vert \le \frac{\pi}{g_i} ||\nabla^R_{\hat{v}_i} u_i(\hat{v}_i, \hat{v}_{j < i})||.
\end{align}
\end{lemma}

\begin{proof}
The Riemannian gradient measures how the utility $u_i$ changes while moving along the manifold. In contrast, the ambient gradient measures how $u_i$ changes while moving in ambient space, possibly off the manifold. Rather than bounding the angular deviation using the projection of the ambient gradient onto the tangent space of the manifold, $(I - \hat{v}_i \hat{v}_i^\top) \nabla_{\hat{v}_i} u_i$, we instead reparameterize $\hat{v}_i$ to ensure it lies on the manifold, $\hat{v}_i = \cos(\theta_i) v_i + \sin(\theta_i) \Delta_i$ where $\Delta_i$ is a unit vector and  $\langle v_i, \Delta_i \rangle = 0$. Computing gradients with respect to the new unconstrained arguments allows recovering a bound on the Riemannian gradient via a simple chain rule calculation.

We lower bound the norm of the Riemannian gradient as follows:
\begin{align}
    \frac{\partial u_i}{\partial \theta_i} &= \nabla^R_{\hat{v}_i} u_i(\hat{v}_i, \hat{v}_{j < i})^\top \frac{\partial v}{\partial \theta_i}
    \\ \implies ||\frac{\partial u_i}{\partial \theta_i}|| &\le ||\nabla^R_{\hat{v}_i} u_i(\hat{v}_i, \hat{v}_{j < i}||  ||\frac{\partial \hat{v}_i}{\partial \theta_i}||
    \\ \implies ||\nabla^R_{\hat{v}_i} u_i(\hat{v}_i, \hat{v}_{j < i}|| &\ge \frac{||\partial u_i/\partial \theta_i||}{||\partial \hat{v}_i / \partial \theta_i||}.
\end{align}
Note that $||\partial \hat{v}_i / \partial \theta_i|| = 1$ by design.
And the numerator can be bounded using Lemma~\ref{ui_sinusoidal} as
\begin{align}
    ||\partial u_i / \partial \theta_i|| &= \sqrt{A^2 + B^2} \vert \sin(2\big(\theta_i - \theta_i^*)\big) \vert
\end{align}
where $\theta_i^* = -\frac{\phi}{2}$ and $\phi = \tan^{-1}\Big( \frac{B}{A} \Big)$. Furthermore, assume $\vert \theta_i - \theta_i^* \vert \le \frac{\pi}{4}$. Then
\begin{align}
    \vert \sin(2\big(\theta_i - \theta_i^*)\big) \vert &\ge \Big\vert \frac{2}{\pi} \big(\theta_i - \theta_i^*) \Big\vert.
\end{align}
Combining the results gives
\begin{align}
    ||\nabla^R_{\hat{v}_i} u_i(\hat{v}_i, \hat{v}_{j < i}|| &\ge \frac{||\partial u_i/\partial \theta_i||}{||\partial v / \partial \theta_i||}
    \\ &= ||\partial u_i/\partial \theta_i||
    \\ &\ge \frac{2}{\pi} \sqrt{A^2 + B^2} \vert \theta_i - \theta_i^* \vert
    \\ &\ge \frac{2}{\pi} \vert A \vert \vert \theta_i - \theta_i^* \vert
    \\ &\stackrel{L\textcolor{cobalt}{\ref{tan_arg_bound}}}{\ge} \frac{2}{\pi} (1-8c) g_i \vert \theta_i - \theta_i^* \vert
    \\ &\ge \frac{g_i}{\pi} \vert \theta_i - \theta_i^* \vert
\end{align}
completing the proof.
\end{proof}

\begin{lemma}
\label{ratio_bound}
Let $\vert \theta_j \vert \le \epsilon < 1$ for all $j < i$. Then the ratio of generalized inner products is bounded as
\begin{align}
    \frac{\langle \hat{v}_i, \Lambda \hat{v}_j \rangle}{\langle \hat{v}_j, \Lambda \hat{v}_j \rangle} &\le \frac{1 + (1 + \kappa_j) \epsilon}{\sqrt{1 - \epsilon^2}}.
\end{align}
\end{lemma}

\begin{proof}
We write $\hat{v}_{j \le i} = \cos(\theta_j) v_j + \sin(\theta_j) \Delta_j$ where $\langle \Delta_j, v_j \rangle = 0$ without loss of generality. Note that $\vert \theta_j \vert \le \epsilon$ implies $\vert \sin(\theta_j) \vert \le \epsilon$. Then
\begin{align}
    &\frac{\langle \hat{v}_i, \Lambda \hat{v}_j \rangle}{\langle \hat{v}_j, \Lambda \hat{v}_j \rangle}
    \\ &= \frac{\langle \cos(\theta_i) v_i + \sin(\theta_i) \Delta_i, \Lambda \big( \cos(\theta_j) v_j + \sin(\theta_j) \Delta_j \big) \rangle}{\langle \cos(\theta_j) v_j + \sin(\theta_j) \Delta_j, \Lambda \big( \cos(\theta_j) v_j + \sin(\theta_j) \Delta_j \big) \rangle}
    \\ &= \frac{\langle \cos(\theta_i) v_i + \sin(\theta_i) \Delta_i, \Lambda \big( \cos(\theta_j) v_j + \sin(\theta_j) \Delta_j \big) \rangle }{\Lambda_{jj} \cos(\theta_j)^2 + \langle \Delta_j, \Lambda \Delta_j \rangle \sin^2(\theta_j)}
    \\ &= \frac{\Lambda_{jj} \sin(\theta_i) \cos(\theta_j) \langle \Delta_i, v_j \rangle + \Lambda_{ii} \sin(\theta_j) \cos(\theta_i) \langle \Delta_j, v_i \rangle + \sin(\theta_i) \sin(\theta_j) \langle \Delta_i, \Lambda \Delta_j \rangle}{\Lambda_{jj} \cos(\theta_j)^2 + ||\Delta_j||^2_{\Lambda^{-1}}  \sin^2(\theta_j)}
    \\ &\le \frac{\Lambda_{jj} \vert \sin(\theta_i) \vert \sqrt{1 - \epsilon^2} + \Lambda_{ii} \epsilon \vert \cos(\theta_i) \vert + \vert \sin(\theta_i) \vert \epsilon \Lambda_{11}}{\Lambda_{jj} (1 - \epsilon^2)}
    \\ &\le \frac{\Lambda_{jj} \sqrt{1 - \epsilon^2} + \Lambda_{ii} \epsilon + \epsilon \Lambda_{11}}{\Lambda_{jj} (1 - \epsilon^2)}
    \\ &= \frac{1}{\sqrt{1 - \epsilon^2}} + \Big( \frac{\Lambda_{ii}}{\Lambda_{jj}} + \kappa_j \Big) \frac{\epsilon}{\sqrt{1 - \epsilon^2}}
    \\ &\le \frac{1 + (1 + \kappa_j) \epsilon}{\sqrt{1 - \epsilon^2}}.
\end{align}
\end{proof}

\begin{lemma}[Lipschitz Bound]
\label{grad_lipschitz_eps}
Let $\vert \theta_j \vert \le \epsilon < 1$ for all $j < i$. Then the norm of the ambient gradient of $u_i$ is bounded as
\begin{align}
    || \nabla_{\hat{v}_i} u_i(\hat{v}_i, \hat{v}_{j < i}) || &\le 2 \Lambda_{11} \Big[ 1 + (i-1) \frac{1 + (1 + \kappa_{i-1}) \epsilon}{\sqrt{1 - \epsilon^2}} \Big].
\end{align}
\end{lemma}

\begin{proof}
Starting with the gradient (Equation~\ref{ui_grad}), we find
\begin{align}
    || \nabla_{\hat{v}_i} u_i(\hat{v}_i, \hat{v}_{j < i}) || &= || 2M \Big[ \hat{v}_i - \sum_{j < i} \frac{\hat{v}_i^\top M \hat{v}_j}{\hat{v}_j^\top M \hat{v}_j} \hat{v}_j \Big] ||
    \\ &\le 2 || M \hat{v}_i || + 2 \sum_{j < i} || \frac{\hat{v}_i^\top M \hat{v}_j}{\hat{v}_j^\top M \hat{v}_j} M \hat{v}_j ||
    \\ &\le 2 || M \hat{v}_i || + 2 \sum_{j < i} || \frac{\hat{v}_i^\top M \hat{v}_j}{\hat{v}_j^\top M \hat{v}_j} || || M \hat{v}_j ||
    \\ &\stackrel{L\textcolor{cobalt}{{\ref{ratio_bound}}}}{\le} 2 \Lambda_{11} + 2 \sum_{j < i} \frac{1 + (1 + \kappa_j) \epsilon}{\sqrt{1 - \epsilon^2}} \Lambda_{11} 
    \\ &= 2 \Lambda_{11} \Big[ 1 + (i-1) \frac{1 + (1 + \kappa_{i-1}) \epsilon}{\sqrt{1 - \epsilon^2}} \Big].
\end{align}
\end{proof}

\begin{lemma}[Lipschitz Bound with Accurate Parents]
\label{grad_lipschitz_eps_parents}
Assume $\vert \theta_j \vert \le \epsilon \le \frac{c_i g_i}{(i-1)\Lambda_{11}} \le \sqrt{\frac{1}{2}}$ for all $j < i$ with $0 \le c_i \le \frac{1}{16}$. Then the norm of the ambient gradient of $u_i$ is bounded as
\begin{align}
    ||\nabla_{\hat{v}_i} u_i(\hat{v}_i, \hat{v}_{j < i}) || &\le 4 \Big[ \Lambda_{11} i  + (1 + \kappa_{i-1}) c_i g_i \Big] \myeq L_i. \label{lipschitz_gradient}
\end{align}
\end{lemma}

\begin{proof}
Starting with Lemma~\ref{grad_lipschitz_eps}, we find
\begin{align}
    || \nabla_{\hat{v}_i} u_i(\hat{v}_i, \hat{v}_{j < i}) || &\le 2 \Lambda_{11} \Big[ 1 + (i-1) \frac{1 + (1 + \kappa_{i-1}) \epsilon}{\sqrt{1 - \epsilon^2}} \Big]
    \\ &\le 2 \Lambda_{11} \Big[ 1 + 2(i-1) \big(  1 + (1 + \kappa_{i-1}) \epsilon \big) \Big]
    \\ &\stackrel{\text{assumption}}{\leq } 2 \Lambda_{11} \Big[ 1 + 2(i-1) + 2\frac{(1 + \kappa_{i-1}) c g_i}{\Lambda_{11}} \Big]
    \\ &\le 4 \Big[ \Lambda_{11} \big( 1 + (i-1) \big)  + (1 + \kappa_{i-1}) c g_i \Big]
    \\ &= 4 \Big[ \Lambda_{11} i  + (1 + \kappa_{i-1}) c g_i \Big].
\end{align}
\end{proof}

\begin{corollary}[Bound on Utility]
\label{bound_on_util}
Assume $\vert \theta_j \vert \le \frac{c_i g_i}{(i-1)\Lambda_{11}} \le \sqrt{\frac{1}{2}}$ for all $j < i$ with $0 \le c_i \le \frac{1}{16}$. Then the absolute value of the utility is bounded as follows
\begin{align}
    \vert u_i(\hat{v}_i, \hat{v}_{j < i}) \vert &= \vert \hat{v}_i^\top \nabla_{\hat{v}_i} \vert \le ||\hat{v}_i|| ||\nabla_{\hat{v}_i}|| = ||\nabla_{\hat{v}_i}|| \le L_i,
\end{align}
thereby satisfying Assumption~\ref{assump_2p3}.
\end{corollary}

\begin{lemma}
\label{smoothness_bound}
Assume $\vert \theta_j \vert \le \frac{c_i g_i}{(i-1)\Lambda_{11}} \le \sqrt{\frac{1}{2}}$ for all $j < i$ with $0 \le c_i \le \frac{1}{16}$. Then Assumption~\ref{assump_2p4} is satisfied with $\xi=\xi'=\frac{8}{5}L_i$.
\end{lemma}

\begin{proof}
Let $\eta = \alpha \nabla^R_{\hat{v}_i} u_i = \alpha (I - \hat{v}_i \hat{v}_i^\top) \nabla_{\hat{v}_i} u_i$, $\alpha > 0$, and $\hat{\eta} = \frac{\eta}{||\eta||}$. Let $\hat{v}_i' = \frac{\hat{v}_i + \eta}{\gamma}$ where $\gamma = || \hat{v}_i + \eta ||$.
\begin{align}
     u_i(\hat{v}_i') &= \frac{1}{\gamma^2} \Big[ (\hat{v}_i + \eta)^\top \Lambda (\hat{v}_i + \eta) - \sum_{j < i} \frac{\Big( (\hat{v}_i + \eta)^\top \Lambda \hat{v}_{j} \Big)^2}{\hat{v}_j^\top \Lambda \hat{v}_j} \Big]
     \\ &= \frac{1}{\gamma^2} \Big[ \textcolor{blue}{\hat{v}_i ^\top \Lambda \hat{v}_i - \sum_{j < i} \frac{(\hat{v}_i^\top \Lambda \hat{v}_{j})^2}{\hat{v}_j^\top \Lambda \hat{v}_j} }+ \textcolor{orange}{\eta^\top \Lambda \eta - \sum_{j < i} \frac{(\eta^\top \Lambda \hat{v}_{j})^2}{\hat{v}_j^\top \Lambda \hat{v}_j} }+ \textcolor{green}{2 \eta^\top \Lambda \hat{v}_i - 2 \sum_{j < i} \frac{(\hat{v}_i^\top \Lambda \hat{v}_{j})(\eta ^\top \Lambda \hat{v}_{j})}{\hat{v}_j^\top \Lambda \hat{v}_j} }\Big]
     \\ &= \frac{1}{\gamma^2} \Big[ \textcolor{blue}{u_i(\hat{v}_i)} + \textcolor{orange}{ u_i(\eta)} + \textcolor{green}{2 \eta^\top \nabla_{\hat{v}_i} u_i(\hat{v}_i) }\Big]
     \\ &= \frac{1}{\gamma^2} \Big[ u_i(\hat{v}_i) + ||\eta||^2 u_i(\hat{\eta}) + 2 \eta^\top \nabla_{\hat{v}_i} u_i(\hat{v}_i) \Big]
\end{align}
The vectors $\hat{v}_i$ and $\nabla_{\hat{v}_i} u_i(\hat{v}_i)$ define a 2-d plane in which $\hat{v}_i'$ lies independent of the step size $\alpha$. Therefore, we can consider gradients confined to a 2-d plane without loss of generality. Specifically, let $\hat{v}_i = \begin{bmatrix} 0 \\ 1 \end{bmatrix}$ and $\nabla = \nabla_{\hat{v}_i} u_i(\hat{v}_i) = \beta \begin{bmatrix} \cos(\psi) \\ \sin(\psi) \end{bmatrix}$. Then $\nabla^R = \nabla^R_{\hat{v}_i} u_i(\hat{v}_i) = \beta \begin{bmatrix} \cos(\psi) \\ 0 \end{bmatrix}$ and $\gamma = \sqrt{1 + ||\eta||^2} = \sqrt{1 + \alpha^2 \beta^2 \cos^2(\psi)}$. Also, let $z = \beta \cos(\psi)$ and $\alpha < \frac{1}{L_i}$ (see \Eqref{lipschitz_gradient} for definition) which implies $\alpha^2 ||\nabla^R||^2 < 1$. Then
\begin{align}
    &u_i(\hat{v}_i') - u_i(\hat{v}_i)
    \\ &= \overbrace{(\frac{1}{\gamma^2} - 1)}^{\le 0} \textcolor{cobalt}{u_i(\hat{v}_i)} + \frac{1}{\gamma^2} ( ||\eta||^2 u_i(\hat{\eta}) + 2 \eta^\top \nabla_{\hat{v}_i} u_i(\hat{v}_i) )
    \\ &\stackrel{C\textcolor{cobalt}{{\ref{bound_on_util}}}}{\ge} (\frac{1}{\gamma^2} - 1) \textcolor{cobalt}{L_i} + \frac{1}{\gamma^2} ( \alpha^2 ||\nabla^R ||^2 \textcolor{green}{u_i(\hat{\eta})} + 2 \alpha \nabla^\top \nabla^R )
    \\ &\stackrel{C\textcolor{green}{{\ref{bound_on_util}}}}{\ge} (\frac{1}{\gamma^2} - 1) L_i + \frac{1}{\gamma^2} ( 2 \alpha \nabla^\top \nabla^R + \alpha^2 ||\nabla^R ||^2 \textcolor{green}{(-L_i)})
    \\ &= (\frac{1}{1 + \alpha^2 \beta^2 \cos^2(\psi)} - 1) L_i + \frac{\alpha}{1 + \alpha^2 \beta^2 \cos^2(\psi)} ( 2 - \alpha L_i ) \beta^2 \cos^2(\psi)
    \\ &= (\frac{1}{1 + \alpha^2 z^2} - 1) L_i + \frac{\alpha(2-\alpha L_i)}{1 + \alpha^2 z^2} z^2
    \\ &= \frac{1}{1 + \alpha^2 z^2} ( L_i - L_i \alpha^2 z^2 - L + \alpha (2 - \alpha L_i) z^2)
    \\ &= \frac{1}{1 + \alpha^2 z^2} ( - 2 L_i \alpha^2 z^2 + 2 \alpha z^2)
    \\ &= \frac{2\alpha z^2}{1 + \alpha^2 z^2} (1 - \alpha L_i) > 0
\end{align}
where the assumption that $\vert \theta_j \vert \le \frac{c_i g_i}{(i-1)\Lambda_{11}}$ was used to leverage Corollary~\textcolor{cobalt}{\ref{bound_on_util}}.
Let $\alpha = \frac{1}{2L_i}$. Then $||\eta||^2 = \alpha^2 z^2 \le \frac{1}{4}$ and
\begin{align}
    u_i(\hat{v}_i') - u_i(\hat{v}_i) &\ge \frac{2\alpha z^2}{1 + \alpha^2 z^2} (1 - \alpha L_i)
    \\ &= \frac{2\alpha^2 z^2}{1 + \alpha^2 z^2} \frac{1 - \alpha L_i}{\alpha}
    \\ &= \frac{2 L_i \alpha^2 z^2}{1 + \alpha^2 z^2}
    \\ &= \frac{2 L_i ||\eta||^2}{1 + ||\eta||^2}
    \\ &\ge \min(\xi ||\eta||^2, \xi' ||\eta||)
\end{align}
with $\xi = \xi' = \frac{8}{5} L_i$.
\end{proof}

\begin{lemma}[Approximate Optimization is Reasonable Given Accurate Parents]
\label{approx_to_true_err}
Assume $\vert \theta_j \vert \le \frac{c_i g_i}{(i-1)\Lambda_{11}} \le \sqrt{\frac{1}{2}}$ for all $j < i$ with $0 \le c \le \frac{1}{16}$, i.e., the parents have been learned accurately. Then for any approximate local maximizer $(\bar{\theta}_i, \bar{\Delta}_i)$ of $u_i(\hat{v}_i(\theta_i, \Delta_i), \hat{v}_{j < i})$, if the angular deviation $\vert \bar{\theta}_i - \theta_i^* \vert \le \bar{\epsilon}$ where $\theta_i^*$ forms the global max,
\begin{align}
    \vert \bar{\theta}_i \vert \le \bar{\epsilon} + 8c_i
\end{align}
where $\bar{\theta}_i$ denotes the angular distance of the approximate local maximizer to the true eigenvector $v_i$.
\end{lemma}

\begin{proof}
Note that the true eigenvector occurs at $\bar{\theta}_i = 0$. The result follows directly from Theorem~\ref{parent_err_to_child_err}: 
\begin{align}
    \vert \bar{\theta}_i \vert = \vert \bar{\theta}_i - 0 \vert \le \vert \bar{\theta}_i - \theta_i^* \vert + \vert \theta_i^* - 0 \vert \le \bar{\epsilon} + 8c_i.
\end{align}
\end{proof}

\begin{lemma}
\label{ui_conv}
Assume $\hat{v}_i$ is initialized within $\frac{\pi}{4}$ of its maximizer and its parents are accurate enough, i.e., that $\vert \theta_{j < i} \vert \le \frac{c_{i} g_i}{(i-1) \Lambda_{11}} \le \sqrt{\frac{1}{2}}$ with $0 \le c_{i} \le \frac{1}{16}$. Let $\rho_i$ be the maximum tolerated error desired for $\hat{v}_i$. Then Riemannian gradient ascent returns
\begin{align}
    \vert \theta_i \vert \le \frac{\pi}{g_i} \rho_i + 8c_{i}
\end{align}
after at most
\begin{align}
    \lceil \frac{5}{4} \cdot \frac{1}{\rho_i^2} \rceil
\end{align}
iterations.
\end{lemma}

\begin{proof}
Note that the assumptions of Lemma~\ref{riemann_grad_descent} are met by Corollary~\ref{bound_on_util} and Lemma~\ref{smoothness_bound} with $\xi=\xi'=\frac{8}{5}$ and Riemannian gradient ascent. Plugging into Lemma~\ref{riemann_grad_descent} ensures that Riemannian gradient ascent returns unit vector $\hat{v}_i$ satisfying $u(\hat{v}_i) \ge u(\hat{v}_i^0)$ and $||\nabla^R|| \le \rho_i$ in at most
\begin{align}
    \lceil \frac{u(\hat{v}_i^*) - u(\hat{v}_i^0)}{\frac{8}{5} L_i} \cdot \frac{1}{\rho_i^2} \rceil
\end{align}
iterations (where $\hat{v}_i$ is initialized to $\hat{v}_i^0$). Additionally, note that for any $\hat{v}_i$, $u_i(\hat{v}_i^*) - u_i(\hat{v}_i) \le 2L_i$ where $L_i$ bounds the absolute value of the utility $u_i$ (see Corollary~\ref{bound_on_util}) and $\hat{v}_i^* = \argmax u_i(\hat{v}_i)$. Combining this with Lemma~\ref{angle_to_riemann_grad} gives
\begin{align}
    \vert \theta_i - \theta_i^* \vert \le \frac{\pi}{g_i} \rho_i
\end{align}
after at most
\begin{align}
    \lceil \frac{5}{4} \cdot \frac{1}{\rho_i^2} \rceil
\end{align}
iterations.
Lastly, translating $\vert \theta_i - \theta_i^* \vert$ to $\vert \theta_i \vert$ using Lemma~\ref{approx_to_true_err} gives the desired result.
\end{proof}

\end{document}